\documentclass{article}

\usepackage{PRIMEarxiv}

\usepackage[utf8]{inputenc} % allow utf-8 input
\usepackage[T1]{fontenc}    % use 8-bit T1 fonts
\usepackage{hyperref}       % hyperlinks
\usepackage{url}            % simple URL typesetting
\usepackage{booktabs}       % professional-quality tables
\usepackage{amsfonts}       % blackboard math symbols
\usepackage{nicefrac}       % compact symbols for 1/2, etc.
\usepackage{microtype}      % microtypography
\usepackage{lipsum}
\usepackage{fancyhdr}       % header
\usepackage{graphicx}       % graphics
\graphicspath{{media/}}     % organize your images and other figures under media/ folder

\usepackage{amsmath,amsfonts,bm}

\def\eqref#1{equation~\ref{#1}}
\def\1{\bm{1}}

\DeclareMathAlphabet{\mathsfit}{\encodingdefault}{\sfdefault}{m}{sl}
\SetMathAlphabet{\mathsfit}{bold}{\encodingdefault}{\sfdefault}{bx}{n}

\newcommand{\E}{\mathbb{E}}

\newcommand{\R}{\mathbb{R}}

\newcommand{\Var}{\mathrm{Var}}

\newcommand{\Cov}{\mathrm{Cov}}
\DeclareMathOperator*{\argmin}{arg\,min}

\DeclareMathOperator{\Tr}{Tr}

\usepackage[utf8]{inputenc} % allow utf-8 input
\usepackage[T1]{fontenc}    % use 8-bit T1 fonts
\usepackage{hyperref}       % hyperlinks
\usepackage{url}            % simple URL typesetting
\usepackage{booktabs}       % professional-quality tables
\usepackage{amsfonts}       % blackboard math symbols
\usepackage{amsmath}        % for \text in math mode
\usepackage{float}          % for [H] float placement
\usepackage{nicefrac}       % compact symbols for 1/2, etc.
\usepackage{microtype}      % microtypography
\usepackage{xcolor}         % colors
\usepackage{graphicx}
\usepackage{algorithm}
\usepackage{algpseudocode}
\usepackage{booktabs}
\usepackage{enumitem}
\usepackage{wrapfig}
\usepackage{natbib}

\usepackage{amssymb, amsthm, mathtools, bm, dsfont}
\usepackage[nameinlink]{cleveref}
\usepackage{enumitem, booktabs}

\usepackage{caption}
\theoremstyle{plain}
\newtheorem{theorem}{Theorem}
\newtheorem{proposition}{Proposition}
\newtheorem{lemma}{Lemma}
\newtheorem{corollary}{Corollary}
\theoremstyle{remark}
\newtheorem{remark}{Remark}
\newtheorem{assumption}{Assumption}
\DeclareMathOperator{\diag}{diag}

\newcommand{\N}{\mathcal{N}}
\newcommand{\tele}{\mathrm{tele}}

\title{WARP: Weight Teleportation for Attack-Resilient Unlearning Protocols
}

\author{
  {Mohammad M Maheri} \thanks{Correspondence to: m.maheri23@imperial.ac.uk.}\\
  {Imperial College London}\\
  \And
  {Xavier Cadet}\\
  {Dartmouth College}\\
  \And
  {Peter Chin}\\
  {Dartmouth College}\\
  \And
  {Hamed Haddadi}\\
  {Imperial College London}\\
}

\begin{document}
\maketitle

\begin{abstract}
Approximate machine unlearning aims to efficiently remove the influence of specific data points from a trained model, offering a practical alternative to full retraining. However, it introduces privacy risks: an adversary with access to pre- and post-unlearning models can exploit their differences for membership inference or data reconstruction. We show these vulnerabilities arise from two factors: large gradient norms of \textit{forget-set} samples and the close proximity of unlearned parameters to the original model. To demonstrate their severity, we propose unlearning-specific membership inference and reconstruction attacks, showing that several state-of-the-art methods (e.g., NGP, SCRUB) remain vulnerable.  
To mitigate this leakage, we introduce \textsc{WARP}, a \emph{plug-and-play teleportation defense} that leverages neural network symmetries to reduce \textit{forget-set} gradient energy and increase parameter dispersion while preserving predictions. This reparameterization obfuscates the signal of forgotten data, making it harder for attackers to distinguish forgotten samples from non-members or recover them via reconstruction. Across six unlearning algorithms, our approach achieves consistent privacy gains, reducing adversarial advantage (AUC) by up to 64\% in black-box and 92\% in white-box settings, while maintaining accuracy on retained data. These results highlight teleportation as a general tool for reducing attack success in approximate unlearning.
\end{abstract}

\section{Introduction}

% \textcolor{red}{{TODO: Explain the lack of previous reconstruction attacks and highlight our tailored attacks get higher results in both MIA and DRA}}

% \textcolor{red}{Unlearning Background -> Privacy Attacks}\\
% Machine unlearning (MU) is a promising approach to enforce the ``right to be forgotten'' in machine learning models~\cite{bourtoule2021machine}. Given a trained model and a subset of training data to forget (the “forget-set”), an unlearning algorithm updates the model so that it behaves as if those samples had never been part of training~\cite{zhao2024makes}. The influence of the forget-set should be removed from the model’s parameters and predictions, ideally yielding the same model one would obtain by retraining from scratch on the remaining data (the “retain-set”). Crucially, this data deletion should not unduly degrade the model’s performance on the retain-set or its generalization to new data.

Machine unlearning (MU) aims to enforce the ``right to be forgotten'' by updating a trained model so that a designated \textit{forget-set} has no influence~\cite{bourtoule2021machine,zhao2024makes}. The ideal outcome matches retraining from scratch on the remaining \textit{retain-set}, with both the model’s parameters and predictions unaffected by the forgotten data, and without degrading generalization.
% One primary motivation for machine unlearning is to ensure privacy compliance for sensitive information~\cite{wang2025unlearning}. Once personal data has been used for training, the model may memorize specific details~\cite{ravikumar2024unveiling}, putting users at risk of privacy breaches~\cite{bourtoule2021machine,carlini2022privacy}. Effective unlearning addresses this by eliminating memorized traces of the sensitive data, thereby preventing the model from exposing that information. 
% A straightforward way to remove the influence of the \textit{forget-set} is to retrain the model from scratch without it, which is computationally expensive. To address this, \textit{Exact Unlearning} algorithms---which modify the training procedure so that the effect of any sample can be provably eliminated without full retraining---have been proposed, e.g., SISA~\cite{bourtoule2021machine}, to guarantee deletion while reducing the unlearning cost.
% However, these methods usually require proactive deployment and consequent computational overhead. Approximate Unlearning methods attempt to reduce the computational overhead but with a trade-off between computational efficiency and unlearning guarantees.
% Several approximate unlearning methods have been proposed to avoid the expensive solution of retraining from scratch~\cite{kurmanji2023towards,chundawat2023can,golatkar2020forgetting,thudi2022unrolling}, typically by finetuning the original model with a procedure that aims to forget the target data while preserving overall utility.
A primary motivation for machine unlearning is to ensure privacy compliance for sensitive information~\cite{wang2025unlearning}. Once personal data is used for training, models may memorize specific details~\cite{ravikumar2024unveiling}, creating risks of privacy breaches~\cite{bourtoule2021machine,carlini2022privacy}. Unlearning addresses this by eliminating such traces, preventing exposure.  
The most direct solution is retraining from scratch without the \textit{forget set}, but this is computationally prohibitive. \textit{Exact Unlearning} methods such as SISA~\cite{bourtoule2021machine} reduce cost by modifying training to allow provable deletion, but they require proactive deployment and add overhead.  
To avoid full retraining, \textit{Approximate Unlearning} methods finetune the original model to forget the target data while preserving utility~\cite{kurmanji2023towards,chundawat2023can,golatkar2020forgetting,thudi2022unrolling}, trading computational efficiency against formal guarantees.

% At the same time, it is well known that ML models are vulnerable to various privacy attacks~\cite{rigaki2023survey}. In Membership Inference Attacks (MIA), an adversary tries to determine whether a specific sample was part of the model’s training dataset~\cite{shokri2017membership}. 
% In a model inversion or Data Reconstruction Attack (DRA), the adversary attempts to recover raw data (or a close approximation of it) from the model’s outputs or parameters~\cite{yin2021see,li2022auditing,jeon2021gradient,fang2023gifd}. These attacks have been demonstrated in both black-box settings (only observing model outputs) and white-box settings (accessing model weights)~\cite{nasr2019comprehensive}.

% At the same time, ML models are known to be vulnerable to privacy attacks~\cite{rigaki2023survey}. In Membership Inference Attacks (MIA), an adversary attempts to determine whether a given sample was part of the training set~\cite{shokri2017membership}.  
% In model inversion or Data Reconstruction Attacks (DRA), the adversary seeks to recover raw data (or a close approximation) from model outputs or parameters~\cite{yin2021see,li2022auditing,jeon2021gradient,fang2023gifd}. Such attacks have been demonstrated in both black-box settings (access to outputs) and white-box settings (access to weights)~\cite{nasr2019comprehensive}.

At the same time, ML models are vulnerable to privacy attacks~\cite{rigaki2023survey}. In Membership Inference Attacks (MIA), an adversary determines whether a given sample was part of the training set~\cite{shokri2017membership}.  
In Data Reconstruction Attacks (DRA), the adversary seeks to recover raw data (or a close approximation) from model outputs or parameters~\cite{yin2021see,li2022auditing,jeon2021gradient,fang2023gifd}. These attacks have been demonstrated in both black-box (access to outputs) and white-box (access to weights) settings~\cite{nasr2019comprehensive}.

Ironically, MU itself can leak the very data it aims to erase. Given access to both the original and unlearned models, an adversary can mount differencing attacks~\cite{hu2024learn,bertran2024reconstruction}, which substantially improve reconstruction success. Even models previously resistant to MIAs can become vulnerable once deletion is performed~\cite{bertran2024reconstruction,chen2021machine}. The key observation is that the parameter difference between the two models approximates the gradient of the forgotten sample (up to second-order terms), effectively releasing it to the adversary. Gradient inversion techniques, as in federated learning~\cite{geiping2020inverting}, can then reconstruct the forgotten data. Thus, approximate unlearning methods, especially gradient-ascent variants~\cite{kurmanji2023towards}, can inadvertently compromise privacy instead of ensuring it.  

% \begin{figure}
%     \centering
%     \includegraphics[width=0.5\linewidth]{images/norm_vs_privacy_risk.png}
%     \caption{Relationship between the gradient norms of \textit{forget-set} samples and their corresponding privacy risk, measured via the U-LIRA membership inference attack~\cite{hayes_inexact_2024} on a model after unlearning.}
%     \label{fig:norm_vs_privacy_risk}
% \end{figure}

\begin{wrapfigure}{r}{0.45\linewidth}
  \centering
  \includegraphics[width=\linewidth]{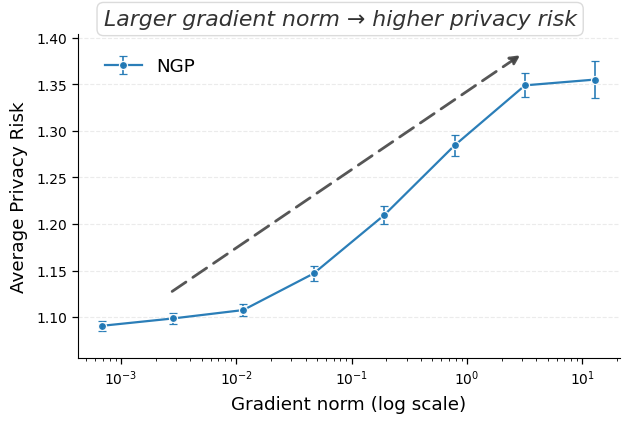}
    \caption{Privacy risk vs. gradient norms of \textit{forget-set} samples, measured with U-LiRA.}
  \label{fig:norm_vs_privacy_risk}
\end{wrapfigure}

% \textcolor{red}{Roots of the problem I) Gradient Norm -> Intuition figure II) Parameter space neighboring -> References}\\

% In this work, we aim to make machine unlearning algorithms more robust against such privacy attacks. We begin by identifying two key factors that contribute to privacy leakage in the unlearning scenario.
% First, our key observation, illustrated in Figure~\ref{fig:norm_vs_privacy_risk}, is that the privacy vulnerability of a forgotten sample correlates with the gradient norm of that sample in the original model prior to unlearning.
% Intuitively, if a model has deeply integrated a particular training example (resulting in a large gradient magnitude during training or finetuning), then removing that example causes a more noticeable change in the model’s behavior or parameters~\cite{ye2023initialization}. 
% Such samples are thus at higher risk: an attacker is more likely to detect their removal via membership inference or exploit the pronounced parameter change to reconstruct them.

In this work, we aim to strengthen MU against privacy attacks by characterizing two key factors driving leakage.
The first, illustrated in Figure~\ref{fig:norm_vs_privacy_risk}, is that a forgotten sample’s privacy risk correlates with its gradient norm in the original model.
Intuitively, samples with large gradient magnitudes during training or finetuning induce stronger parameter changes when removed, making them more detectable via MIA and more exploitable for reconstruction~\cite{ye2023initialization}.

Second, as shown in prior work~\cite{thudi2022unrolling,kurmanji2023towards}, most approximate unlearning methods make minor parameter updates, typically by maximizing the \textit{forget-set} loss while keeping retain-set accuracy stable.  
This keeps the unlearned model close to the original, so the parameter difference encodes information about the forgotten data.  
In gradient-ascent–based methods~\cite{kurmanji2023towards,chundawat2023can}, this difference is essentially the \textit{forget-set} gradient.  
Recent studies confirm that such updates expose information equivalent to a single gradient step on the forgotten sample~\cite{bertran2024reconstruction}, which attackers can invert to reconstruct it.

To mitigate these risks, we propose \textsc{WARP}, a plug-and-play defense that integrates into existing unlearning algorithms without training-time statistics. Our method leverages neural network teleportation~\cite{armenta2023neural}, exploiting parameter-space symmetries (e.g., rescaling or permutation) that preserve predictions. 
By applying selective teleportation steps before or during unlearning, we reduce \textit{forget-set} gradient norms while injecting symmetry-preserving randomness. This yields unlearned models that retain accuracy yet are displaced in parameter space, making it harder for an attacker to disentangle forgetting from teleportation. Consequently, membership inference and reconstruction attacks are significantly weakened, as shown in Sections~\ref{sec:exp-ulira}, \ref{sec:exp-whitebox-mia}, and~\ref{sec:exp-recon}.

Our \textbf{contributions} are summarized as follows:
\begin{itemize}[leftmargin=*]

% \item \textbf{Tailored privacy attacks for unlearning.}  
% We propose membership inference and reconstruction attacks specifically designed for the unlearning setting, where the adversary has access to both the pre- and post-unlearning models. These tailored attacks demonstrate that state-of-the-art methods such as NGP and SCRUB~\cite{kurmanji2023towards} remain vulnerable, since the parameter changes introduced by approximate unlearning still allow adversaries to recover sensitive information about the \textit{forget-set}.

\item \textbf{Tailored privacy attacks.}
We design MIA and DRA for the unlearning setting, where the adversary compares pre- and post-unlearning models. These attacks show that leading methods 
% such as NGP and SCRUB~\cite{kurmanji2023towards} 
remain vulnerable, as parameter updates still expose information about the \textit{forget-set}.

% \item \textbf{Symmetry-based teleportation defense.}  
% We introduce a new defense that leverages symmetry transformations of neural networks to mitigate privacy leakage. By applying loss-preserving transformations, our method reduces gradient norms of \textit{forget-set} samples while simultaneously increasing parameter dispersion, making it substantially harder for an attacker to exploit unlearning updates. This approach operates as a plug-and-play module and can be integrated into a broad range of gradient-based post-hoc unlearning algorithms without requiring auxiliary training-time statistics.  

\item \textbf{Symmetry-based defense.}
% We propose \textsc{WARP}, a plug-and-play defense that applies loss-preserving transformations to reduce \textit{forget-set} gradient norms and increase parameter dispersion, thereby obscuring the signal exploited in reconstruction and inference. It integrates into gradient-based post-hoc unlearning algorithms without requiring training-time statistics.
We propose \textsc{WARP}, a plug-and-play defense that, building on existing teleportation and symmetry constructions, applies loss-preserving transformations to reduce \textit{forget-set} gradient norms and increase parameter dispersion, thereby obscuring the signal exploited in reconstruction and inference, while remaining agnostic to the particular symmetry mechanism used to realize these transformations. WARP integrates into gradient-based post-hoc unlearning algorithms without requiring training-time statistics.

% \item \textbf{Comprehensive evaluation.}  
% We evaluate our attacks and defense on CIFAR-10 and ImageNet with ResNet-18 and Vision Transformer models, across black-box and white-box settings. Experiments covering multiple unlearning algorithms demonstrate that teleportation consistently reduces privacy leakage while preserving model accuracy on the retain-set.  
% \end{itemize}

% \item \textbf{Comprehensive evaluation.}
% We benchmark our attacks and defense on CIFAR-10 and Tiny-ImageNet with ResNet-18 and ViT models, across black- and white-box settings. Teleportation consistently lowers privacy leakage while maintaining retain-set accuracy.

\item \textbf{Comprehensive evaluation.}  
We evaluate our attacks and defense across three datasets—CIFAR-10, Tiny-ImageNet, and ImageNet-1K—using ResNet-18 and ViT-B/16 models under both black-box and white-box settings. Results across multiple unlearning algorithms show that teleportation consistently reduces privacy leakage while preserving accuracy on the retain set.

\end{itemize}

% \noindent Overall, our work introduces a \emph{new perspective} on protecting privacy in machine unlearning: we frame privacy risk through the lens of \emph{gradient norm reduction} and connect it to neural network symmetry, an underexplored optimization principle that lays a conceptual foundation for more privacy-preserving unlearning algorithms.

\noindent Overall, our work reframes unlearning privacy risk through the lens of \emph{gradient norm reduction} and connects it to neural network symmetry, an underexplored optimization principle that lays a conceptual foundation for more privacy attack–resilient unlearning algorithms. Related works to ours are discussed in more detail in Appendix~\ref{sec:appendix-rw}. 
The implementation\footnote{Our code is available at \url{https://github.com/mammadmaheri7/WARP_Unlearning}.} is publicly available.
% The code is available at \href{https://anonymous.4open.science/r/munl_mmd-06E3}{\textcolor{blue}{this link}}.

% \textcolor{red}{Add sth like this: to the best our knowledge, no previous works considered the symmetry of deep neural network in the unlearning and its effectiveness to make the unlearning algorithm to make them more robust against data privacy leakage}

% \textcolor{red}{Use sth like this at the end of the Introduction or RW : Unlike existing work, we investigate unlearning robustness against relearning attacks through the lens of smoothness optimization, establishing a seamless connection to SAM, a direct yet underexplored optimization foundation for robust LLM unlearning. $\to$ we investigate the privacy protection of forget-set data in unlearning through the lens of gradient norm reduction, establishing a connection to the symmetry of the neural network, a direct yet unexplored optimization foundation toward private unlearning.}

% \textcolor{red}{like described in \cite{naderloui2025rectifying}: As highlighted in Pitfall III, inexact unlearning lacks formal guarantees, making empirical evaluation essential \cite{hayes2025inexact}}

\section{Threat Model}
% \textcolor{red}{define both white/black box -> attacker goal (and difference with other works) -> attacker assumptions (strong such that awareness about the unlearning algorithms and hyper-parameters)}

% We consider a strong adversary aiming to perform \textit{sample-wise membership inference} to distinguish whether a given sample belongs to the \textit{forget-set} $\mathcal{D}_{\text{forg}}$ (i.e., data explicitly removed by an unlearning procedure) or to the \textit{test set} $\mathcal{D}_{\text{test}}$ (unseen non-member data). The attacker is given access to both the model before unlearning and after unlearning.

We consider a strong adversary performing \textit{sample-wise membership inference}, distinguishing whether a sample belongs to the \textit{forget-set} $\mathcal{D}_{\text{f}}$ or the \textit{test set} $\mathcal{D}_{\text{test}}$. The attacker has access to both the pre- and post-unlearning models.

% \subsection{Attacker Capabilities}

% The attacker is assumed to have the following access:
% \begin{itemize}
%     \item \textbf{Model Parameters:} Full access to the original model $\theta^{\text{org}}$ and the unlearned model $\theta^{u}$.
%     \item \textbf{Unlearning Algorithm and Hyperparameters:} Full knowledge of the unlearning algorithm $\mathcal{A}_{\text{unlearn}}$ and its hyperparameter set $\mathcal{H}_{\text{unlearn}}$, including the optimizer, learning rate, number of update steps, and the size of the retained dataset.
% \end{itemize}

\paragraph{Attacker Capabilities.}
The attacker has full access to both the original $\theta^{\text{org}}$ and unlearned model $\theta^{u}$, as well as complete knowledge of the unlearning algorithm $\mathcal{A}{\text{unlearn}}$ and its hyperparameters $\mathcal{H}{\text{unlearn}}$ (e.g., optimizer, learning rate, update steps, retain-set size).

% We consider two attack settings:

% \textbf{Black-box setting.} The attacker can query both models and observe the output probabilities $f(x; \theta^{\text{org}})$ and $f(x; \theta^{u})$, but has no access to internal parameters or gradients.

% \textbf{White-box setting.} The attacker has access to all model internals (e.g., gradients, weights, activations) from both $\theta^{\text{org}}$ and $\theta^{u}$, allowing detailed analysis of model changes during unlearning.

We consider two settings:  
\textbf{Black-box} — the attacker queries outputs $f(x;\theta^{u})$.  
\textbf{White-box} — the attacker additionally accesses full internals of both models $(\theta^{\text{org}}, \theta^{u})$, including weights.

\paragraph{Attack Objective}

Given a sample $(x, y)$ from either the \textit{forget-set} $\mathcal{D}_{\text{forg}}$ or the held-out test set $\mathcal{D}_{\text{test}}$, the attacker computes a score $A'(x, y)$ and predicts membership as $A(x, y)=\mathbb{I}[A'(x, y)>\tau]$, where $\mathbb{I}[\cdot]$ is the indicator function and $\tau$ is a decision threshold. The attacker seeks a high true positive rate (TPR) on forgotten samples while maintaining a low false positive rate (FPR) on test samples. This directly measures privacy risk: if membership can be reliably inferred, incomplete unlearning is exposed and the forgotten samples identified.  
Unlike prior work~\cite{bourtoule2021machine,maheri2025zk}, our goal is to audit unlearning algorithms from a \emph{privacy perspective}, rather than evaluating indistinguishability between approximate and exact unlearning outcomes.

\section{Methodology}

\subsection{Privacy Attacks}
To systematically evaluate privacy leakage in unlearning, we consider two complementary classes of attacks: \emph{membership inference} and \emph{data reconstruction}. 

\paragraph{Black-box (U-LiRA).}  
For the black-box setting, we adopt U-LiRA~\citep{hayes2025inexact}, an adaptation of LiRA~\citep{carlini2022membership} to unlearning.  
U-LiRA leverages shadow models trained and unlearned with the same algorithm as the target, yielding a strong adaptive baseline for auditing privacy.  
We defer full algorithmic details to Appendix~\ref{sec:appendix-ulira}.

\paragraph{White-box (Gaussian Gradient--Difference).}  
In the white-box setting, we extend the Gaussian gradient–difference framework of \citet{leemann2023gaussian} to the unlearning case by contrasting gradients computed on both the original and unlearned models.  
This contrast provides a powerful signal of residual membership leakage when both model versions are available to attacker.  
The detailed proposed formulation and test statistic are presented in Appendix~\ref{sec:appendix-glir}.

% \subsubsection{Reconstruction Attack}
\paragraph{Reconstruction Attack in Unlearning.}  
\label{sec:method-recon}

We develop a \emph{white-box} reconstruction attack tailored to approximate unlearning with retain-set updates.  
Let $\Delta\theta=\theta^{u}-\theta^{\mathrm{org}}$ be the observed parameter change after one unlearning stage (possibly aggregating multiple optimizer steps).  
As in gradient inversion, we seek an input whose parameter-gradient aligns with a target vector; here the natural target is $\Delta\theta$.  
Our baseline (single-sample) objective is:
\begin{equation}
\label{eq:baseline-inv}
\hat{x},\,\hat{y}
~\in~\arg\min_{x,y}\;
\mathcal{D}\!\left(\nabla_{\theta}\ell\!\big(f(x;\theta^{\mathrm{org}}),y\big),\;\Delta\theta\right),
\end{equation}
where $\ell$ is the training loss, $f(\cdot;\theta)$ the network, and $\mathcal{D}$ a distance (e.g., $\ell_2$ or negative cosine).

% \paragraph{Unlearning-specific mixture.}
With approximate unlearning, the update $\Delta\theta$ mixes retain and forget gradients.  
For a forget example $(x_f,y_f)$ and a retain minibatch $\mathcal{B}_r$,
\begin{equation}
\label{eq:mixture}
\Delta\theta ~\approx~ -\eta\Big(g_r - \alpha\,g_f\Big),
\qquad
g_r \!=\! \tfrac{1}{|\mathcal{B}_r|}\!\sum_{(x_r,y_r)\in\mathcal{B}_r}\!\nabla_{\theta}\ell(f(x_r;\theta^{\mathrm{org}}),y_r),
\quad
g_f \!=\! \nabla_{\theta}\ell(f(x_f;\theta^{\mathrm{org}}),y_f),
\end{equation}
with effective step size $\eta$ and ascent weight $\alpha>0$.  
Directly targeting $\Delta\theta$ in \eqref{eq:baseline-inv} is therefore confounded by $g_r$.
Even when \eqref{eq:baseline-inv} is instantiated with state-of-the-art gradient inversion methods, naively inverting the unfiltered update $\Delta\theta$ remains ineffective, producing low accuracy of the reconstruction (see Section~\ref{sec:experiments}, Table~\ref{tab:recon_teleport}).

% \paragraph{Orthogonal subspace filtering.}
Let $G_{\mathrm{org}}=[g(b_i;\theta^{\mathrm{org}})]_{i=1}^m$ and $G_{u}=[g(b_i;\theta^{u})]_{i=1}^m$ be gradient snapshots on a small probe set drawn from the training distribution.  
We compute thin SVDs, $G_{\mathrm{org}}\!=\!U_{\mathrm{org}}\Sigma_{\mathrm{org}}V_{\mathrm{org}}^\top$ and $G_u\!=\!U_{u}\Sigma_{u}V_{u}^\top$, and keep the top-$k$ left singular vectors to obtain orthonormal bases (columns) for the dominant gradient subspaces.  
Define the \emph{orthogonal projectors}
\[
\Pi_{\mathrm{org}} \!=\! U_{\mathrm{org}}U_{\mathrm{org}}^\top,\qquad
\Pi_{u} \!=\! U_{u}U_{u}^\top,\qquad
\Pi_{u}^{\perp} \!=\! I - \Pi_{u}.
\]
Unlearning attenuates the forget component, so retain gradients are expected to persist in both models, whereas the forget component is prominent in $\theta^{\mathrm{org}}$ but suppressed in $\theta^{u}$.  
We therefore \emph{orthogonalize} the update against the unlearned subspace and keep only directions supported by the original model:
\begin{equation}
\label{eq:filtered-target}
\tilde g_f \;=\; \Pi_{\mathrm{org}}\,\Pi_{u}^{\perp}\!\left(-\tfrac{1}{\eta}\,\Delta\theta\right).
\end{equation}
Intuitively, $\Pi_{u}^{\perp}$ removes directions consistent with retain gradients that remain after unlearning, while $\Pi_{\mathrm{org}}$ preserves directions active before unlearning where the forget signal resides.  
If the retain subspace is well captured, then $\Pi_{u}^{\perp} g_r \!\approx\! 0$ and $\Pi_{\mathrm{org}}\Pi_{u}^{\perp}(\alpha g_f)\!\approx\!\alpha g_f$, yielding a high-SNR estimate of the forget gradient.

% \paragraph{Final objective.}
We reconstruct the forgotten sample by solving the filtered inversion:
\begin{equation}
\label{eq:final-inv}
\hat{x}_f,\,\hat{y}_f
~\in~\arg\min_{x,y}\;
\mathcal{D}\!\left(\nabla_{\theta}\ell\!\big(f(x;\theta^{\mathrm{org}}),y\big),\;\tilde g_f\right),
\end{equation}
with optional priors or constraints on $(x,y)$.  
In practice, we choose $k$ to retain a fixed fraction of gradient energy (e.g., 90–95\%), which stabilizes the projectors and reliably isolates the forget component via orthogonalization. We empirically validate that orthogonal subspace filtering boosts reconstruction success across models and datasets; see Section~\ref{sec:exp-recon} and Appendix Table~\ref{tab:recon_core}.

% \subsection{Teleportation-based Defense}
\subsection{WARP (Teleportation-based Defense)}
\label{sec:method-teleportation-based-defense}

\paragraph{Motivation I: Parameter closeness increases privacy leakage.}
We formulate post-hoc unlearning as minimizing a composite objective that balances forgetting on $\mathcal{D}_{\mathrm{f}}$ with utility on $\mathcal{D}_{\mathrm{r}}$:
\begin{equation}
\label{eq:unlearn-obj}
\min_{\theta}\;\; \underbrace{\ell_{\mathrm{f}}\!\big(\theta \mid \mathcal{D}_{\mathrm{f}}\big)}_{\text{Forget}}
\;+\;
\lambda\,\underbrace{\ell_{\mathrm{r}}\!\big(\theta \mid \mathcal{D}_{\mathrm{r}}\big)}_{\text{Retain}},
\qquad \lambda \ge 0,
\end{equation}
where $\theta$ denotes model parameters; $\ell_{\mathrm{f}}$ is any differentiable \emph{forgetting surrogate} that penalizes high confidence or reduces fidelity on $\mathcal{D}_{\mathrm{f}}$ (e.g., loss-inflation, uniform/soft labels, margin expansion); and $\ell_{\mathrm{r}}$ is the standard training/consistency loss on $\mathcal{D}_{\mathrm{r}}$ to preserve performance. The trade-off coefficient $\lambda$ controls how strongly the unlearning step remains anchored to the retain-set: larger $\lambda$ keeps $\theta^u$ closer to $\theta^{\mathrm{org}}$, preserving accuracy but reducing the parameter shift introduced by forgetting.
A first–order optimizer with mini-batches $\mathcal{B}_{\mathrm{f}}\!\subset\!\mathcal{D}_{\mathrm{f}}$ and $\mathcal{B}_{\mathrm{r}}\!\subset\!\mathcal{D}_{\mathrm{r}}$ yields the iterative update
\begin{equation}
\label{eq:first-order}
\theta_{t+1}
=
\theta_t
-\eta_t\Big(
\nabla_{\theta}\ell_{\mathrm{f}}\!\big(\theta_t \mid \mathcal{B}_{\mathrm{f}}\big)
+\lambda\,\nabla_{\theta}\ell_{\mathrm{r}}\!\big(\theta_t \mid \mathcal{B}_{\mathrm{r}}\big)
\Big),
\end{equation}
which encompasses common post-training approximate unlearning schemes; for instance, “negative-gradient” methods are recovered by taking $\ell_{\mathrm{f}}(\cdot)=-\ell_{\text{train}}(\cdot)$ (i.e., ascent on the standard training loss over $\mathcal{D}_{\mathrm{f}}$), whereas rehearsal/consistency-based approaches instantiate $\ell_{\mathrm{r}}$ with supervised loss or distillation on $\mathcal{D}_{\mathrm{r}}$ \cite{thudi2022unrolling,kurmanji2023towards,chundawat2023can}.

% Because \eqref{eq:unlearn-obj} explicitly regularizes utility on $\mathcal{D}_{\mathrm{r}}$ and is optimized with small steps and early stopping on $\mathcal{D}_{\mathrm{f}}$, the resulting unlearned parameters $\theta^{u}$ typically remain \emph{close} to the original $\theta^{\mathrm{org}}$ in parameter space, with the displacement
% $\Delta\theta=\theta^{u}-\theta^{\mathrm{org}}$ well-approximated (to first order) by a weighted combination of gradients on the \textit{forget-set}, mildly contaminated by retain gradients \cite{thudi2022unrolling,kurmanji2023towards,huang2024unified}. 
Because \eqref{eq:unlearn-obj} explicitly regularizes utility on $\mathcal{D}{\mathrm{r}}$ and is optimized with small steps and early stopping on $\mathcal{D}{\mathrm{f}}$, the resulting unlearned parameters $\theta^{u}$ typically remain \emph{close} to the original $\theta^{\mathrm{org}}$ in parameter space. The displacement $\Delta\theta=\theta^{u}-\theta^{\mathrm{org}}$ is well-approximated (to first order) by a weighted combination of gradients on the \textit{forget-set}, mildly contaminated by retain gradients \cite{thudi2022unrolling,kurmanji2023towards,huang2024unified}.
This proximity creates a privacy attack surface: An adversary with access to $(\theta^{\mathrm{org}},\theta^{u})$ can leverage $\Delta\theta$ to perform membership inference or gradient-based reconstruction of $\mathcal{D}_{\mathrm{f}}$ \cite{hu2024learn,bertran2024reconstruction}, motivating the defenses applied over unlearning algorithms.

% \paragraph{Motivation II: gradient norm controls leakage.}
% Consider neighboring datasets $D$ and $D'$ that differ by a single example $z_f=(x_f,y_f)$. 
% Under noisy gradient descent with constant step size (Theorem~E.1 in \cite{ye2023initialization}), the parameter iterate satisfies
% \[
% \theta_{k+1}
% ~=~
% \theta_k \;-\; \eta\,\nabla_\theta \mathcal{L}(\theta_k;D)
% \;+\; \sqrt{2\eta\sigma^2}\,Z_k,
% \qquad Z_k\sim\mathcal{N}(0,I),
% \]
% and the divergence between trajectories on $D$ and $D'$ is bounded by
% \begin{equation}
% \label{eq:kl-noisy-gd}
% \mathrm{KL}\!\big(\theta_{1:K}\,\|\,\theta'_{1:K}\big)
% ~=~ \frac{1}{2\sigma^2}\sum_{k=0}^{K-1}\eta\;\mathbb{E}\!\left[
% \big\|\nabla_\theta \mathcal{L}(\theta_k;D)-\nabla_\theta \mathcal{L}(\theta_k;D')\big\|_2^{\,2}
% \right].
% \end{equation}
% When $D$ and $D'$ differ only by $z_f$, the gradient difference in \eqref{eq:kl-noisy-gd} reduces (in expectation over minibatching) to the per-sample gradient $\nabla_\theta \ell\!\big(f(x_f;\theta_k),y_f\big)$, so the bound is controlled by the \emph{squared gradient norms} of $z_f$ along training.
% Consequently, large gradients—either at initialization (the original model in our unlearning setting) or during subsequent unlearning steps—inflate the KL term and thus the privacy risk attributable to $z_f$.
% Empirically, this trend aligns with Figure~\ref{fig:norm_vs_privacy_risk}, where samples with larger gradient norms exhibit higher risk under our U-LiRA evaluation.

% \textcolor{red}{Unlike "initilization metter" we should preserve the converged model on retain-set}

\paragraph{Motivation II: Gradient norm and curvature amplify leakage.}
Recent evidence suggests that the per-sample gradient trajectory is a strong predictor of privacy vulnerability. \cite{tobaben2024understanding} show that training examples that accumulate larger gradient norms during optimization are significantly more prone to MIA, reflecting the intuition from differential privacy that each update’s privacy loss scales with gradient magnitude. Complementing this, \cite{ravikumar2024curvature} demonstrate that curvature around training samples—captured via local sharpness of the loss—serves as a reliable discriminator between members and non-members, with sharper regions implying higher membership exposure. These findings aligns with theoretical analyses such as \cite{ye2023initialization}, who prove that large per-sample gradients at initialization inflate the KL divergence between neighboring training trajectories, directly increasing the sample’s privacy risk. Motivated by this, we hypothesize that approximate unlearning inherits the same vulnerability: samples with higher gradient norms tend to push parameters towards sharper local extrema during both training and unlearning, thereby overshooting the target update and leaving a stronger privacy footprint. Our experiments (Fig.\ref{fig:norm_vs_privacy_risk}) confirm this intuition, revealing a clear correlation between a sample’s gradient norm in the original model and its susceptibility to membership inference after unlearning.

To simultaneously address (i) the parameter–space proximity that enables differencing and (ii) the gradient–norm driver of leakage, we leverage \emph{loss-invariant symmetries} of deep networks.

% Our defense therefore targets two objectives: (i) reduce the gradient energy of forget examples and (ii) enlarge symmetry-preserving parameter displacement to weaken differencing, while keeping utility on $\mathcal{D}_{\mathrm r}$.

% \paragraph{Teleportation objective and feasible moves.}

\paragraph{Symmetry framework.}
Let $\mathcal{G}$ denote a set of symmetry transformations acting on parameters $\theta$ (and, when needed, internal representations) such that the task loss is invariant: $\mathcal{L}(X,\theta)=\mathcal{L}(g\!\cdot\!(X,\theta))$ for all $g\in\mathcal{G}$ \cite{zhao2022symmetry,zhao2023improving,armenta2023neural,simsek2021geometry}.  
A \emph{teleportation} step chooses $g$ and updates $\theta\leftarrow g\!\cdot\!\theta$, moving within the loss level set.  
In our defense, we select $g$ to reduce the gradient norm of the \textit{forget-set} while preserving utility on the retain-set:
% \begin{equation}
% \label{eq:teleport-g-opt}
% g^\star \;\in\; \arg\min_{g\in\mathcal{G}}
% \Bigg\{
% \underbrace{\sum_{(x,y)\in\mathcal{D}_{\mathrm f}}
% \left\|\nabla_\theta \ell\!\big(f(x;g\!\cdot\!\theta),y\big)\right\|_2^2}_{\text{shrink forget-set gradients}}
% \;-\;
% \beta\,\underbrace{\|g\!\cdot\!\theta-\theta\|_2^2}_{\text{increase parameter dispersion}}
% \Bigg\}
% \quad\text{s.t.}\quad
% \ell_{\mathrm r}(g\!\cdot\!\theta\,|\,\mathcal{D}_{\mathrm r}) \le \ell_{\mathrm r}(\theta\,|\,\mathcal{D}_{\mathrm r})+\varepsilon,
% \end{equation}

\begin{equation}
\label{eq:teleport-g-opt}
g^\star \in \arg\min_{g\in\mathcal{G}}
\Big\{ \underbrace{\textstyle\sum_{(x,y)\in\mathcal{D}_{\mathrm f}}
\|\nabla_\theta \ell(f(x;g\cdot\theta),y)\|_2^2}_{\text{shrink forget-set gradients}}
\;-\; \beta\,\underbrace{\|g\cdot\theta-\theta\|_2^2}_{\text{increase parameter dispersion}} \Big\}
\end{equation}
\begin{equation*}
\text{s.t.}\quad
\ell_{\mathrm r}(g\cdot\theta\,|\,\mathcal{D}_{\mathrm r})
\;\le\; \ell_{\mathrm r}(\theta\,|\,\mathcal{D}_{\mathrm r})+\varepsilon.
\end{equation*}

% with trade-off $\beta\!\ge\!0$ and a small tolerance $\varepsilon\!\ge\!0$.  
% The first term directly operationalizes Motivation~II by driving down the squared gradient norms of forget examples; the dispersion term injects symmetry-preserving randomness and pushes the post-teleport parameters away from the original (Motivation~I), while the constraint preserves retain performance.

with trade-off $\beta\!\ge\!0$ and tolerance $\varepsilon\!\ge\!0$.  
The first term reduces squared gradient norms of forget examples (Motivation~II); 
the dispersion term adds symmetry-preserving randomness, displacing parameters from $\theta^{\mathrm{org}}$ (Motivation~I); 
the constraint preserves retain performance.

WARP operates on an abstract prediction-preserving symmetry map $T_\phi$, and any such symmetry family can instantiate the framework. In practice, we use two concrete realizations—the retain–null-space projection introduced in the next paragraph, and the change-of-basis teleportation detailed in Appendix~\ref{sec:appendix-alternative}—to illustrate this generality.
To complement this algorithmic view, Appendix~\ref{sec:appendix-theoretical} develops teleportation-aware information-theoretic bounds on gradient-based reconstruction, showing how injecting symmetry-induced noise via $T_\phi$ expands the symmetry orbit and provably increases the expected reconstruction error for attackers observing $(\theta^{\mathrm{org}}, \theta^{u})$.

\paragraph{Primary instantiation: teleportation with retain null-space projection.}
\label{pargraph:primary-instansiation}
We first describe one convenient way to instantiate $T_\phi$ using retain–null-space projections~\cite{wu2025teleportation}.
To optimize \eqref{eq:teleport-g-opt} efficiently on modern architectures without explicit group actions, we adopt \emph{teleportation with input null-space gradient projection} \cite{wu2025teleportation} and instantiate it using the recent projector formulation that keeps updates on the loss-invariant level set by per-layer projections onto the input null space (thus leaving the task loss unchanged up to numerical error).  
Concretely, define the \emph{teleportation loss}
\[
\mathcal{L}_{\mathrm{tel}}(\theta)
~=~
\sum_{(x,y)\in\mathcal{B}_{\mathrm f}}
\left\|\nabla_\theta \ell\big(f(x;\theta),y\big)\right\|_2^2
\;-\;
{\beta}\,\|\theta-\theta^{\mathrm{org}}\|_2^2,
\]
% where $\mathcal{B}_{\mathrm f}$ is a minibatch from $\mathcal{D}_{\mathrm f}$.  
% Let $R_\ell$ be the per-layer representation matrix built from a \emph{retain} minibatch (inputs at layer $\ell$), and $R_\ell = U_\ell \Sigma_\ell V_\ell^\top$ its thin SVD.  
% We keep the top-$k$ left singular vectors $B_\ell = U_{\ell,1:k}$ that span the significant retain subspace and define the orthogonal projector onto the residual (null) subspace by $\Pi_\ell^\perp = I - B_\ell B_\ell^\top$.  
% A teleportation step then performs the layer-wise update

where $\mathcal{B}_{\mathrm f}$ is a minibatch from $\mathcal{D}_{\mathrm f}$.  
Let $R_\ell$ be the per-layer representation matrix from a \emph{retain} minibatch (layer-$\ell$ inputs), with thin SVD $R_\ell = U_\ell \Sigma_\ell V_\ell^\top$.  
We keep the top-$k$ left singular vectors $B_\ell = U_{\ell,1:k}$ to span the retain subspace and define the orthogonal projector onto its complement $\Pi_\ell^\perp = I - B_\ell B_\ell^\top$.  
A teleportation step then applies the layer-wise update

\begin{equation}
\label{eq:teleport-update}
W_{\ell}^{\,t+1}
\;\leftarrow\;
W_{\ell}^{\,t}
-\eta_{\mathrm{tel}}\;\Pi_\ell^\perp\!\big(\nabla_{W_\ell}\mathcal{L}_{\mathrm{tel}}(\theta^{\,t})\big)
\end{equation}

which (i) \emph{reduces} the forget-set gradient norms by descending on $\mathcal{L}_{\mathrm{tel}}$,  
(ii) \emph{preserves} the function on the retain-set by restricting motion to the retain-orthogonal subspace. 
% and  
% (iii) \emph{randomizes} parameters through symmetry-equivalent moves with a tunable dispersion term.
The projection operator in~\eqref{eq:teleport-update} corresponds to the input-null-space projector. This is implemented by subtracting the component in the subspace of the core gradient, leaving only the residual for the teleport step.
% The projection operator in \eqref{eq:teleport-update} corresponds to the input-null-space projector is simply implemented by subtracting the component in the significant representation (core gradient) subspace, leaving only the residual component for the teleport step.  

% This implementation follows the “Teleport With Null Space Gradient Projection” framework, which provides fast, architecture-agnostic teleportation while keeping loss nearly invariant and enabling SVD-based control of the moved subspace.

% \paragraph{Retain-aligned subspace selection.}
% To align the invariance with utility preservation, we compute $B_\ell$ \emph{only from retain data}:  
% if $R_\ell(\mathcal{D}_{\mathrm r}) = [\phi_\ell(x)]_{x\in\mathcal{B}_{\mathrm r}}$ stacks layer-$\ell$ inputs for a retain minibatch $\mathcal{B}_{\mathrm r}$, then:

% \begin{equation}
% R_\ell(\mathcal{D}_{\mathrm r}) = U_\ell \Sigma_\ell V_\ell^\top,\quad
% B_\ell = U_{\ell,1:k},\quad
% \Pi_\ell^\perp = I - B_\ell B_\ell^\top,
% \end

To align the invariance with utility preservation, we compute $B_\ell$ \emph{only from retain data}. 
Let $R_\ell(\mathcal{D}_{\mathrm r}) = [\phi_\ell(x)]_{x \in \mathcal{B}_{\mathrm r}}$ denote the matrix formed by stacking the layer-$\ell$ inputs for a retain minibatch $\mathcal{B}_{\mathrm r}$. Then:

\begin{equation}
R_\ell(\mathcal{D}_{\mathrm r}) = U_\ell \Sigma_\ell V_\ell^\top,\qquad
B_\ell = U_{\ell,1:k},\qquad
\Pi_\ell^\perp = I - B_\ell B_\ell^\top .
\end{equation}

% We set $k$ to capture a fixed fraction of retain variance (typically $90\%\text{–}95\%$) and use the resulting projectors in \eqref{eq:teleport-update}. This constrains each teleport step to the retain-orthogonal subspace, which keeps predictions on $\mathcal{D}_{\mathrm r}$ stable while selectively suppressing gradient energy on $\mathcal{D}_{\mathrm f}$. Because $\Pi_\ell^\perp$ removes directions spanned by retain representations, there exist choices of rank $k$ and step size $\eta_{\mathrm{tel}}$ such that
We set $k$ to capture a fixed fraction of retain variance (typically $95\%\text{–}99\%$) and apply the resulting projectors in \eqref{eq:teleport-update}. This confines each teleport step to the retain-orthogonal subspace, stabilizing predictions on $\mathcal{D}_{\mathrm r}$ while suppressing gradient energy on $\mathcal{D}_{\mathrm f}$. Since $\Pi_\ell^\perp$ removes directions spanned by retain representations, suitable choices of rank $k$ and step size $\eta_{\mathrm{tel}}$ ensure that
\[
\big|\ell_{\mathrm r}(g\!\cdot\!\theta\,|\,\mathcal D_{\mathrm r})-\ell_{\mathrm r}(\theta\,|\,\mathcal D_{\mathrm r})\big|\;\le\;\varepsilon,
\]
% which matches the constraint below \eqref{eq:teleport-g-opt}; empirically, the prediction drift on $\mathcal D_{\mathrm r}$ is within numerical tolerance. 
which matches the constraint below \eqref{eq:teleport-g-opt}; in practice, prediction drift on $\mathcal D_{\mathrm r}$ remains within numerical tolerance (see Appendix~\ref{sec:teleport-hp-sensitivity} for hyperparameter sensitivity).
% An alternative instantiation of symmetry is detailed in Appendix~\ref{sec:appendix-alternative}.
To underline that WARP is not tied to retain–null-space projections, Appendix~\ref{sec:appendix-alternative} instantiates $T_\phi$ using the SVD-free change-of-basis symmetries introduce in~\cite{armenta2023neural}.

% \paragraph{Plug-and-play scope.}
% Teleportation is interleaved with the standard unlearning update \eqref{eq:first-order}, without requiring training-time per-sample gradients or stored statistics. Specifically, the teleportation update \eqref{eq:teleport-update} is applied at scheduled intervals $t \in K \subset {0,\ldots,T{-}1}$ (e.g., every $S$ iterations), which maintains a low gradient norm on the \textit{forget-set} $\mathcal{D}_{\mathrm f}$ throughout unlearning while preserving retention performance. The complete algorithm is provided in Appendix~\ref{sec:appendix-teleport}.

\paragraph{Plug-and-play scope.}
Teleportation is interleaved with the standard unlearning update \eqref{eq:first-order}, requiring no training-time per-sample gradients or stored statistics. The update \eqref{eq:teleport-update} is applied at intervals $t \in K \subset {0,\ldots,T{-}1}$ (e.g., every $S$ steps), keeping \textit{forget-set} gradient norms low while preserving retention performance. The full algorithm appears in Appendix~\ref{sec:appendix-teleport}.

\section{Experiments}
\label{sec:experiments}

We now empirically evaluate the proposed teleportation-based defense across multiple unlearning algorithms, datasets, and model architectures. 
Our experiments are designed to answer the following research questions:  
(i) How vulnerable are state-of-the-art unlearning algorithms to privacy attacks under both black-box and white-box threat models?
(ii) To what extent does teleportation reduce membership and reconstruction leakage without sacrificing utility on the retain-set?

% \textcolor{red}{Guidance of following Exp subsection - Especially the DRA}

\paragraph{Experimental Setup. }
\label{sec:exp-setup}
% We conduct experiments on two image classification benchmarks: CIFAR-10 and ImageNet-1K.  
% For CIFAR-10, we use ResNet-18 as the base architecture, while for ImageNet we evaluate both ViT-B/16 to capture convolutional and transformer-based models.  
% All models are trained with SGD using standard data augmentation (random crop and horizontal flip).  
% Following prior work on machine unlearning~\cite{kurmanji2023towards,chundawat2023can,thudi2022unrolling}, we construct forget sets $\mathcal{D}_f$ by randomly sampling $n_f$ examples per class, with retain sets $\mathcal{D}_r$ consisting of the remaining training samples.  
% Unless otherwise noted, we report results with $n_f=500$ for CIFAR-10 and $n_f=5{,}000$ for ImageNet, corresponding to approximately $1\%$ of the respective training sets—a standard experimental setup in prior unlearning studies~\cite{kurmanji2023towards,chundawat2023can}.

We conduct experiments on CIFAR-10, Tiny-ImageNet, and ImageNet-1K.  
On CIFAR-10 we use ResNet-18, while on ImageNet we evaluate ViT-B/16, covering both convolutional and transformer models.  
All models are trained with SGD and standard augmentation.  
Following prior work~\cite{kurmanji2023towards,chundawat2023can}, forget sets $\mathcal{D}_f$ are sampled as roughly $1\%$ of training data per class, with retain sets $\mathcal{D}_r$ comprising the rest.   

% \textcolor{red}{Menton finetuning schema}

\paragraph{Baselines. }
\label{sec:methods-baselines}
We benchmark six representative unlearning algorithms—\textsc{NegGrad+}~\cite{kurmanji2023towards}, \textsc{SCRUB}~\cite{kurmanji2023towards}, \textsc{SalUn}~\cite{fan2023salun}, \textsc{PGU}~\cite{hoang2024learn}, \textsc{BadTeacher}~\cite{chundawat2023can}, and \textsc{SRF-ON}~\cite{huang2024unified}—covering paradigms of gradient ascent, regularization, saliency, projection, and distillation. Full details are in Appendix~\ref{sec:appendix-baselines}.

\subsection{Overview Effectiveness of WARP}

\begin{figure}[t]
    \centering
    % ---------------- First row ----------------
    \begin{minipage}{0.32\linewidth}
        \centering
        \includegraphics[width=\linewidth]{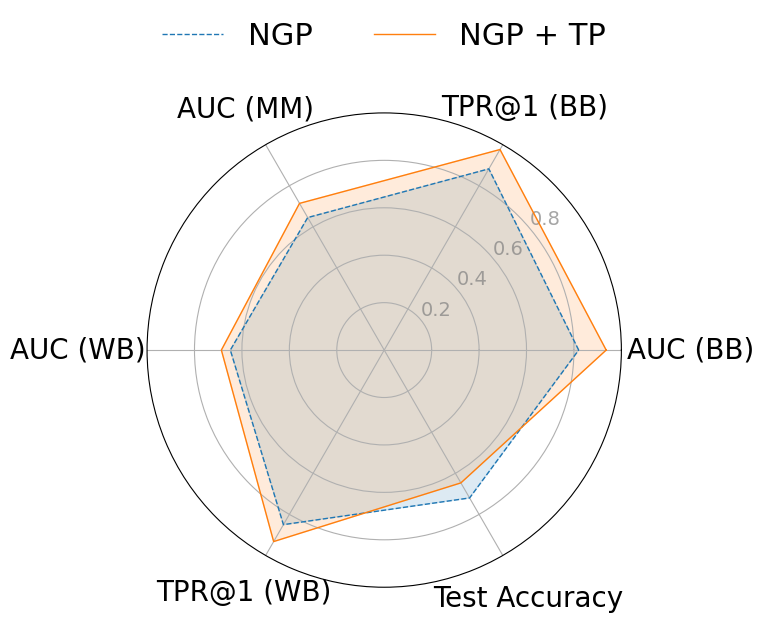}
        \vspace{0.3em}
        \textbf{NGP}
        \label{fig:npg_spider}
    \end{minipage}\hfill
    \begin{minipage}{0.32\linewidth}
        \centering
        \includegraphics[width=\linewidth]{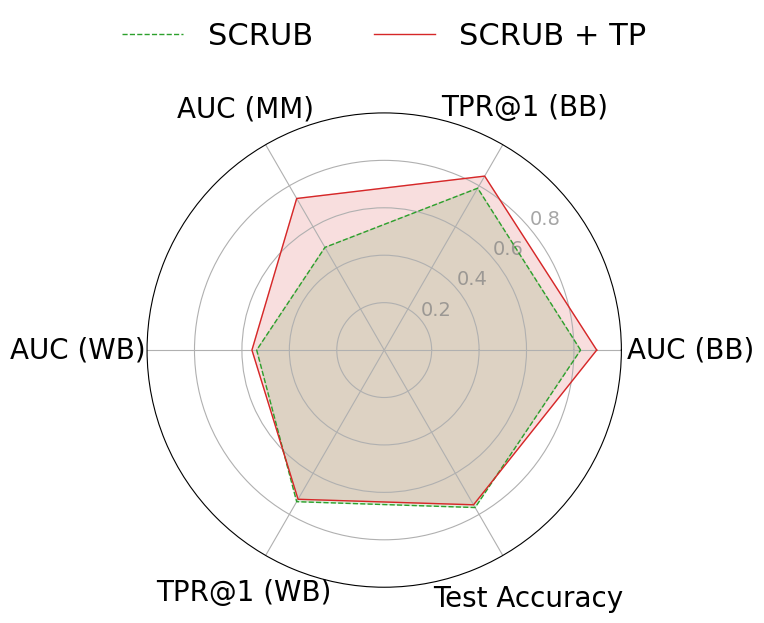}
        \vspace{0.3em}
        \textbf{SCRUB}
        \label{fig:scrub_spider}
    \end{minipage}\hfill
    \begin{minipage}{0.32\linewidth}
        \centering
        \includegraphics[width=\linewidth]{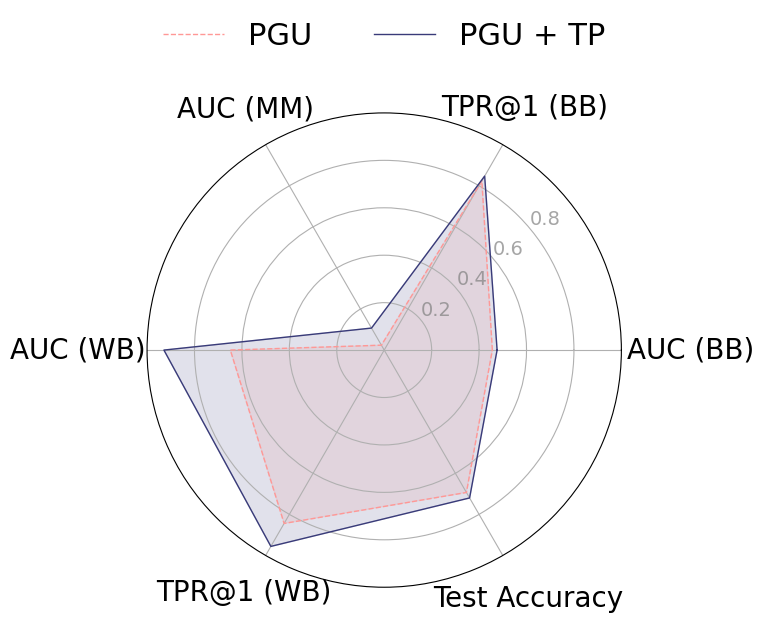}
        \vspace{0.3em}
        \textbf{PGU}
        \label{fig:pgu_spider}
    \end{minipage}

    % ---------------- Second row ----------------
    \vspace{0.8em} % space between rows
    \begin{minipage}{0.32\linewidth}
        \centering
        \includegraphics[width=\linewidth]{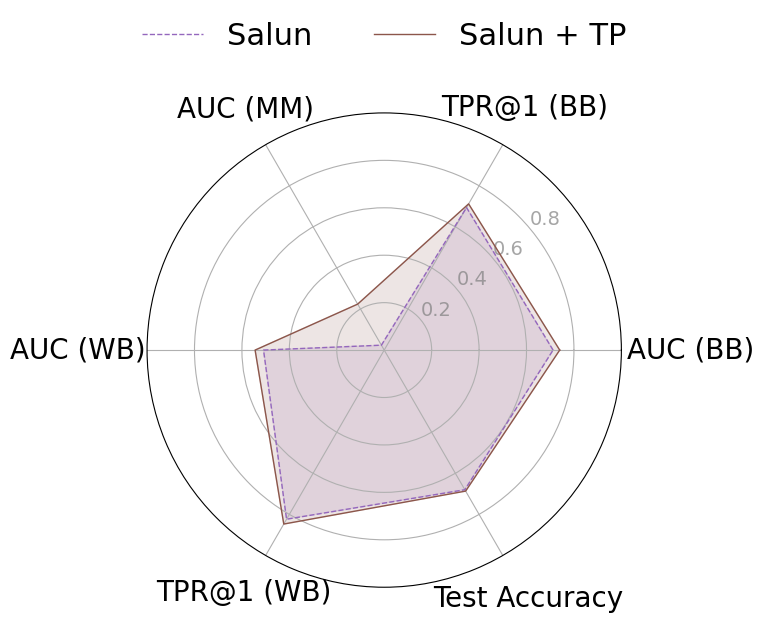}
        \vspace{0.3em}
        \textbf{Salun}
        \label{fig:salun_spider}
    \end{minipage}\hfill
    \begin{minipage}{0.32\linewidth}
        \centering
        \includegraphics[width=\linewidth]{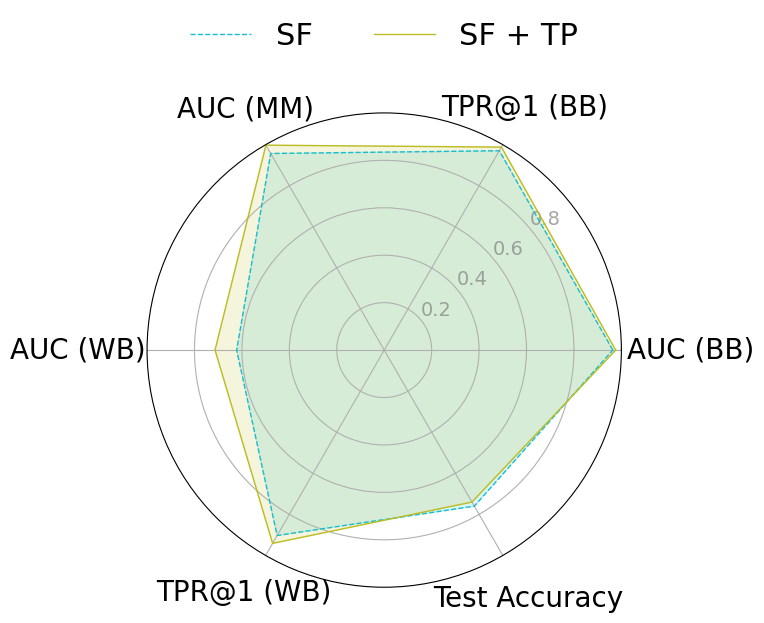}
        \vspace{0.3em}
        \textbf{SF}
        \label{fig:sf_spider}
    \end{minipage}\hfill
    \begin{minipage}{0.32\linewidth}
        \centering
        \includegraphics[width=\linewidth]{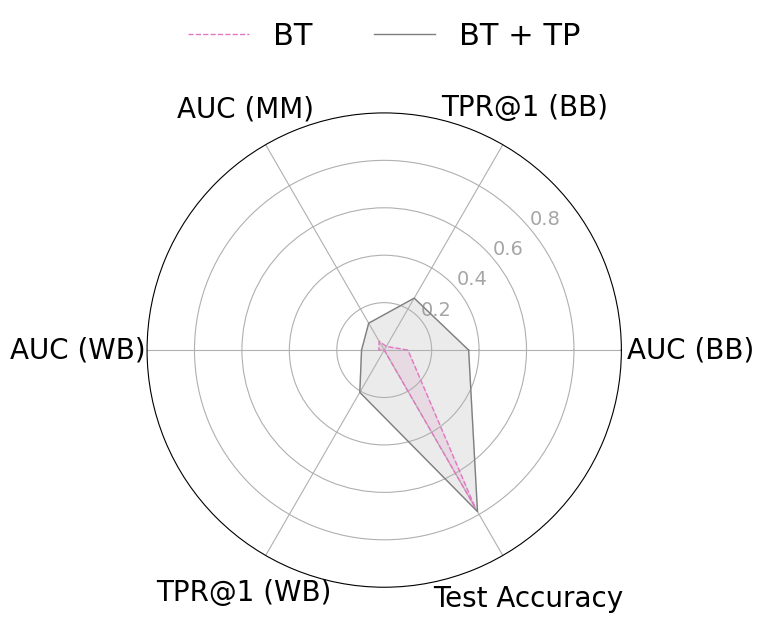}
        \vspace{0.3em}
        \textbf{BT}
        \label{fig:bt_spider}
    \end{minipage}

    \caption{Comparison of unlearning vs. teleportation across six unlearning methods.}
    \label{fig:spider_six_methods}
\end{figure}

Figure~\ref{fig:spider_six_methods} summarizes privacy and utility across six unlearning methods with and without our plug-in defense. Each radar chart reports black-box membership inference risk (AUC and TPR at low FPR), accuracy on the most-memorized subset, white-box membership inference risk (AUC and TPR at low FPR), and standard test accuracy. The most-memorized subset is selected following our U-LiRA protocol in Sec.~\ref{sec:exp-ulira}, motivated by prior findings that highly memorized samples carry elevated unlearning risk~\cite{naderloui2025rectifying}. For visualization, all metrics are min–max normalized across methods. Privacy metrics in which lower is better are inverted by plotting $1-\text{metric}$, so that larger polygons correspond to stronger privacy, while higher test accuracy remains preferable.

Three key observations emerge. First, no unlearning algorithm dominates across all axes. For instance, SF performs well under black-box auditing but is weaker under white-box auditing and in test accuracy, illustrating the necessity of evaluating under both threat models. Second, algorithms that appear robust under black-box evaluation such as NGP and SF still exhibit substantial leakage under our white-box test, underscoring the importance of auditing with gradient- or weight-based evidence. Third, adding our symmetry-based teleportation module, instantiated via retain null-space projection, consistently improves privacy across both black-box and white-box metrics while maintaining utility. In some cases, such as BT and SF, teleportation even improves test accuracy. The only noticeable accuracy drop occurs for NGP (about one percentage point), for which we provide a detailed privacy–utility trade-off analysis in Appendix~\ref{sec:appendix-exp-ngp-tradeoff}. 
The runtime overhead of teleportation is analyzed separately in Appendix~\ref{sec:appendix-runtime}, and Appendix~\ref{sec:teleport-hp-sensitivity} presents ablations showing that WARP’s performance does not hinge on fragile choices of teleportation hyperparameters. Overall, these results demonstrate that the proposed defense empirically reduces attack success consistently and effectively across a diverse set of unlearning algorithms and threat models.
For completeness, we also compare WARP against the strongest noise-based alternative, namely projected DP–Langevin unlearning~\cite{chien2024langevin}, using its formally calibrated update rule; the full comparison is provided in Appendix~\ref{sec:dp-comparison}.

\subsection{U-LiRA (Black-box)}
\label{sec:exp-ulira}

\begin{table}[t]
\centering
\small
\setlength{\tabcolsep}{2pt}
% \caption{\textbf{Privacy (Black-box) with and without WARP.} 
% Columns report black-box membership risk on \emph{all forget samples} (AUC, TPR@0.1\%, TPR@1\%, TPR@5\%) and on the \emph{most–memorized} subset (top 1\%), as well as the final test accuracy. 
% For each method we show the baseline, the WARP model, and the \emph{Improvement (\%)} computed as \emph{advantage reduction over random}.}
% \caption{\textbf{Privacy (Black-box) with and without WARP.} 
% Columns report black-box membership risk on \emph{all forget samples} (AUC, TPR@0.1\%, TPR@1\%, TPR@5\%) and on the \emph{most–memorized} subset (top 1\%), as well as the final test accuracy. 
% For each method we show the baseline, the WARP model, and the \emph{Improvement (\%)} computed as \emph{advantage reduction over random}: 
% for AUC computed as $(\text{AUC}_{\text{base}}-\text{AUC}_{\text{warp}})/(\text{AUC}_{\text{base}}-0.5)$; 
% for TPR@$\alpha$ computed as $(\text{TPR}_{\text{base}}-\text{TPR}_{\text{warp}})/(\text{TPR}_{\text{base}}-\alpha)$ with $\alpha\in\{0.001,0.01,0.05\}$.}
% \caption{\textbf{Privacy (Black-box) with and without WARP.} 
% Metrics on \emph{all forget samples} (AUC, TPR@0.1/1/5\%), the \emph{most–memorized} 1\%, and test accuracy. 
% For each method we show baseline, WARP, and \emph{Improvement (\%)} via advantage reduction: 
% AUC $(\text{AUC}_{\text{base}}-\text{AUC}_{\text{warp}})/(\text{AUC}_{\text{base}}-0.5)$; 
% TPR@$\alpha$ $(\text{TPR}_{\text{base}}-\text{TPR}_{\text{warp}})/(\text{TPR}_{\text{base}}-\alpha)$, $\alpha\!\in\!\{0.001,0.01,0.05\}$.}
\caption{\textbf{Privacy (Black-box) with and without WARP.} 
Reported are risks on \emph{all forget samples} and the \emph{most–memorized} 1\% (AUC, TPR@0.1/1/5\%), plus test accuracy. 
Each row shows baseline, WARP, and relative improvement (\%).}
\vspace{2mm}
\resizebox{0.75\linewidth}{!}{%
\begin{tabular}{lccccccccc}
\toprule
& \multicolumn{4}{c}{All samples (BB)} & \multicolumn{4}{c}{Most–memorized (top 1\%)} & \multicolumn{1}{c}{Acc.} \\
\cmidrule(lr){2-5}\cmidrule(lr){6-9}\cmidrule(lr){10-10}
Method & AUC & TPR@0.1 & TPR@1 & TPR@5 & AUC & TPR@0.1 & TPR@1 & TPR@5 & Test \\
\midrule
NGP (base)        & 0.545 & 0.012 & 0.030 & 0.077 & 0.649 & 0.058 & 0.157 & 0.277 & \textbf{0.808} \\
\textbf{+ WARP}   & \textbf{0.516} & \textbf{0.003} & \textbf{0.014} & \textbf{0.055} & \textbf{0.598} & \textbf{0.015} & \textbf{0.082} & \textbf{0.206} & 0.797 \\
Improvement (\%)  & 64.4 & 81.8 & 80.0 & 81.5 & 34.2 & 75.4 & 51.0 & 31.3 & -5.7 \\
\midrule
SCRUB (base)      & 0.543 & 0.020 & 0.047 & 0.092 & 0.710 & 0.086 & 0.227 & 0.397 & \textbf{0.815} \\
\textbf{+ WARP}   & \textbf{0.526} & \textbf{0.015} & \textbf{0.036} & \textbf{0.078} & \textbf{0.610} & \textbf{0.041} & \textbf{0.119} & \textbf{0.213} & 0.813 \\
Improvement (\%)  & 39.5 & 26.3 & 29.7 & 33.3 & 47.6 & 52.9 & 49.8 & 53.0 & -1.1 \\
\midrule
PGU (base)        & 0.636 & 0.024 & 0.040 & \textbf{0.098} & 0.910 & 0.201 & 0.511 & 0.706 & 0.804 \\
\textbf{+ WARP}   & \textbf{0.631} & \textbf{0.018} & \textbf{0.036} & 0.104 & \textbf{0.875} & \textbf{0.160} & \textbf{0.431} & \textbf{0.663} & \textbf{0.808} \\
Improvement (\%)  & 3.7  & 26.1 & 13.3 & -12.5 & 8.5  & 20.5 & 16.0 & 6.6 & +2.0 \\
\midrule
Salun (base)      & 0.572 & 0.020 & 0.062 & 0.121 & 0.910 & 0.129 & 0.321 & 0.520 & 0.802 \\
\textbf{+ WARP}   & \textbf{0.565} & \textbf{0.019} & \textbf{0.059} & \textbf{0.113} & \textbf{0.826} & \textbf{0.107} & \textbf{0.264} & \textbf{0.487} & \textbf{0.803} \\
Improvement (\%)  & 9.7  & 5.3  & 5.8  & 11.3 & 20.5 & 17.2 & 18.3 & 7.0 & +0.5 \\
\midrule
SF (base)         & 0.509 & 0.004 & 0.015 & 0.056 & 0.518 & 0.089 & 0.034 & 0.079 & \textbf{0.814} \\
\textbf{+ WARP}   & \textbf{0.506} & \textbf{0.002} & \textbf{0.012} & \textbf{0.051} & \textbf{0.501} & \textbf{0.006} & \textbf{0.026} & \textbf{0.068} & 0.811 \\
Improvement (\%)  & 33.3 & 66.7 & 60.0 & 83.3 & 94.4 & 94.3 & 33.3 & 37.9 & -1.6 \\
\midrule
BT (base)         & 0.725 & 0.000 & 0.177 & 0.287 & 0.902 & 0.119 & 0.295 & 0.582 & 0.816 \\
\textbf{+ WARP}   & \textbf{0.661} & 0.000 & \textbf{0.137} & \textbf{0.219} & \textbf{0.865} & \textbf{0.113} & \textbf{0.275} & \textbf{0.537} & \textbf{0.818} \\
Improvement (\%)  & 28.4 & -- & 24.0 & 28.7 & 9.2 & 5.1 & 7.0 & 8.5 & +1.1 \\
\bottomrule
\end{tabular}
}
\label{tab:bb_mm_advred}
\end{table}

We evaluate our teleportation defense with U-LiRA~\cite{hayes2025inexact}, a state-of-the-art black-box unlearning auditor.
Following Deep Unlearn~\cite{cadet2024deep}, we train $T=64$ shadow models with $10$ random forget sets each.
To model a strong adaptive adversary, shadows use the same unlearning algorithm, teleportation, and hyperparameters as the target, reducing proxy–target miscalibration~\cite{cretu2023investigating}.
Details of U-LiRA appear in Appendix~\ref{sec:appendix-ulira}.
% , along with ROC curves of attack performance (Appendix~\ref{sec:appendix-ulira-roc}).  

% As emphasized in prior work~\cite{carlini2022membership}, the most informative regime is low false-positive rates (FPR), where practical attacks must operate. We therefore report not only overall AUC but also true-positive rates at 0.1\%, 1\%, and 5\% FPR (denoted TPR@0.1, TPR@1, TPR@5), which capture attacker success in the stringent low-FPR regime.
% In addition, following the RULI framework~\cite{naderloui2025rectifying}, we stratify the \textit{forget-set} by \emph{memorization} (ranked by training confidence or equivalently low training loss) and evaluate U-LiRA on the top quantile of most–memorized samples. 
% RULI shows that these highly memorized points carry elevated unlearning privacy risk. We therefore report TPR at low FPR separately on this slice in addition to the aggregate metrics.

As emphasized in prior work~\cite{carlini2022membership}, the most informative regime is low false-positive rates (FPR), where practical attacks must operate.
We therefore report AUC as well as TPR@0.1, TPR@1, and TPR@5, which capture attacker success in this stringent regime.
In addition, following RULI~\cite{naderloui2025rectifying}, we stratify the \textit{forget-set} by \emph{memorization} (ranked by training confidence) and evaluate U-LiRA on the most–memorized slice.
These points carry elevated privacy risk, so we report low-FPR TPR on this subset alongside aggregate metrics.

% Table~\ref{tab:bb_mm_advred} shows that adding our teleportation plug-in reduces black-box membership leakage for every unlearning method, both on the full \textit{forget-set} and on the most-memorized slice, with the largest relative gains appearing at low FPR where practical attacks operate. For instance, NGP’s TPR@1 roughly halves (0.030→0.014), SCRUB’s AUC on the memorized slice drops by 0.10 (0.710→0.610), and SF’s memorized-slice AUC moves to within 0.001 of random guessing (0.501). Improvements in TPR at low FPR are often substantial even when aggregate AUC changes are modest, indicating that teleportation suppresses the high-confidence tails that attackers exploit. Methods that look competitive on average can still leak on the memorized slice, while our defense frequently pushes that slice close to random without degrading utility. Teleportation is most effective on the TPR@0.1 and TPR@1 columns because the retain-null-space projection reduces forget-set gradient energy and shrinks extreme logit margins, directly weakening the rare high-confidence signals that drive low-FPR success. Overall, symmetry-based teleportation consistently lowers attack success across diverse unlearning algorithms and is especially effective in the stringent low-FPR regime.

Table~\ref{tab:bb_mm_advred} shows that adding our teleportation plug-in reduces black-box membership leakage across all methods, on both the full \textit{forget-set} and the most-memorized slice, with the largest gains at low FPR.
For example, NGP’s TPR@1 nearly halves (0.030→0.014), SCRUB’s memorized-slice AUC drops by 0.10 (0.710→0.610), and SF’s AUC falls to near-random (0.501).
Low-FPR TPR gains are often large even when aggregate AUC shifts are modest, showing that teleportation suppresses the high-confidence tails attacks exploit.
Some methods remain leaky on memorized points, but teleportation frequently drives this slice close to random without hurting accuracy.
Its impact is strongest on TPR@0.1 and TPR@1, as retain-null-space projection reduces forget gradients and shrinks extreme margins, weakening the rare signals enabling low-FPR success.

\subsection{White-box MIA}
\label{sec:exp-whitebox-mia}

\begin{figure}[t]
    \centering
    \begin{minipage}[c]{0.52\linewidth}
        \centering
        \includegraphics[width=\linewidth]{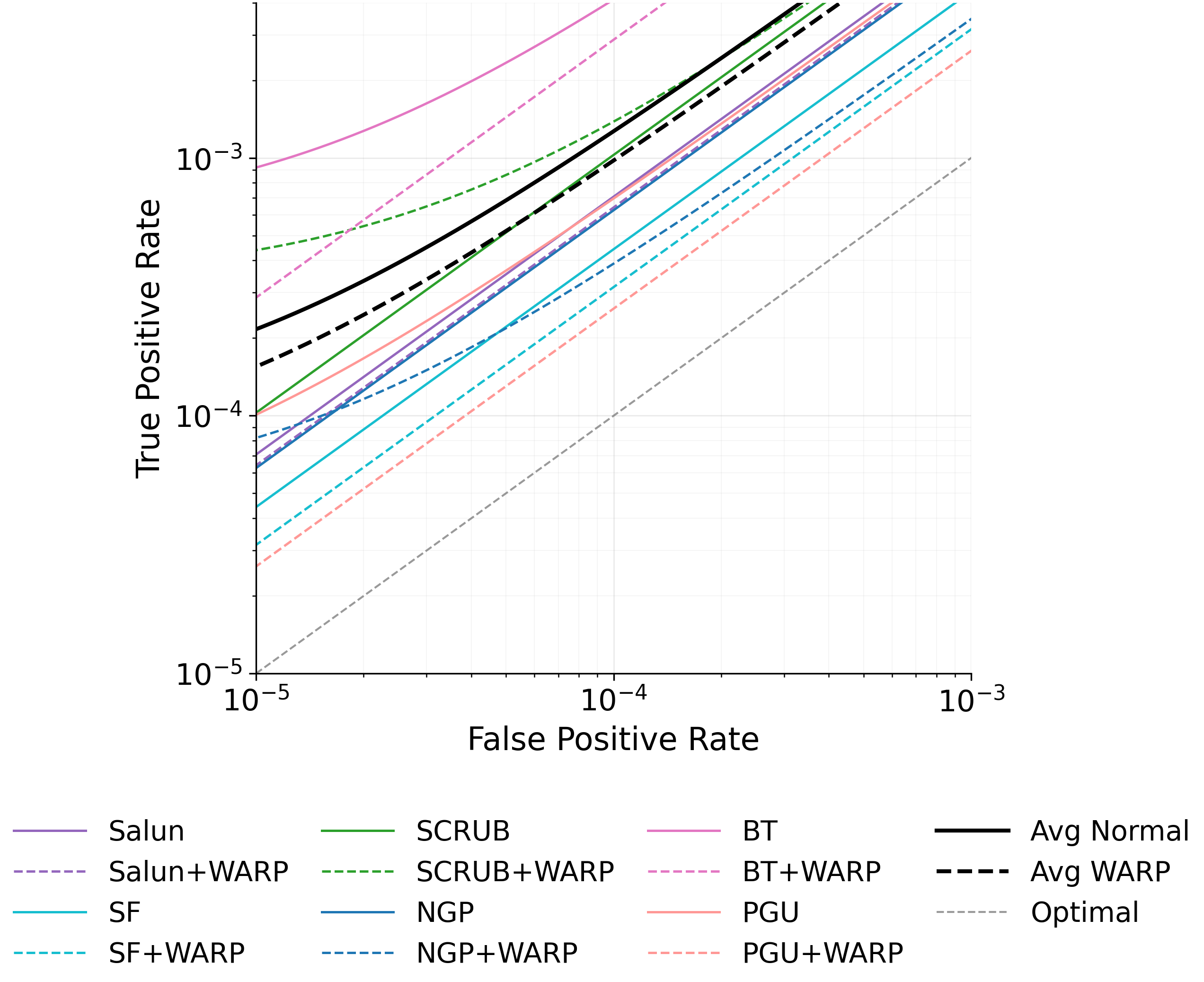}
    \end{minipage}\hfill
    \begin{minipage}[c]{0.46\linewidth}
        \centering
        \renewcommand{\arraystretch}{0.9} % tighter rows
        \setlength{\tabcolsep}{3pt}       % tighter cols
        \resizebox{\linewidth}{!}{%
        \begin{tabular}{lcccc}
        \toprule
        Method & AUC & TPR@0.1 & TPR@1 & TPR@5 \\
        \midrule
        NGP (base)   & 0.642 & 0.004 & 0.034 & 0.139 \\
        \textbf{+ WARP}       & \textbf{0.614} & \textbf{0.002} & \textbf{0.021} & \textbf{0.097} \\
        Improvement (\%) & 17.0 & 50.0 & 40.6 & 34.2 \\
        \midrule
        SCRUB (base) & 0.700 & 0.011 & 0.102 & 0.287 \\
        \textbf{+ WARP}       & \textbf{0.657} & \textbf{0.006} & \textbf{0.061} & \textbf{0.193} \\
        Improvement (\%) & 14.3 & 54.5 & 42.5 & 33.5 \\
        \midrule
        PGU (base)   & 0.659 & 0.007 & 0.064 & 0.215 \\
        \textbf{+ WARP}       & \textbf{0.533} & \textbf{0.002} & \textbf{0.025} & \textbf{0.085} \\
        Improvement (\%) & 92.9 & 83.3 & 64.5 & 65.5 \\
        \midrule
        Salun (base) & 0.721 & 0.008 & 0.069 & 0.230 \\
        \textbf{+ WARP}       & \textbf{0.705} & \textbf{0.006} & \textbf{0.062} & \textbf{0.214} \\
        Improvement (\%) & 9.5 & 33.3 & 10.1 & 7.0 \\
        \midrule
        SF (base)    & 0.670 & 0.005 & 0.043 & 0.161 \\
        \textbf{+ WARP}       & \textbf{0.629} & \textbf{0.003} & \textbf{0.030} & \textbf{0.124} \\
        Improvement (\%) & 29.2 & 50.0 & 34.9 & 23.2 \\
        \midrule
        BT (base)    & 0.938 & 0.037 & 0.346 & 0.809 \\
        \textbf{+ WARP}       & \textbf{0.907} & \textbf{0.028} & \textbf{0.279} & \textbf{0.684} \\
        Improvement (\%) & 49.2 & 25.7 & 19.4 & 18.4 \\
        \bottomrule
        \end{tabular}}
    \end{minipage}
%     \caption{\textbf{White-box privacy with and without WARP.}
% Results from the same Gaussian gradient–diff test on 640 unlearned models (as in U\hbox{-}LiRA). 
% ROC curves (left) and AUC/TPRs (right) are shown; complete ROC plots over the full FPR range appear in the Appendix~\ref{sec:appendix-fullWBROC}.}
\caption{\textbf{White-box privacy with and without WARP.}
Gaussian gradient–diff test on 640 unlearned models. % shortened: removed "Results from the same" and "(as in U-LiRA)"
ROC curves (left) and AUC/TPRs (right); full ROC plots are in Appendix~\ref{sec:appendix-fullWBROC}. % shortened: compressed phrasing
}
    \label{fig:wb_privacy}
\end{figure}

We evaluate the Gaussian gradient–difference test of Section~\ref{sec:appendix-glir} under the setup of Section~\ref{sec:exp-setup}, using ResNet-18 on CIFAR-10 and ViT-B/16 on Tiny-ImageNet (full ViT in Appendix~\ref{sec:appendix-wevit}). For the null background we draw $m{=}1000$ non-members from $\mathcal{D}_{\text{test}}$, estimate $(\hat\mu,\hat\Sigma)$ with ridge $\lambda{=}10^{-3}$, and restrict the test to the top-10\% most-variant $\Delta(b)$ coordinates. Figure~\ref{fig:wb_privacy} shows ROC curves with and without teleportation (log–log for low-FPR). Across methods, teleported variants shift toward chance ($\text{TPR}=\text{FPR}$) and flatten between $10^{-5}$–$10^{-2}$ FPR, suppressing high-confidence tails. The strongest effect appears for \textsc{BT} and \textsc{PGU}, which show the largest AUC drops, while \textsc{NGP}, \textsc{SF}, and \textsc{Salun} show smaller but consistent shifts. An exception is \textsc{SCRUB}, where teleportation lowers ROC above $10^{-3}$ FPR but slightly raises TPR at $<10^{-3}$, due to knowledge distillation interacting with symmetry moves that amplify high-leverage directions. Overall, null-space teleportation reduces white-box evidence at low FPR, with a narrow corner case for \textsc{SCRUB}.

\subsection{Reconstruction Attack Results}
\label{sec:exp-recon}

% Effect of teleportation defense, with grouped rows and separators
\begin{table}[t]
\centering
\caption{\textbf{Effect of teleportation defense} on reconstruction (ImageNet-1K, ResNet-18, NGP). 
}
\label{tab:recon_teleport}
\setlength{\tabcolsep}{4pt}
\footnotesize
\resizebox{\columnwidth}{!}{
\begin{tabular}{lcccccc}
\toprule
\textbf{Variant} & \textbf{PSNR (dB)} $\uparrow$ & \textbf{LPIPS (VGG)} $\downarrow$ & \textbf{LPIPS (Alex)} $\downarrow$ & \textbf{SSIM} $\uparrow$ & \textbf{Test MSE} $\downarrow$ & \textbf{Feat MSE} $\downarrow$ \\
\midrule
Ours (normal unlearning) 
& $10.74\pm0.31$ 
& $0.56\pm0.013$ 
& $0.34\pm0.015$ 
& $0.12\pm0.008$ 
& $0.10\pm0.007$ 
& $5.39\pm0.50$ \\
Ours + \textit{WARP} 
& $7.38\pm0.40$ 
& $0.68\pm0.01$ 
& $0.46\pm0.02$ 
& $0.08\pm0.006$ 
& $0.21\pm0.02$ 
& $11.28\pm1.89$ \\
\midrule
\textit{Improvement of Defense (\%)} 
& $+45.5$ 
& $+21.2$ 
& $+26.1$ 
& $+31.6$ 
& $+52.4$ 
& $+52.2$ \\
\bottomrule
\end{tabular}
}
\end{table}

% \begin{figure}[t]
% \centering
% \setlength{\tabcolsep}{3pt}
% \renewcommand{\arraystretch}{1.0}
% \begin{tabular}{c c c @{\hskip 8pt}|@{\hskip 8pt} c c c}
% \textbf{Original} & \textbf{NGP} & \textbf{NGP+WARP} & 
% \textbf{Original} & \textbf{NGP} & \textbf{NGP+WARP} \\
% \includegraphics[width=0.12\linewidth]{images/DRA/Original/1000_Original.png} &
% \includegraphics[width=0.12\linewidth]{images/DRA/NGP_PGFID/1000_NGP_PGFID.png} &
% \includegraphics[width=0.12\linewidth]{images/DRA/TNGP_PGIFD/1000_TNGP_PGFID.png} &
% \includegraphics[width=0.12\linewidth]{images/DRA/Original/3500_Original.png} &
% \includegraphics[width=0.12\linewidth]{images/DRA/NGP_PGFID/3500_NGP_PGFID.png} &
% \includegraphics[width=0.12\linewidth]{images/DRA/TNGP_PGIFD/3500_TNGP_PGFID.png} \\
% \includegraphics[width=0.12\linewidth]{images/DRA/Original/6200_Original.png} &
% \includegraphics[width=0.12\linewidth]{images/DRA/NGP_PGFID/6200_NPG_PGFID.png} &
% \includegraphics[width=0.12\linewidth]{images/DRA/TNGP_PGIFD/6200_TNGP_PGFID.png} &
% \includegraphics[width=0.12\linewidth]{images/DRA/Original/9400_Original.png} &
% \includegraphics[width=0.12\linewidth]{images/DRA/NGP_PGFID/9400_NGP_PGFID.png} &
% \includegraphics[width=0.12\linewidth]{images/DRA/TNGP_PGIFD/9400_TNGP_PGFID.png} \\
% \end{tabular}
% \caption{\textbf{Reconstructions under NGP vs. NGP+WARP.}} 
% \label{fig:recon_ngp_tp}
% \end{figure}

\begin{figure}[t]
\centering
\setlength{\tabcolsep}{1pt} % reduce column padding
\renewcommand{\arraystretch}{0.9} % reduce row spacing
\resizebox{0.7\linewidth}{!}{% scale the whole figure to 90% of text width
\begin{tabular}{c c c @{\hskip 4pt}|@{\hskip 4pt} c c c}
\textbf{Original} & \textbf{NGP} & \textbf{NGP+WARP} & 
\textbf{Original} & \textbf{NGP} & \textbf{NGP+WARP} \\
\includegraphics[width=0.12\linewidth]{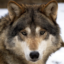} &
\includegraphics[width=0.12\linewidth]{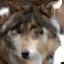} &
\includegraphics[width=0.12\linewidth]{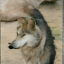} &
\includegraphics[width=0.12\linewidth]{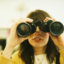} &
\includegraphics[width=0.12\linewidth]{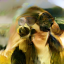} &
\includegraphics[width=0.12\linewidth]{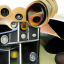} \\
\includegraphics[width=0.12\linewidth]{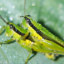} &
\includegraphics[width=0.12\linewidth]{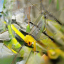} &
\includegraphics[width=0.12\linewidth]{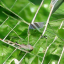} &
\includegraphics[width=0.12\linewidth]{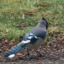} &
\includegraphics[width=0.12\linewidth]{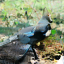} &
\includegraphics[width=0.12\linewidth]{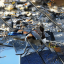} \\
\end{tabular}
}
\caption{\textbf{Reconstructions under NGP vs. NGP+WARP.}} 
\label{fig:recon_ngp_tp}
\end{figure}

% We evaluate the white-box reconstruction attack of Section~\ref{sec:method-recon} on \textbf{ImageNet-1K} with \textbf{ResNet-18}, focusing on \textbf{NGP} unlearning. 
% We reconstruct a \emph{single} forgotten example at a time and report averages over \textbf{100} uniformly sampled forgotten points. 
% For each target we use a retain minibatch of size \(|\mathcal{B}_r|=5\) in \eqref{eq:mixture}. 
% Subspace projectors are built \emph{per layer} from probe gradients: we draw \(m{=}100\) training-distribution samples to form \(G_{\mathrm{org}},G_u\), compute thin SVDs, and keep the smallest rank \(k\) that preserves \(\mathbf{90\%}\) gradient energy; we then apply \(\Pi_u^\perp\) followed by \(\Pi_{\mathrm{org}}\) layerwise to obtain the filtered target \(\tilde g_f\) in \eqref{eq:filtered-target}. 
% The attacker knows the true label \(y_f\) and optimizes \eqref{eq:final-inv} with a TV regularizer as in \cite{geiping2020inverting}. 
% Following our implementation, the matching loss operates on \emph{masked} per-layer gradients: for each layer \(i\), we retain all coordinates without compression and compute a weighted dot-product alignment between the masked trial and input gradients, following \cite{fang2023gifd}.  
% % (details in Appendix \ref{sec:appendix-rec-criteria}).

We evaluate the white-box reconstruction attack of Section~\ref{sec:method-recon} on \textbf{ImageNet-1K} with \textbf{ResNet-18}, focusing on \textbf{NGP}. 
We reconstruct a \emph{single} forgotten example and average over \textbf{100} uniformly sampled points. 
For each target we use a retain minibatch of size \(|\mathcal{B}_r|=5\). 
Subspace projectors are built per layer from probe gradients: we draw \(m{=}100\) training samples to form \(G_{\mathrm{org}},G_u\), compute thin SVDs, and keep rank \(k\) preserving \(\mathbf{90\%}\) gradient energy. 
We then apply \(\Pi_u^\perp\) and \(\Pi_{\mathrm{org}}\) layerwise to obtain the filtered target \(\tilde g_f\). 
The attacker knows the label \(y_f\) and optimizes \eqref{eq:final-inv} with a TV regularizer~\cite{geiping2020inverting}. 
The matching loss uses \emph{masked} per-layer gradients: for each layer, all coordinates are kept and a weighted dot-product alignment is computed~\cite{fang2023gifd}.

\paragraph{Effect of teleportation.} 
Table~\ref{tab:recon_teleport} and Figure~\ref{fig:recon_ngp_tp} compare reconstruction risk under standard \textsc{NGP} unlearning and its teleported variant using change-of-basis reparameterization. Despite negligible cost, this symmetry-based randomization disrupts reconstruction: even strong generative-prior attacks fail to recover meaningful features of forgotten data. Teleportation injects a symmetry component into $\Delta\theta$ that is nearly orthogonal to per-sample gradients~\cite{armenta2023neural}, reducing alignment with the true forget gradient $g_f$ and driving gradient-matching toward low signal-to-noise optima. It also undermines our subspace-filtered attack (Eq.~\ref{eq:filtered-target}), since teleportation reshapes gradient subspaces so $U_{\mathrm{org}}$ and $U_u$ overlap little, leaving the residual $\Pi_{\mathrm{org}}\Pi_u^{\perp}(-\Delta\theta/\eta)$ small and noisy. In practice, optimization collapses to the generative prior or class cues, yielding label-consistent but semantically poor reconstructions (Figure~\ref{fig:recon_ngp_tp}). Symmetry moves thus decouple updates from data-dependent directions, removing the geometric handle exploited by white-box reconstruction.
This motivates examining how teleportation reshapes the information relationship between parameters and training data (forget-set); a stronger symmetry-aware adaptive reconstruction attack is evaluated in Appendix~\ref{sec:adaptive-recon}, and Appendix~\ref{sec:appendix-theoretical} provides complementary information-theoretic bounds showing how teleportation expands the symmetry orbit and increases expected reconstruction error.

\section{Conclusion and Future Work}

Approximate unlearning provides scalability but introduces privacy risks.
We showed that adversaries with access to original and unlearned models can mount strong membership inference and reconstruction attacks.
These risks stem from two properties: parameter proximity and large forget-set gradient norms, which amplify leakage.

To counter this, we proposed \textsc{WARP}, a symmetry-based defense that interleaves teleportation with unlearning.
By exploiting network symmetries, \textsc{WARP} reduces forget-set gradient energy and displaces parameters in symmetry-preserving directions, weakening both membership and reconstruction leakage while preserving retain performance.
Across six unlearning algorithms, \textsc{WARP} improves privacy, cutting adversarial advantage by up to $64\%$ in black-box and $92\%$ in white-box settings.
We also stress the need for white-box auditing: methods seemingly robust in black-box mode (e.g., SF~\cite{huang2024unified}) still leak when gradients are exposed.
Even simple teleportation disrupts reconstruction, reducing quality by $\sim45\%$.

Our findings suggest future directions.
First, extending Langevin-based privacy analyses to practical unlearning with gradient ascent and symmetry moves is promising.
Second, recent work shows approximate unlearning leaves low-rank weight signals, reversible via re-unlearning~\cite{fan2025towards} or removed by quantization~\cite{zhang2024catastrophic}.
Exploring teleportation directly on weights may help obscure these signals and mitigate reversals.
Finally, as the study of neural network symmetries continues to evolve and more efficient estimators and richer invariance families become available, WARP can directly inherit these advances by instantiating its symmetry map with stronger or cheaper symmetry mechanisms, which further strengthens its resistance to unlearning attacks.

\subsubsection*{Acknowledgments}
The research in this paper was supported by the UKRI  Open Plus Fellowship (EP/W005271/1 Securing the Next Billion Consumer Devices on the Edge) and EU CHIST-ERA GNNs for Network Security and Privacy (GRAPHS4SEC) projects and by the Defense Advanced Research Projects Agency (DARPA), under contract W912CG23C0031.

% \newpage

\bibliographystyle{plainnat} % or another style like unsrt, IEEEtran, etc.
\bibliography{references,non_zotero_refrences} % Ensure the filename matches your .bib file (without .bib extension)

\appendix

\section{Related Work}
\label{sec:appendix-rw}

\paragraph{Approximate Unlearning.}

The removal of training samples was introduced by \citet{cao2015towards} in the context of the “right to be forgotten.” Retraining from scratch guarantees deletion but is infeasible for modern networks~\cite{vatter2023evolution}. Exact unlearning methods such as SISA~\cite{bourtoule2021machine} and Amnesiac Unlearning~\cite{graves2021amnesiac} lower cost through partitioning or selective retraining but still require storage and scale poorly~\cite{nguyen2022survey}.

Approximate unlearning directly updates the trained model to erase the \textit{forget-set}~\cite{kurmanji2023towards,chundawat2023can,golatkar2020forgetting,thudi2022unrolling}. These methods aim to match the predictive distribution of retraining while preserving retain accuracy, offering a practical forgetting–utility trade-off with large savings in computation and memory.
% \textcolor{red}{Mention the concept unlearning in the following paragraph}\\
% Within the umbrella of approximate unlearning in classification problems, different scenarios have been explored. Some methods focus on structured forget sets, such as removing an entire class \cite{chundawat2023zero,seo2025revisiting}. 
% However, most recent works address the more general and challenging task of instance-wise unlearning: forgetting an arbitrary subset of individual training points \cite{fan2024challenging,cha2024learning,zhao2024makes}.
% Many approximate unlearning algorithms rely on training-time side information (statistics) such as stored per-sample gradients \cite{qiao2024hessian,mehta2022deep} or assume the model was trained with special regimes such as adversarial robustness \cite{liu2023muter} or differential-privacy noise \cite{chien2024langevin,chien2024certified,sepahvand2025leveraging}; while effective, these assumptions add storage and computation overhead and are impractical in post-hoc settings, motivating our focus on training-agnostic unlearning.
% Hence, we focus on post-hoc, instance-wise unlearning that takes only a pre-trained classifier and a \textit{forget-set}, without stored gradients or training, time modifications—as in recent training-agnostic approaches \cite{kurmanji2023towards,thudi2022unrolling}.
Related methods target structured forget sets such as entire classes~\cite{chundawat2023zero,seo2025revisiting}, or tackle the harder instance-wise setting, where arbitrary samples must be removed~\cite{fan2024challenging,cha2024learning,zhao2024makes}. Many approaches rely on training-time side information like per-sample gradients~\cite{qiao2024hessian,mehta2022deep}, or assume specialized regimes with adversarial robustness~\cite{liu2023muter} or differential-privacy noise~\cite{chien2024langevin,chien2024certified,sepahvand2025leveraging}. While effective, these assumptions add resource overhead, limiting post-hoc use. Our focus, therefore, is training-agnostic, instance-wise unlearning that takes only a pretrained classifier and a designated \textit{forget-set}, without stored gradients or training modifications~\cite{kurmanji2023towards,thudi2022unrolling}.

\paragraph{Privacy Unlearning.}
The effectiveness of approximate unlearning is accessed by two criteria: (I) the model should maintain accuracy on non-forgotten data, and (II) its outputs on the \textit{forget-set} should be indistinguishable from those of a model with no access to it~\cite{naderloui2025rectifying}. In practice, this is evaluated using MIA~\cite{shokri2017membership,carlini2022membership}, which test whether a sample was part of training. Effective unlearning removes this membership advantage on the \textit{forget-set}.

Most prior work evaluates unlearning by comparing outputs of the unlearned model to a retrained reference on the \textit{forget-set}~\cite{cadet2024deep,kurmanji2023towards,hayes2025inexact,georgiev2024attribute,naderloui2025rectifying}.
This black-box view ignores parameters, even though in practice—such as MU on edge devices—an adversary may access both original and unlearned models.
Some studies consider this stronger setting: \cite{chen2021machine} showed that output-comparison across models can detect unlearning, while others adapted reconstruction to infer forgotten data from parameter differences\cite{salem2020updates,hu2024learn,bertran2024reconstruction}.
These works, however, are limited to toy models and simplified updates, leaving privacy risk under realistic conditions unclear.
In particular, they do not capture the robustness of recent multi-step approximate methods such as NGP or SCRUB~\cite{kurmanji2023towards,chundawat2023can}, where iterative updates with retain-set supervision weaken inversion of \textit{forget-set} gradients.
We address this gap with stronger white-box MIAs (Sec.\ref{sec:appendix-glir}) and DRAs (Sec.\ref{sec:method-recon}) tailored to realistic unlearning.

\paragraph{Neural Network Symmetry.}
Continuous symmetries in neural networks arise when transformations of the weights leave the output unchanged.
Such invariances, a byproduct of overparameterization, mean that many distinct weight configurations represent the same function~\cite{gluch2021noether}.
They appear in homogeneous activations~\cite{badrinarayanan2015symmetry,du2018algorithmic,maheri2025telesparse} and in components like softmax and batch normalization~\cite{kunin2020neural}, and have been linked to both improved optimization and generalization.
Neural teleportation leverages these symmetries by relocating parameters within the loss-invariant level set, yielding equivalent models that accelerate optimization\cite{armenta2021representation,armenta2023neural}.
Building on this idea,\cite{zhao2022symmetry} introduced symmetry teleportation, which searches for beneficial relocations while providing a framework for analyzing symmetry-induced minima.
More recently, teleportation with null-space gradient projection~\cite{wu2025teleportation} leverages the input null space: moving along projected directions leaves the function unchanged, directly aligning with the goal of teleportation.

\section{U-LiRA Algorithm}
\label{sec:appendix-ulira}

To evaluate sample-wise privacy leakage, we employ the U-LiRA attack~\cite{cadet2024deep,hayes2025inexact}, an adaptation of LiRA~\cite{carlini2022membership} to the unlearning setting. 
The attack relies on shadow models to estimate two distributions for a target sample $(x,y)$: 
(i) models trained with $(x,y)$ and subsequently unlearned using the same unlearning algorithm, and 
(ii) models trained from scratch without $(x,y)$. 
By fitting simple parametric models (e.g., Gaussians) to the outputs of these shadow ensembles, U-LiRA computes the likelihood of the target model’s output under each case and classifies membership according to a likelihood ratio test. 

Crucially, all shadow models are trained with the \emph{same unlearning algorithm and hyperparameters} as the audited model. 
This makes U-LiRA effectively an \emph{adaptive attack}, since it tailors the proxies to each specific unlearning method. 
Such alignment minimizes miscalibration between shadow and target models and is known to increase attack success~\cite{cretu2023investigating}. 
Therefore, U-LiRA serves as a strong black-box baseline for auditing privacy in unlearning.
A complete description of the algorithm can demonstrated in Algorithm~\ref{alg:ulira}.

\begin{algorithm}[H]
\caption{U-LiRA (used for auditing unlearning)}
\label{alg:ulira}
\begin{algorithmic}[1]
\Require Target model $\theta^\ast$, learning algorithm $A$, unlearning algorithm $U$, number of shadows $T$, sample $(x,y)$
\Ensure Prediction: is $(x,y)$ in the \textit{forget-set}?
\State Initialize empty lists $O \gets \{\}$ and $\hat{O} \gets \{\}$
\For{$t = 1$ to $T$}
    \State Sample dataset $D$ containing $(x,y)$
    \State Train $\theta^0 \gets A(D)$
    \State Unlearn $\theta^f \gets U(\theta^0,\{(x,y)\})$
    \State Retrain $\theta^r \gets A(D \setminus \{(x,y)\})$
    \State Record $O[t] \gets f(x;\theta^f)_y$, \ \ $\hat{O}[t] \gets f(x;\theta^r)_y$
\EndFor
\State Fit Gaussian $(\mu,\sigma^2)$ to $O$, and $(\hat{\mu},\hat{\sigma}^2)$ to $\hat{O}$
\State Compute $o^\ast \gets f(x;\theta^\ast)_y$
\State Compute likelihood ratio:
\[
p_{\mathrm{member}} = 
\frac{\mathcal{N}(o^\ast;\mu,\sigma^2)}{\mathcal{N}(o^\ast;\mu,\sigma^2)+\mathcal{N}(o^\ast;\hat{\mu},\hat{\sigma}^2)}
\]
\If{$p_{\mathrm{member}} > 0.5$}
    \State \Return ``member of training''
\Else
    \State \Return ``non-member''
\EndIf
\end{algorithmic}
\end{algorithm}

\section{White-box Gaussian Gradient--Difference Attack Algorithm}
\label{sec:appendix-glir}

Guided by the GLiR framework of \cite{leemann2023gaussian}, we formulate sample-wise MIA in the unlearning setting as a binary hypothesis test that uses \emph{both} the pre-unlearning and post-unlearning models. Let \(A\) denote the training algorithm, \(U\) the unlearning operator, \(S\) the original training set, and \(F\subseteq S\) the forget subset. For a candidate example \((x,y)\), we test
\[
\begin{aligned}
H_0 &: (x,y)\sim\mathcal D_{\text{test}}, \quad (\theta^{\mathrm{org}},\theta^{u})=\big(A(S),\,U(A(S),F)\big)\ \text{with } x\notin S,\ x\notin F,\\
H_1 &: (x,y)\in\mathcal D_{\text{forg}}, \quad (\theta^{\mathrm{org}},\theta^{u})=\big(A(S),\,U(A(S),F)\big)\ \text{with } x\in S\ \text{and}\ x\in F,
\end{aligned}
\]
i.e., under \(H_1\) the point participated in the original training and was subsequently targeted by unlearning, whereas under \(H_0\) it was never used. With white-box access, we form the gradient-difference statistic
\[
\Delta(x)\;=\;\nabla_{\theta}\,\ell\!\left(f(x;\theta^{u}),y\right)\;-\;\nabla_{\theta}\,\ell\!\left(f(x;\theta^{\mathrm{org}}),y\right)\in\mathbb{R}^{d}.
\]
Assuming access to draws from \(\mathcal D_{\text{test}}\), the adversary builds a background set \(B=\{(b_i,\tilde y_i)\}_{i=1}^{m}\sim\mathcal D_{\text{test}}^{m}\) and estimates the null (non-member) distribution of gradient differences via
\[
\hat\mu=\frac{1}{m}\sum_{i=1}^{m}\Delta(b_i),\qquad 
\hat\Sigma=\frac{1}{m-1}\sum_{i=1}^{m}\big(\Delta(b_i)-\hat\mu\big)\big(\Delta(b_i)-\hat\mu\big)^{\!\top}.
\]
Following \cite{leemann2023gaussian}, we adopt a Gaussian model for \(\Delta(x)\) under \(H_0\) and compute the whitened Mahalanobis statistic
\[
s(x)\;=\;\big(\Delta(x)-\hat\mu\big)^{\!\top}\!\big(\hat\Sigma+\lambda I\big)^{-1}\!\big(\Delta(x)-\hat\mu\big),
\]
with a small ridge \(\lambda>0\) for numerical stability. Under \(H_0\), \(s(x)\) is approximately \(\chi^2_d\)-distributed, yielding the log-\(p\)-value score
\[
A'(x,y)\;=\;-\log\!\Big(1-F_{\chi^2_{d}}\!\big(s(x)\big)\Big),
\]
and the final decision rule
\[
A(x,y)\;=\;\mathbb{I}\!\left[A'(x,y)>\tau\right],
\]
predicting \emph{forgotten} when the score exceeds threshold \(\tau\). 
Algorithm~\ref{alg:ggd-unlearn} provides the full details of the proposed attack.  

\paragraph{Relation to GLiR and unlearning specifics.}
GLiR aggregates evidence across training steps by comparing per-step sample gradients to a Gaussian background of batch gradients; our adaptation replaces the (typically unavailable) per-step trajectory with the two-model contrast \(\Delta(x)\). The geometry is unchanged: Evidence corresponds to the squared norm of the whitened difference, \(\|(\hat\Sigma+\lambda I)^{-1/2}\Delta(x)\|_2^2\). Unlike standard MIAs that query a single model, the test exploits white-box access to \(\theta^{\mathrm{org}}\) and \(\theta^{u}\) and targets the unlearning-specific alternative \(H_1\) (membership in both \(S\) and \(F\)), providing a simple and powerful auditor for residual leakage after unlearning.

\begin{algorithm}[H]
\caption{White-box Gaussian Gradient--Difference Attack for Unlearning Audit}
\label{alg:ggd-unlearn}
\begin{algorithmic}[1]
\setlength{\itemsep}{2pt}
\Require Pre-unlearning model $\theta^{\mathrm{org}}$, post-unlearning model $\theta^{u}$, candidate sample $(x,y)$, loss $\ell$, predictor $f(\cdot;\theta)$, background sampler $\mathcal{S}_{\text{test}}(m)$ that returns $m$ i.i.d.\ draws from $\mathcal{D}_{\text{test}}$
\Require Hyperparameters: background size $m$, repetitions $T$, ridge $\lambda>0$, decision threshold $\tau$
\State $S \leftarrow 0$ \Comment{initialize cumulative evidence}
\For{$t=1$ {\bf to} $T$}
    \State $B_t=\{(b_i,\tilde y_i)\}_{i=1}^{m} \leftarrow \mathcal{S}_{\text{test}}(m)$ 
    \Comment{if labels are unavailable, set $\tilde y_i\!=\!\arg\max f(b_i;\theta^{\mathrm{org}})$}
    \For{$i=1$ {\bf to} $m$}
        \State $\Delta_i \leftarrow \nabla_{\theta}\ell(f(b_i;\theta^{u}),\tilde y_i) - \nabla_{\theta}\ell(f(b_i;\theta^{\mathrm{org}}),\tilde y_i) \in \mathbb{R}^{d}$
    \EndFor
    \State $\hat{\mu}_t \leftarrow \frac{1}{m}\sum_{i=1}^{m}\Delta_i$
    \State $\hat{\Sigma}_t \leftarrow \frac{1}{m-1}\sum_{i=1}^{m}(\Delta_i-\hat{\mu}_t)(\Delta_i-\hat{\mu}_t)^{\top}$
    \State $\hat{\Sigma}_{t,\lambda}\leftarrow \hat{\Sigma}_t+\lambda I_d$ \Comment{ridge for numerical stability}
    \State $\Delta_x \leftarrow \nabla_{\theta}\ell(f(x;\theta^{u}),y)-\nabla_{\theta}\ell(f(x;\theta^{\mathrm{org}}),y)$
    \State $v \leftarrow \Delta_x-\hat{\mu}_t$
    \State Solve $\hat{\Sigma}_{t,\lambda}\, w = v$ for $w$ (e.g., Cholesky); \quad $s_t \leftarrow v^\top w$
    \State $\ell_t \leftarrow -\log\!\big(1 - F_{\chi^2_{d}}(s_t)\big)$ \Comment{log tail $p$-value under $H_0$}
    \State $S \leftarrow S + \ell_t$
\EndFor
\State \textbf{return} $\textsc{Forgotten}$ if $S>\tau$; else $\textsc{Test}$
\end{algorithmic}
\end{algorithm}
%%%%%%%%%%%%%%%%%%%%%%%%%%%%%%%%%%%%%%%%%%%%%%%%%%%%%%%%%%%%

\section{Alternative symmetry: change-of-basis neural teleportation.}
\label{sec:appendix-alternative}
We also support the ``neural teleportation'' family of symmetry moves from \cite{armenta2023neural}. 
Let $\tau_a>0$ be a scale attached to neuron $a$. 
For an edge $a\!\to\! b$ with weight $\theta_{ab}$ the teleported weight is
\begin{equation}
\label{eq:nt-weight}
\theta'_{ab}=\frac{\tau_b}{\tau_a}\,\theta_{ab},
\end{equation}
and if $f_d$ is the activation at neuron $d$ then the teleported activation is
\begin{equation}
\label{eq:nt-activation}
g_d(x)=\tau_d\,f_d\!\left(\frac{x}{\tau_d}\right),
\end{equation}

which preserves the function for positively homogeneous activations and extends naturally to batch-norm scales \cite{armenta2023neural}.
In a subset of experiments, we choose $\tau$ to further increase parameter dispersion under loss invariance (outputs unchanged), thereby weakening the differencing signal and making reconstruction harder; most results rely on the null space instantiation in \eqref{eq:teleport-update}. In the experimental section, it is explicitly indicated when both mechanisms are enabled.

\section{Baselines}
\label{sec:appendix-baselines}

We evaluate our teleportation-based defense as a \emph{plug-and-play} module layered on top of several state-of-the-art approximate post-hoc unlearning methods. These baselines are representative of the most widely studied approaches in recent literature, requiring no access to training-time auxiliary statistics (e.g., per-sample gradients) and operating directly on a pretrained model. Specifically, we consider:

\begin{enumerate}
    % \item \textbf{Gradient Ascent (GA)}~\cite{graves2021amnesiac,thudi2022unrolling}:  
    % A naive baseline that directly performs gradient ascent on the forget set to increase its loss, thereby erasing the corresponding knowledge from the model. While simple, GA is highly vulnerable to privacy leakage since the update direction is tightly coupled with the forgotten data.
    
    \item \textbf{NegGrad+ (NGP)}~\cite{kurmanji2023towards}:  
    An improved variant of GA that incorporates a regularization term on the retain-set. The method balances ascent on the \textit{forget-set} with descent on the retain-set, aiming to preserve model utility while unlearning.
    
    \item \textbf{SCRUB}~\cite{kurmanji2023towards}:  
    A knowledge distillation approach that aligns the unlearned model with the original model on the retain-set via a consistency loss, while simultaneously removing the \textit{forget-set}’s influence. SCRUB represents one of the most competitive baselines in recent evaluations.

    \item \textbf{SalUn}~\cite{fan2023salun}:  
    A saliency-based unlearning method that directs updates to a subset of weights deemed \emph{salient} for forgetting, identified via gradient-based weight saliency maps.  
    By restricting optimization to these salient weights, SalUn enhances stability and efficiency compared to updating the full parameter set, and aims to reduce the gap to exact retraining.
    
    \item \textbf{Projected Gradient Unlearning (PGU)}~\cite{hoang2024learn}:  
    A method that projects the gradient ascent update for the \textit{forget-set} onto a subspace orthogonal to retain-set, thereby mitigating catastrophic forgetting. PGU is particularly relevant as it addresses gradient-level entanglement between forget and retain data.
    
    \item \textbf{BadTeacher (BT)}~\cite{chundawat2023can}:  
    A recent distillation-based unlearning method where the unlearned model (student) is trained against a deliberately corrupted teacher that provides noisy or adversarial labels for the \textit{forget-set}, encouraging the student to erase their influence while preserving performance on the retain-set.

    \item \textbf{SRF-ON (SF)}~\cite{huang2024unified}:  
A geometry-aware unlearning method that decomposes updates into forget ascent, retain descent, and saliency modulation.  
By embedding updates into the manifold of retain data and approximating Hessian modulation with a fast–slow strategy, SRF-ON improves stability–plasticity trade-offs and enables efficient large-scale unlearning.

\end{enumerate}

These methods span the main paradigms of approximate unlearning—gradient ascent, retain-aware regularization, distillation, and projection-based updates—making them representative state-of-the-art baselines.

\section{Additional White-box Results on CIFAR-10}
\label{sec:appendix-fullWBROC}

Figure~\ref{fig:roc_full} reports the complete ROC curves for the Gaussian gradient–diff test, covering the entire FPR range. 
These correspond to the same 640 unlearned models as in Figure~\ref{fig:wb_privacy}, shown here without zoom to provide the full view.

\begin{figure}[h]
    \centering
    \includegraphics[width=0.7\linewidth]{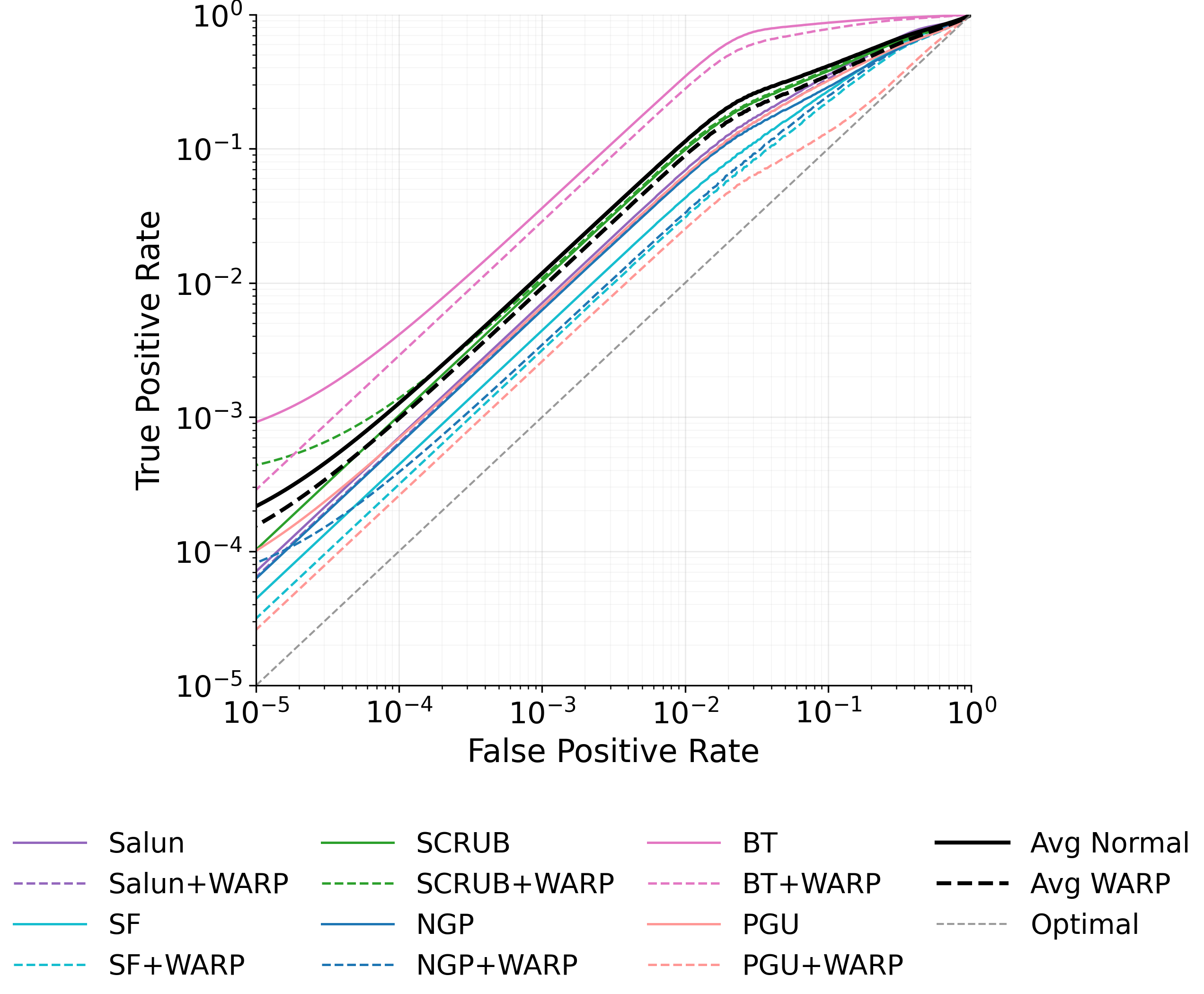}
    \caption{\textbf{Complete ROC curves for the white-box Gaussian gradient–diff test.} 
    Averaged over 640 unlearned models, identical to Figure~\ref{fig:wb_privacy}. 
    Lower curves (closer to the random-guess diagonal) indicate stronger privacy.}
    \label{fig:roc_full}
\end{figure}

\section{Reconstruction Attack Baselines and Comparison.}
Table~\ref{tab:recon_core} compares three strategies for unlearning:
(i) \emph{simple differencing}, directly inverting $\Delta\theta$\cite{hu2024learn,bertran2024reconstruction};
(ii) \emph{generative inversion} (GIFD)\cite{fang2023gifd} applied to $\Delta\theta$; and
(iii) \emph{Ours}, which adds \emph{orthogonal subspace filtering} (Eq.~\eqref{eq:filtered-target}) to a generative backbone.
Results average 100 forgotten samples on ImageNet-1K with ResNet-18 under NGP unlearning.

\begin{table}[t]
\centering
\caption{\textbf{Reconstruction on ImageNet-1K (ResNet-18), NGP (no defense).}
Averages over 100 forgotten samples. Higher is better for PSNR/SSIM; lower is better for LPIPS/MSE.}
\label{tab:recon_core}
\resizebox{\textwidth}{!}{%
\begin{tabular}{lcccccc}
\toprule
\textbf{Method} & \textbf{PSNR (dB)} $\uparrow$ & \textbf{LPIPS (VGG)} $\downarrow$ & \textbf{LPIPS (Alex)} $\downarrow$ & \textbf{SSIM} $\uparrow$ & \textbf{Test MSE} $\downarrow$ & \textbf{Feat MSE} $\downarrow$ \\
\midrule
% Simple differencing (\(\Delta\theta\))~\cite{hu2024learn,bertran2024reconstruction} & 6.98 \pm 0.24 & --- & --- & --- & --- & --- \\
GIFD \cite{fang2023gifd}
& $8.28 \pm 0.28$ & $0.630 \pm 0.012$ & $0.448 \pm 0.016$ & $0.098 \pm 0.007$ & $0.174 \pm 0.012$ & $6.725 \pm 0.506$ \\
\textbf{Ours} (subspace-filtered + GFID) 
& $\mathbf{10.74} \pm 0.31$ & $\mathbf{0.564} \pm 0.013$ & $\mathbf{0.345} \pm 0.015$ & $\mathbf{0.117} \pm 0.008$ & $\mathbf{0.100} \pm 0.007$ & $\mathbf{5.388} \pm 0.497$ \\
\midrule
\textbf{Improvement (\%)} & $+29.7$ & $+10.5$ & $+22.9$ & $+19.4$ & $+42.5$ & $+19.9$ \\
\bottomrule
\end{tabular}%
}
\end{table}

\section{Additional Results: ViT on Tiny-ImageNet}
\label{sec:appendix-wevit}

To extend the white-box analysis of Section~\ref{sec:exp-whitebox-mia}, we evaluate Vision Transformer models trained on Tiny-ImageNet. We adopt ViT-B/16 as the base architecture and follow the same setup described in Section~\ref{sec:exp-setup}, with the \textit{forget-set} constructed by randomly sampling $1\%$ of the training data and the retain-set consisting of the remainder. All models are trained with SGD and standard augmentations for ViT training. Unlearning is applied with NGP (\textsc{NGP}) and its teleported variant (\textsc{NGP+WARP}).

\begin{table}[t]
\centering
\small
\setlength{\tabcolsep}{4pt}
\caption{\textbf{White-box membership inference risk with and without teleportation (ViT, Tiny-ImageNet).} 
Results are reported as mean $\pm$ standard deviation across five splits. 
Improvements are computed as advantage reduction over random guessing.}
\vspace{2mm}
\begin{tabular}{lccccc}
\toprule
Method & AUC & TPR@0.01\% & TPR@0.1\% & TPR@1\% & TPR@5\% \\
\midrule
NGP (base)       
& $0.792 \pm 0.019$ 
& $0.0019 \pm 0.001$ 
& $0.0188 \pm 0.009$ 
& $0.178 \pm 0.072$ 
& $0.444 \pm 0.035$ \\
\,+ WARP     
& $\mathbf{0.755 \pm 0.019}$ 
& $\mathbf{0.0008 \pm 0.000}$ 
& $\mathbf{0.0079 \pm 0.003}$ 
& $\mathbf{0.075 \pm 0.027}$ 
& $\mathbf{0.302 \pm 0.054}$ \\
Improvement (\%) 
& \textbf{12.7} 
& \textbf{61.1} 
& \textbf{61.2} 
& \textbf{61.2} 
& \textbf{36.1} \\
\bottomrule
\end{tabular}
\label{tab:wb_vit}
\end{table}

\begin{figure}[t]
    \centering
    \begin{minipage}{0.48\linewidth}
        \centering
        \includegraphics[width=\linewidth]{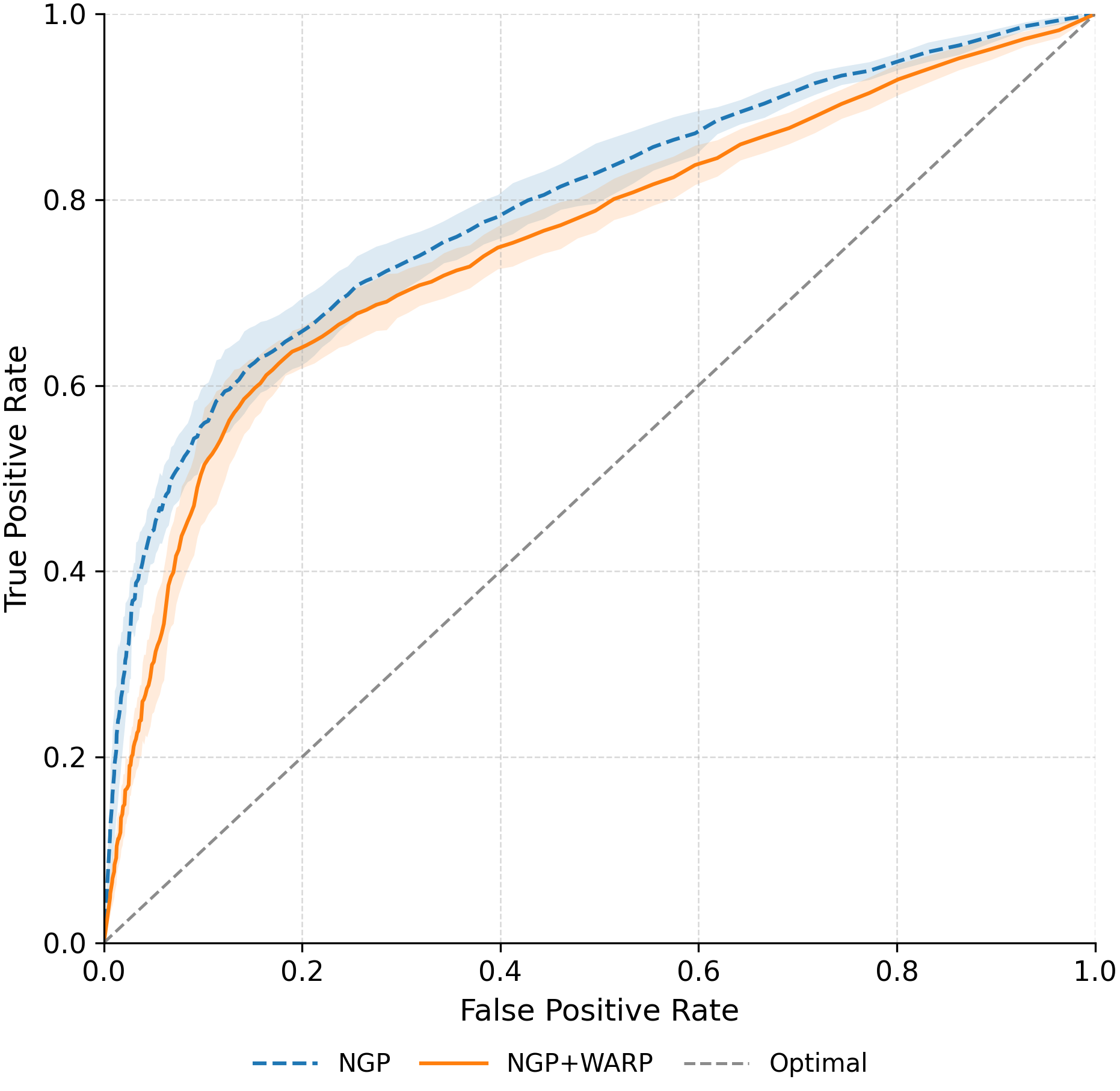}
        \vspace{0.3em}
        (a) ROC curve (linear scale)
    \end{minipage}\hfill
    \begin{minipage}{0.48\linewidth}
        \centering
        \includegraphics[width=\linewidth]{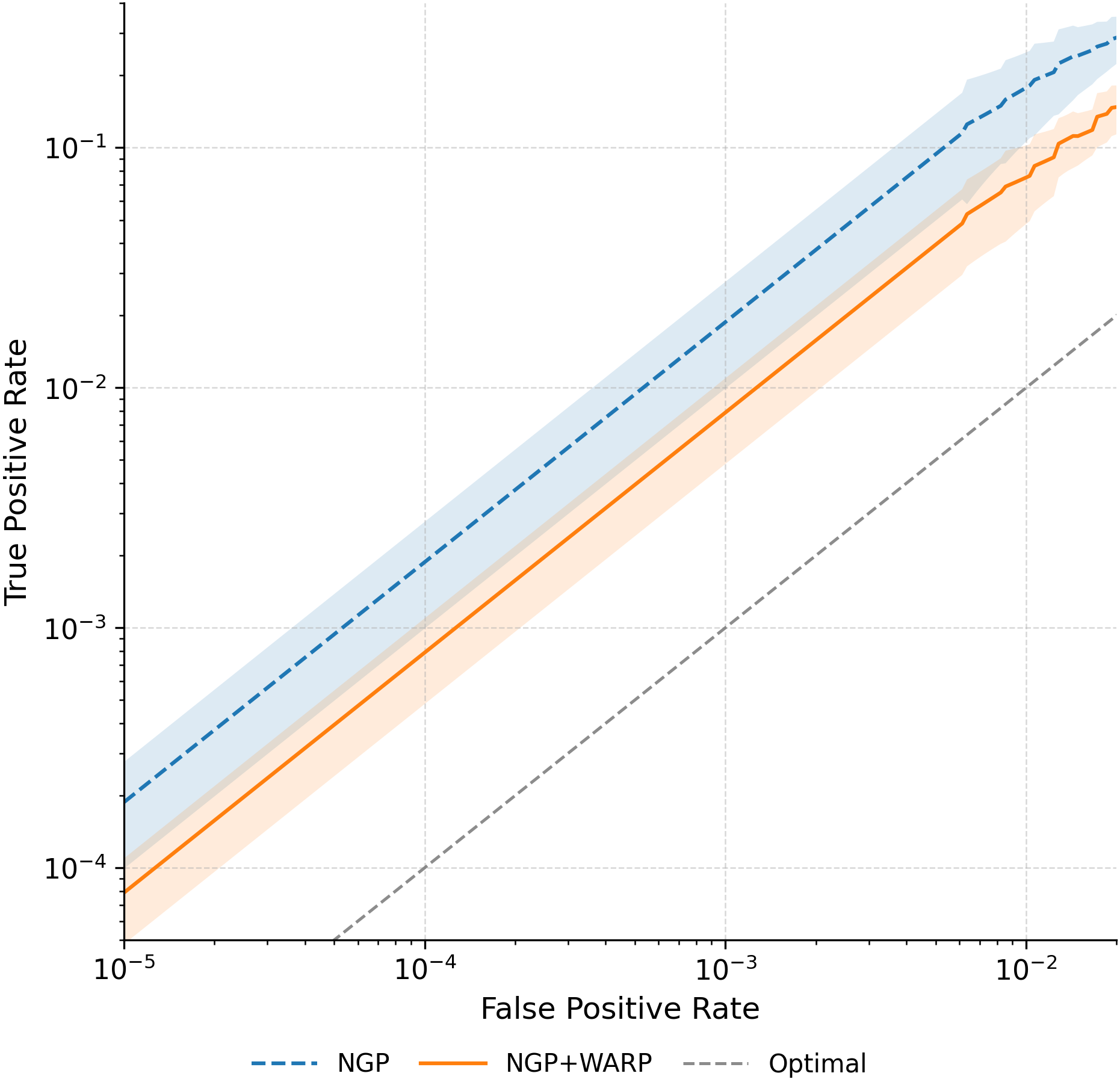}
        \vspace{0.3em}
        (b) ROC curve (log–log scale, low-FPR region)
    \end{minipage}
    \caption{\textbf{White-box ROC for the Gaussian gradient–difference test on ViT-B/16 (Tiny-ImageNet).}  
    Each curve is averaged over five different forget-set splits, with shaded regions showing the standard deviation.  
    Both figures compare \textsc{NGP} and \textsc{NGP+WARP}; (a) presents the full ROC on a linear axis, while (b) zooms into the low-FPR regime on log–log scale, which is the operational region for practical attacks.}
    \label{fig:glir_vit}
\end{figure}

As shown in Table~\ref{tab:wb_vit} and Figure~\ref{fig:glir_vit}, \textsc{WARP} substantially reduces attack success across all thresholds, with the largest relative gains at low false-positive rates where practical attacks operate. These results confirm that the symmetry-based defense proposed in \textsc{WARP} extends effectively to transformer models, demonstrating applicability beyond convolutional architectures.

\section{Privacy-Utility Trade-off}
\label{sec:appendix-exp-ngp-tradeoff}

\begin{figure}
    \centering
    \includegraphics[width=0.8\linewidth]{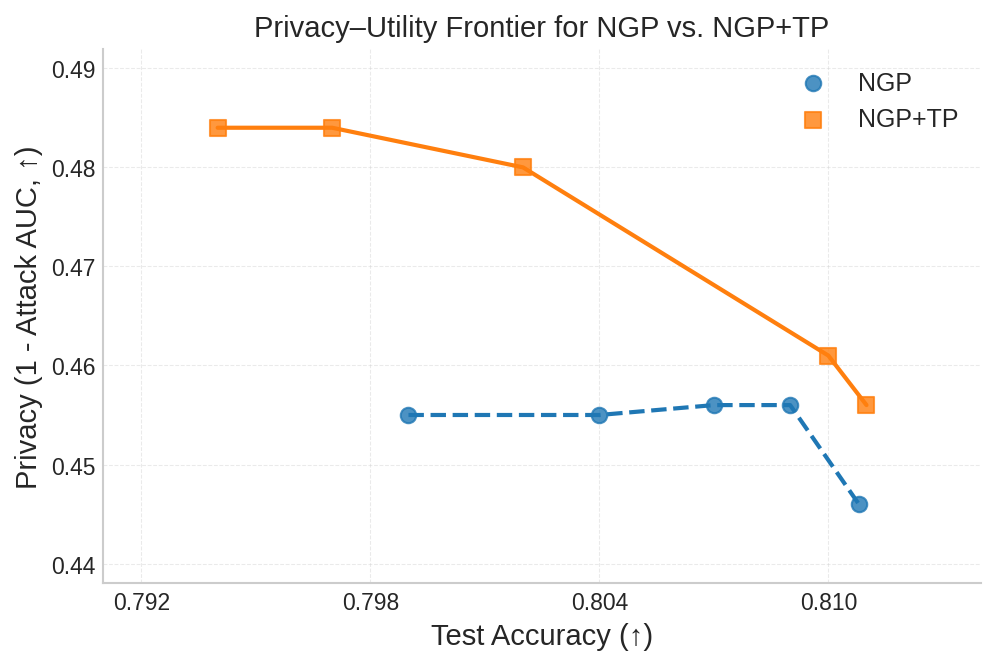}
    \caption{\textbf{Privacy--utility trade-off for \textsc{NGP} with and without WARP.} 
Each point is a hyperparameter trial, with privacy (1--AUC) averaged over $640$ shadow models (64 shadows $\times$ 10 forget sets) under the U\hbox{-}LiRA protocol. 
Points further to the right (higher accuracy) and upward (higher privacy) indicate better trade-offs.}
    \label{fig:tradeoff}
\end{figure}

Improving privacy in unlearning often comes at the cost of reduced model accuracy. Since test accuracy on the retain-set is one of the primary criteria for evaluating unlearning algorithms, it is critical to examine whether the proposed defense introduces unfavorable trade-offs. 

We focus this analysis on \textsc{NGP}, as Figure~\ref{fig:spider_six_methods} indicates that teleportation applied to \textsc{NGP} yields the most noticeable accuracy drop (roughly one percentage point), whereas for other methods accuracy remains stable or even improves. To probe this trade-off more carefully, we follow the hyperparameter tuning procedure described in Section~\ref{sec:exp-setup} and select the top $20$ trials with the highest validation score. From this pool we examine: (i) the single best-performing trial reported in Figure~\ref{fig:spider_six_methods}, (ii) the two trials with the highest validation accuracy, and (iii) the two trials with the lowest validation attack AUC. 

Figure~\ref{fig:tradeoff} plots test accuracy against privacy ($1{-}\text{AUC}$ of black-box MIA, higher is better) for \textsc{NGP} and \textsc{NGP+WARP} across the selected hyperparameter trials. 
The overall trade-off is clear: higher accuracy typically coincides with lower privacy. 
Yet teleportation consistently shifts the Pareto frontier upward, delivering strictly better privacy at nearly every accuracy level. 
While \textsc{NGP} saturates around privacy $\approx 0.455$, teleportation extends this frontier up to $ 0.484$, breaking through the baseline ceiling. 
At the highest-accuracy setting, teleportation still provides a $\sim\!18\%$ reduction in attack advantage over random, demonstrating that even at stringent accuracy targets the defense yields nontrivial privacy gains. 
Across the frontier, improvements remain stable, confirming that teleportation meaningfully reshapes the privacy–utility boundary in favor of the defender.

\section{Runtime Analysis}
\label{sec:appendix-runtime}
\begin{figure}
    \centering
    \includegraphics[width=0.8\linewidth]{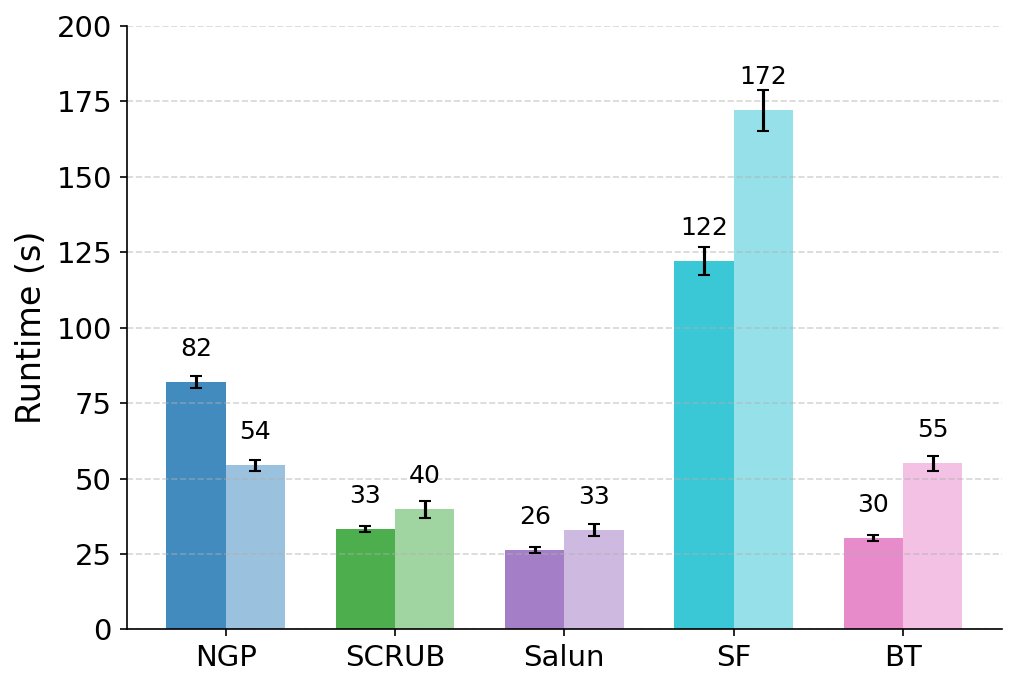}
    \caption{\textbf{Runtime overhead of teleportation.} 
Average runtimes (seconds) of unlearning algorithms with and without the \textsc{WARP} plugin, 
evaluated on CIFAR-10 with ResNet-18. 
Each bar reports the mean over five runs, with error bars showing standard deviations.}
    \label{fig:runtime}
\end{figure}

In this appendix we focus on the retain–null-space instantiation of $T_\phi$, which is the only variant that requires explicit SVDs; the change-of-basis teleportation in Appendix~\ref{sec:appendix-alternative} is SVD-free and without its computational overhead as a result.  Moreover, Section~\ref{sec:fastwarp} introduces FastWARP, which replaces full SVD with randomized low-rank approximations and further reduces this overhead.

We benchmark the runtime of our teleportation defense across unlearning algorithms on a machine equipped with an NVIDIA GeForce RTX 4090 GPU (24 GB memory) and an Intel 13th Gen Core i9-13900KF CPU (24 cores, 32 threads, base 3.0 GHz, boost up to 5.8 GHz).  
Each experiment was repeated five times, and Figure~\ref{fig:runtime} reports averages with standard deviations in the caption.  
All algorithms were run with the hyperparameters used in Table~\ref{tab:bb_mm_advred} and Figure~\ref{fig:wb_privacy}, ensuring runtime reflects the same conditions as our privacy–utility evaluations.  

For this particular SVD-based instantiation, teleportation increases runtime by approximately $+27\%$ relative to the baseline on average, reflecting the overhead of constructing the retain subspace.  
The main exception is \textsc{NGP}, where teleportation reduces runtime by about $-32\%$, due to more stable updates that in turn lower the required number of unlearning epochs.  
Since subspace computation can be pre-computed offline and does not need to be repeated after every teleportation step, this overhead can be amortized in practice. 
% Approximations such as updating the subspace at intervals rather than per step could further lower the cost, though runtime optimization is left for future work.
While updating the retain subspace less frequently can reduce cost, the primary computational overhead from full SVD is addressed directly by an approximate low-rank implementation (Appendix~\ref{sec:fastwarp}), which removes the per-step bottleneck entirely.

\section{Teleportation-based Unlearning Algorithm}
\label{sec:appendix-teleport}

In Algorithm~\ref{alg:tau}, $T_\phi$ denotes an abstract symmetry operator; in our experiments we instantiate it either with retain–null-space
teleportation or with change-of-basis teleportation, but any other loss-preserving symmetry could be used in its
place.

\begin{algorithm}[H]
% \caption{Teleportation–Augmented Unlearning (TAU)}
\caption{\textsc{WARP} (retain–null-space instantiation): teleportation-augmented gradient-based unlearning.}
\label{alg:tau}
\begin{algorithmic}[1]
\setlength{\itemsep}{2pt}
\Require $\theta^{\mathrm{org}}$, $\mathcal{D}_{\mathrm f}$, $\mathcal{D}_{\mathrm r}$, $\ell_{\mathrm f}$, $\ell_{\mathrm r}$, $\lambda,\beta$, $\{\eta_t\}$, $\eta_{\mathrm{tel}}$, $k$, $S$ or $\tau_{\text{grad}}$, $\sigma^2$, $\varepsilon$, $T$
\State $\theta_0 \gets \theta^{\mathrm{org}}$
\For{$t=0,\dots,T-1$}
  \State sample $\mathcal{B}_{\mathrm f}\subset\mathcal{D}_{\mathrm f}$,\; $\mathcal{B}_{\mathrm r}\subset\mathcal{D}_{\mathrm r}$
  \State $\theta_{t+\frac12}\gets\theta_t-\eta_t\big(\nabla_\theta \ell_{\mathrm f}(\theta_t\mid\mathcal{B}_{\mathrm f})+\lambda\,\nabla_\theta \ell_{\mathrm r}(\theta_t\mid\mathcal{B}_{\mathrm r})\big)$
  \If{$(t\bmod S=0)\ \lor\ \|\nabla_\theta \ell_{\mathrm f}(\theta_{t+\frac12}\mid\mathcal{B}_{\mathrm f})\|_2>\tau_{\text{grad}}$}
    \For{layer $\ell$}
      \State build $R_\ell(\mathcal{B}_{\mathrm r})$; \quad $R_\ell=U_\ell\Sigma_\ell V_\ell^\top$ (SVD)
      \State $B_\ell\gets U_{\ell,1:k}$;\quad $\Pi_\ell^\perp\gets I-B_\ell B_\ell^\top$
    \EndFor
    \State $\mathcal{L}_{\mathrm{tel}}(\theta)=\tfrac12\!\sum_{(x,y)\in\mathcal{B}_{\mathrm f}}\!\|\nabla_\theta \ell(f(x;\theta),y)\|_2^2-\tfrac{\beta}{2}\|\theta-\theta^{\mathrm{org}}\|_2^2$
    \For{layer $\ell$}
      \State $W_\ell^{\,t+1}\gets W_\ell^{\,t+\frac12}-\eta_{\mathrm{tel}}\ \Pi_\ell^\perp\big(\nabla_{W_\ell}\mathcal{L}_{\mathrm{tel}}(\theta_{t+\frac12})\big)
      % +\textcolor{red}{\sqrt{2\,\eta_{\mathrm{tel}}\,\sigma^2}\ \boldsymbol{\varepsilon}_{\ell,t}}$
      $
      % \State \hfill $\boldsymbol{\varepsilon}_{\ell,t}\sim\mathcal{N}(0,I)$
    \EndFor
    \State $\theta_{t+1}\gets\{W_\ell^{\,t+1}\}_\ell$
    \If{$\ell_{\mathrm r}(\theta_{t+1}\mid\mathcal{B}_{\mathrm r})>\ell_{\mathrm r}(\theta_t\mid\mathcal{B}_{\mathrm r})+\varepsilon$}
      \State $\theta_{t+1}\gets\theta_{t+\frac12}$ \Comment{backtrack/safeguard}
    \EndIf
  \Else
    \State $\theta_{t+1}\gets\theta_{t+\frac12}$
  \EndIf
\EndFor
\State \Return $\theta^{u}\gets\theta_T$
\end{algorithmic}
\end{algorithm}

\section{Approximate Null-Space Teleportation}
\label{sec:fastwarp}

\paragraph{Low-rank structure of retain representations.}
For a retain minibatch $\mathcal{B}_{\mathrm r}$ and layer $\ell$, let
$R_\ell(\mathcal{D}_{\mathrm r}) \in \mathbb{R}^{|\mathcal{B}_{\mathrm r}|\times d_\ell}$ 
denote the matrix whose rows collect the layer-$\ell$ inputs
$\{\phi_\ell(x)\}_{x \in \mathcal{B}_{\mathrm r}}$.
Empirically, $R_\ell(\mathcal{D}_{\mathrm r})$ exhibits strong spectral decay:
its spectrum is dominated by a small number of singular values, and most of the
energy lies in a low-dimensional subspace.
Such low-rank structure of activations, gradients and Hessians has been observed
repeatedly in modern deep networks
\citep{arora2019fine,ghorbani2019data,fort2020deep,gurari2019tiny},
and is often attributed to overparameterisation and the implicit regularisation of SGD.
In \textsc{WARP}, the retain subspace at layer $\ell$ is defined by the top-$k$
left singular vectors of $R_\ell(\mathcal{D}_{\mathrm r})$:
\[
R_\ell(\mathcal{D}_{\mathrm r}) = U_\ell \Sigma_\ell V_\ell^\top,
\qquad
B_\ell = U_{\ell,1:k},
\qquad
\Pi_\ell^\perp = I - B_\ell B_\ell^\top.
\]
Since only the span of these dominant directions matters for teleportation,
\emph{exact} SVD is not required: any procedure that recovers a good approximation
to the top-$k$ principal subspace suffices.

\paragraph{Covariance-based PCA and subspace iteration.}
Instead of computing a full thin SVD of $R_\ell(\mathcal{D}_{\mathrm r})$,
\textsc{FastWARP} estimates $B_\ell$ via a covariance eigen-decomposition
and a small number of subspace-iteration updates, following classical
PCA and online PCA methods
\citep{golub2013matrix,oja1982simplified,warmuth2007randomized,mitliagkas2013memory}.
We first form the covariance
\[
C_\ell \;=\; X_\ell X_\ell^\top \in \mathbb{R}^{d_\ell \times d_\ell},
\]
where $X_\ell \in \mathbb{R}^{d_\ell \times N}$ is a layer-wise input matrix
constructed from $\mathcal{B}_{\mathrm r}$ (for convolutional layers we use
unfolded patches; for batch-norm we aggregate per-channel features).
We then compute the eigen-decomposition
$C_\ell = Q_\ell \Lambda_\ell Q_\ell^\top$ and retain the smallest $k$ such that
the cumulative explained variance exceeds a threshold $\tau$:
\[
k \;=\; \min \Bigl\{ j : 
  \tfrac{\sum_{i=1}^{j} \max(\lambda_{\ell,i},0)}{\sum_{i=1}^{d_\ell} \max(\lambda_{\ell,i},0)}
  \ge \tau \Bigr\}, 
\qquad
B_\ell \;=\; Q_{\ell,1:k},
\]
optionally capped by a user-specified $k_{\max}$.
For subsequent teleportation steps, we update $B_\ell$ using a few iterations of
subspace iteration
\citep{golub2013matrix,halko2011finding,musco2015randomized,tropp2017practical,woodruff2014sketching}:
\[
Y \;\leftarrow\; C_\ell B_\ell,
\qquad
[B_\ell,\_ ] \;\leftarrow\; \mathrm{qr}(Y),
\]
which amounts to an Oja-style streaming PCA update
\citep{oja1982simplified} with QR re-orthogonalisation.
This reduces the cost of updating $B_\ell$ for a new minibatch from the
$\mathcal{O}(|\mathcal{B}_{\mathrm r}| d_\ell^2)$ cost of a fresh thin SVD to
$\mathcal{O}(|\mathcal{B}_{\mathrm r}| d_\ell k)$ for the covariance application
plus $\mathcal{O}(d_\ell k^2)$ for QR, with $k \ll d_\ell$.
The resulting projector $\Pi_\ell^\perp = I - B_\ell B_\ell^\top$ is then used
exactly as in the original \textsc{WARP} update.

\begin{algorithm}[H]
\caption{\textsc{FastWARP} basis update at layer $\ell$}
\label{alg:fastwarp_basis}
\begin{algorithmic}[1]
\setlength{\itemsep}{2pt}
\Require $d_\ell$, retain minibatch $\mathcal{B}_{\mathrm r}$, 
$B_\ell^{\mathrm{prev}}$ (or \textsc{None}), $\tau\in(0,1]$, $k_{\max}$, $T_{\mathrm{track}}$
\State build $X_\ell\in\mathbb{R}^{d_\ell\times N}$ from $\mathcal{B}_{\mathrm r}$
\State $C_\ell \gets X_\ell X_\ell^\top$;\quad $C_\ell \gets \tfrac12(C_\ell + C_\ell^\top)$
\If{$B_\ell^{\mathrm{prev}}=\textsc{None}$}
  \State $C_\ell = Q_\ell \Lambda_\ell Q_\ell^\top$
  \State sort $\Lambda_\ell$ in descending order, permute $Q_\ell$ accordingly
  \State $k \gets \min\Big\{k_{\max},\ \min\big\{k:\tfrac{\sum_{i=1}^k \Lambda_{\ell,ii}}{\sum_{i} \Lambda_{\ell,ii}}\ge\tau\big\}\Big\}$
  \State $B_\ell \gets Q_\ell[:,1\!:\!k]$
\Else
  \State $B_\ell \gets B_\ell^{\mathrm{prev}}$
  \For{$t = 1,\dots,T_{\mathrm{track}}$}
    \State $Y \gets C_\ell B_\ell$
    \State $[B_\ell,\_] \gets \mathrm{qr}(Y)$
    \State $B_\ell \gets B_\ell[:,1\!:\!k]$
  \EndFor
\EndIf
\State $\Pi_\ell^\perp \gets I_{d_\ell} - B_\ell B_\ell^\top$
\State \Return $B_\ell,\ \Pi_\ell^\perp$
\end{algorithmic}
\end{algorithm}

% \begin{figure}[t]
%     \centering
%     \includegraphics[width=0.55\linewidth]{images/fast_time.png}
%     \caption{Runtime of the \textsc{WARP} plug-in on CIFAR-10 with ResNet-18. 
%     Each bar reports the mean over five runs. 
%     The top hatched segments correspond to the additional teleportation time; 
%     the solid base is the runtime of the underlying MU algorithm.}
%     \label{fig:fastTime}
% \end{figure}

% \begin{figure}[t]
%     \centering
%     \includegraphics[width=0.55\linewidth]{images/FastNGP.png}
%     \caption{Privacy--utility comparison of NG+\textsc{WARP} and NG+\textsc{FastWARP}.  
%     The approximate teleportation method (\textsc{FastWARP}) matches the privacy--utility frontier of the exact variant, 
%     achieving nearly identical privacy ($1-\mathrm{AUC}$) and test accuracy while significantly reducing the cost of null-space projection.}
%     \label{fig:fastwarp}
% \end{figure}

\begin{figure}[t]
    \centering

    %--------- Left figure ---------
    \begin{minipage}{0.48\linewidth}
        \centering
        \captionsetup{width=\linewidth}
        \includegraphics[width=\linewidth]{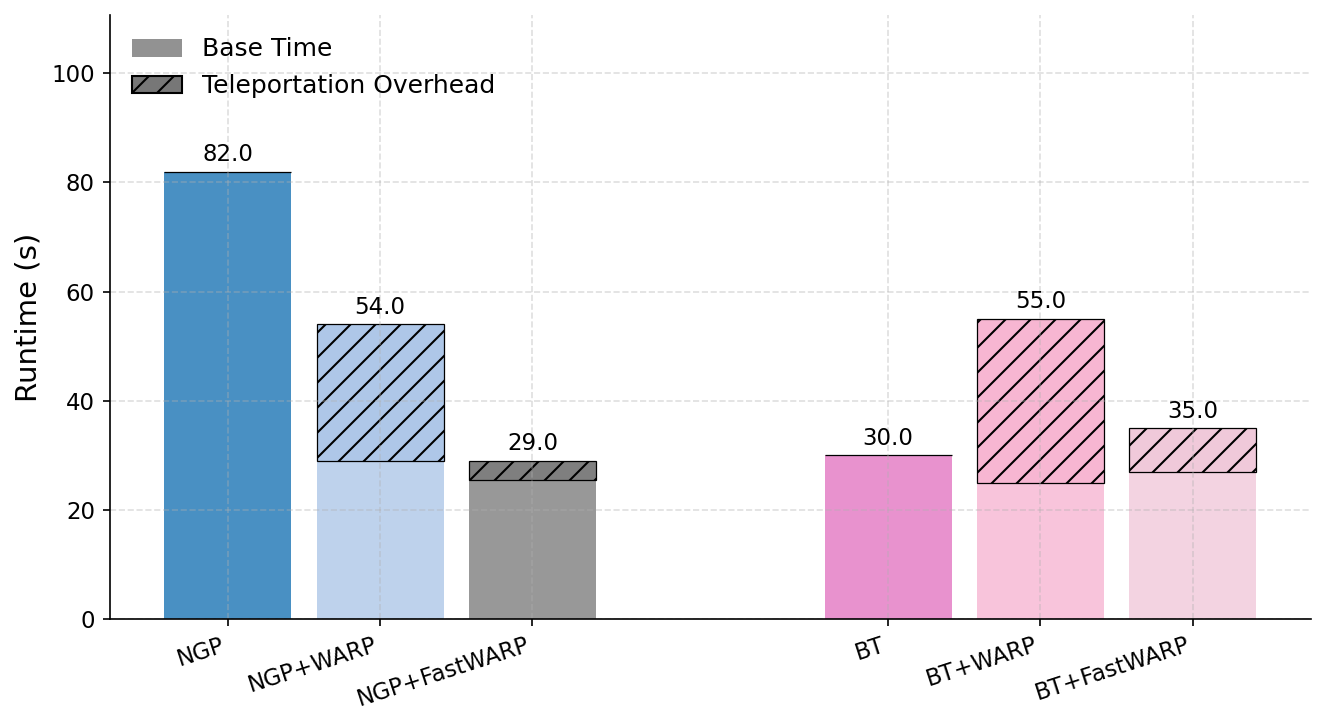}
        \caption{Runtime of the \textsc{WARP} plug-in on CIFAR-10 with ResNet-18. 
        Each bar reports the mean over five runs. 
        The top hatched segments correspond to the additional teleportation time; 
        the solid base is the runtime of the underlying MU algorithm.}
        \label{fig:fastTime}
    \end{minipage}
    \hfill
    %--------- Right figure ---------
    \begin{minipage}{0.48\linewidth}
        \centering
        \captionsetup{width=\linewidth}
        \includegraphics[width=\linewidth]{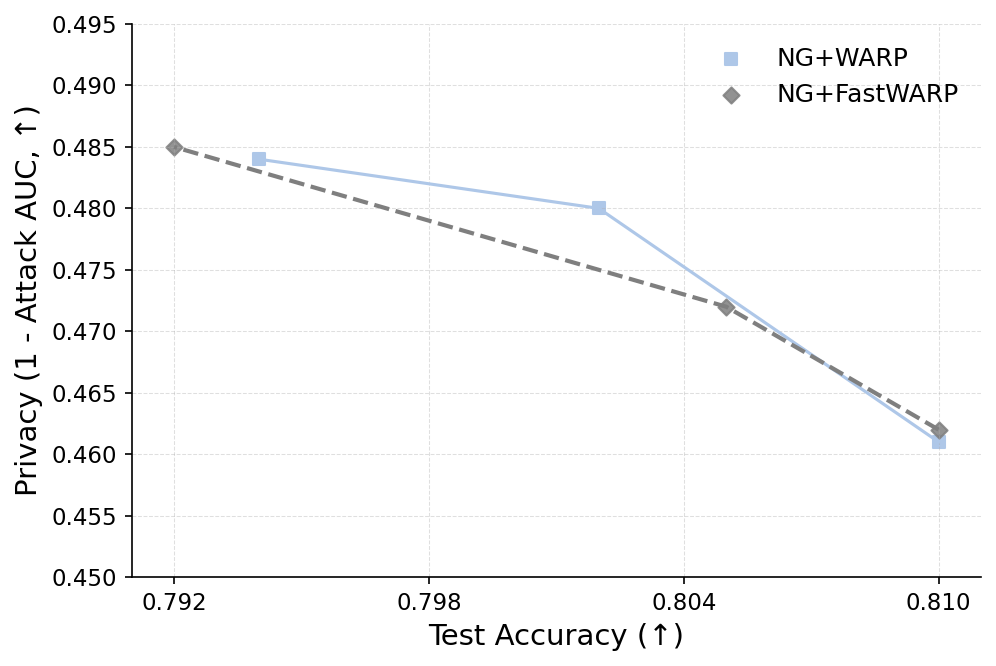}
        \caption{Privacy--utility comparison of NG+\textsc{WARP} and NG+\textsc{FastWARP}.  
        The approximate teleportation method 
        (\textsc{FastWARP}) matches the privacy--utility frontier of the exact variant, 
        achieving nearly identical privacy and test accuracy.}
        \label{fig:fastwarp}
    \end{minipage}

\end{figure}

\paragraph{Runtime and privacy--utility impact.}
Figure~\ref{fig:fastTime} reports the runtime for NG and BT with and without
teleportation on CIFAR-10/ResNet-18.
The hatched segments correspond to the teleportation component.
Using full SVD yields a moderate yet visible overhead (e.g.,  BT+WARP increases runtime from $30$\,s to $55$\,s).
Replacing full SVD with the covariance-based PCA and subspace iteration of Algorithm~\ref{alg:fastwarp_basis} (\textsc{FastWARP}) shrinks this overhead
substantially: total runtime drops to $29$\,s and $35$\,s for NG+FastWARP and
BT+FastWARP, corresponding to a $2\times$--$3\times$ reduction in the
teleportation time.  
The teleportation component becomes only a small fraction of the overall MU cost.

To measure the effect of this approximation on privacy and accuracy, 
Figure~\ref{fig:fastwarp} compares NG+\textsc{WARP} and NG+\textsc{FastWARP}
along the privacy--utility frontier.
The two curves are nearly indistinguishable: privacy ($1-\mathrm{AUC}$) differs
by at most $0.3$--$0.6\%$ across operating points, and test accuracy changes by
at most $0.2$--$0.3$ percentage points.
We also track retain-set loss during teleportation and observe that the relative
drift under \textsc{FastWARP} remains below $2\%$, indicating that the approximate
projector continues to enforce practical loss invariance.
In some configurations, the additional numerical noise introduced by the
approximation yields slightly \emph{higher} privacy for the same utility.
Overall, these results show that the privacy gains of \textsc{WARP} are robust
to approximate PCA, and that \textsc{FastWARP} preserves the empirical
privacy--utility trade-off while significantly reducing computational overhead.

\paragraph{Scalability to LLMs and calibration of the retain subspace.}
A natural concern is whether null-space teleportation remains practical and stable at LLM scale, where layer widths reach $d_\ell \sim 10^3$–$10^4$ and a single minibatch may not span the retain subspace. Empirically, recent compression work shows that truncated SVD and related low-rank factorizations are already applied efficiently to full LLM weight matrices with comparable or larger dimensions: SVD-LLM~\cite{wang2024svd,wang2025svd} optimizes singular-value truncation for LLaMA~\cite{touvron2023llama}- and GPT~\cite{brown2020language}-class models while preserving perplexity and throughput, demonstrating that rank-$k$ SVD with $k \ll d_\ell$ is tractable in practice on modern hardware.   Complementary methods such as ResSVD~\cite{bai2025ressvd} leverage the residual matrix left by truncation to correct the approximation, further reducing the effective loss of expressivity at fixed rank.  Orthogonal lines of work, e.g., weighted low-rank factorization for LMs, explicitly introduce data-dependent weights in the covariance (or Gram) operator to bias the recovered subspace toward high-importance tokens or examples, and report competitive compression ratios on transformer-based LMs~\cite{hsu2022language,sakr2024espace}. In our setting, we can adopt the same design principles: instead of forming $R_\ell(\mathcal{B}_{\mathrm r})$ from an arbitrary minibatch, we maintain a small buffer of retain batches with large gradient norm~\cite{sakr2024espace} or Fisher information, and construct the activation matrix $X\ell$ from this “high-influence” pool. This yields a weighted or importance-sampled covariance $C_\ell = X_\ell X_\ell^\top$ whose top-$k$ eigenspace more faithfully captures the retain subspace seen over the full retain stream, while keeping the per-teleportation cost at $\mathcal{O}(|\mathcal{B}_{\mathrm r}| d_\ell k)$. Combined with low-rank SVD implementations that are already optimized for LLM compression, these heuristics make the FastWARP projector construction compatible with large transformer architectures without breaking the retain loss invariance enforced by WARP.
We leave the adaptation to large language models for future research. Our contributions target symmetry-based defenses for generic neural networks and established MU baselines, and do not address LLM-specific challenges in unlearning, which constitute a distinct line of investigation.

\section{Comparison With DP–Langevin Noise Defences}
\label{sec:dp-comparison}

While our goal is to make neural networks more resilient to privacy attacks \emph{post~hoc}, a natural question is how \textsc{WARP} compares with defences based on differential privacy (DP).  
DP is the strongest known framework for providing indistinguishability guarantees between neighbouring datasets, and a small number of recent unlearning methods have attempted to translate these guarantees into \emph{certified} machine unlearning.  
Among these, noisy-gradient (Langevin) approaches provide the closest analogue to our setting; we therefore include them as a comparison point.

% \paragraph{Why compare with DP–Langevin?}
Certified unlearning methods such as \citet{guo2020certified,chien2024langevin} formalise unlearning as an indistinguishability requirement between (i) a model obtained by training on the full dataset, and (ii) a counterfactual model that has never seen the forget set.  
These works build on the principle that if the training algorithm is itself DP, then suitable post-processing can yield certified removal of training points.  
Such guarantees make DP–Langevin the strongest known \emph{general-purpose} defence with explicit indistinguishability guarantees, hence a meaningful baseline to evaluate privacy–utility trade-offs.

\paragraph{What the DP guarantees actually require.}
The formal guarantees in \citet{guo2020certified,chien2024langevin} rely on assumptions that do \emph{not} hold in the deep, non-convex MU regime we consider:
\begin{enumerate}[leftmargin=1.5em]
    \item \textbf{Convexity and strong dissipativity.}  
    Both works require (strongly) convex, $\ell_2$–regularised objectives to bound the stationary distribution of the noisy dynamics.  
    Deep convolutional networks trained with cross-entropy fundamentally violate these assumptions.
    \item \textbf{DP-trained initial model required.}  
    The certified-unlearning guarantee requires that the \emph{original} model be obtained using \emph{the same} noisy-gradient mechanism (noisy SGD or Langevin) applied throughout training on the full dataset.  
    This is explicitly stated as a necessary condition in \citet{chien2024langevin}.  
    In contrast, our setting begins from a standard ERM-trained model, which is non-DP and therefore outside the scope of their certification theorem.
\end{enumerate}
As a result, the ``$\varepsilon$'' obtained from the RDP accountant in our experiments should be interpreted purely as a calibrated \emph{noise level}, not as a valid DP guarantee.  
Our use of Langevin noise is therefore a \emph{strong noise-based defence}, not a certified mechanism.

\paragraph{Adapting projected Langevin unlearning to MU.}
Following \citet{chien2024langevin}, we implement projected Langevin dynamics on top of the same MU objective used throughout the paper.  
For a per-sample clipped gradient with radius $C$ and loss
\[
\mathcal{L}_{\mathrm{MU}}(\theta)
    = \alpha\,\big(\ell_{\mathrm r}(\theta) + \lambda\|\theta-\theta_p\|_2^2\big)
      - (1-\alpha)\,\ell_{\mathrm f}(\theta),
\]
the DP–Langevin update is
\begin{align}
   g_t &= \mathrm{clip}\big(\nabla_\theta \mathcal{L}_{\mathrm{MU}}(\theta_t),\, C \big), \\
   \theta_{t+1}
   &= \theta_t
      - \eta_t g_t
      + \sqrt{2\,\eta_t\,\lambda}\;\boldsymbol{\xi}_t,
      \qquad \boldsymbol{\xi}_t \sim \mathcal{N}(0,I),
\end{align}
where $\lambda$ is the regularisation parameter entering the RDP privacy analysis.  
Given a target privacy level $\varepsilon$, we follow the exact R\'enyi-DP accounting of \citet{chien2024langevin} to compute the Gaussian noise standard deviation $\sigma$ required by their Langevin update.  
In our implementation, three quantities act as tunable hyperparameters: the learning rate $\eta$, the per-sample gradient-clipping radius $C$, and the regularisation coefficient $\lambda$ that appears in the RDP analysis.  
For any chosen $(\eta, C, \lambda)$ and target $\varepsilon$, the formulas of \citet{chien2024langevin} uniquely determine the corresponding noise scale $\sigma$.  
To ensure fairness across baselines, we run the same number of hyperparameter-search trials as for the MU baselines, jointly sweeping $(\eta, C, \lambda)$ to obtain the set of reported results in Table~\ref{tab:dp-langevin-tradeoff}.

\begin{table}[t]
\centering
\caption{\textbf{NGP+WARP vs.\ Langevin noise (U-LiRA, black-box).}
Reported are risks on \emph{all forget samples} and on the \emph{most–memorized} subset (top 5\%), plus test accuracy.
U-LiRA AUC and TPR@0.1\% (FPR) are shown for each setting.}
\vspace{2mm}

\footnotesize
\resizebox{0.85\columnwidth}{!}{%
\begin{tabular}{lccccc}
\toprule
& \multicolumn{2}{c}{All samples (BB)} & \multicolumn{2}{c}{Most–memorized (top 5\%)} & \multicolumn{1}{c}{Acc.} \\
\cmidrule(lr){2-3}\cmidrule(lr){4-5}\cmidrule(lr){6-6}
Method & AUC & TPR@0.1 & AUC & TPR@0.1 & Test \\
\midrule
Langevin ($\varepsilon=1$)  & 0.523 & 0.004 & 0.671 & 0.029 & 0.682 \\
Langevin ($\varepsilon=4$)  & 0.571 & 0.006 & 0.766 & 0.048 & 0.718 \\
Langevin ($\varepsilon=8$)  & 0.627 & 0.020 & 0.912 & 0.166 & 0.771 \\
Langevin ($\varepsilon=16$) & 0.650 & 0.027 & \textbf0.935 & 0.224 & \textbf{0.798} \\
\midrule
NGP + WARP                  & \textbf{0.516} & \textbf{0.003} & \textbf{0.598} & \textbf{0.015} & 0.797 \\
\bottomrule
\end{tabular}
}
\label{tab:dp-langevin-tradeoff}
\end{table}

\paragraph{Interpretation under non-convexity.}
Although the privacy accountant yields a numerical $\varepsilon$, none of the formal conditions needed for DP-certified unlearning hold for our deep ResNet models.  
Consequently, we reiterate that the resulting values should not be interpreted as DP guarantees but rather as a systematic way of calibrating the magnitude of injected noise.  
The comparison therefore isolates the \emph{empirical} effect of noise injection on forgetting, retention, and attack success.

\paragraph{Empirical privacy--utility trade-off.}
Table~\ref{tab:dp-langevin-tradeoff} reveals a clear tension between nominal DP guarantees and empirical membership privacy.  
As the target privacy budget for Langevin is relaxed from $\varepsilon=1$ to $\varepsilon=16$, test accuracy gradually recovers (from $0.682$ up to $0.798$), but U-LiRA risk monotonically \emph{increases}: the all-sample AUC rises from $0.523$ to $0.650$, and the AUC on the top-$5\%$ most memorised points grows from $0.671$ to $0.935$, with TPR@0.1\% FPR increasing from $0.029$ to $0.224$.  
In contrast, \textsc{NGP+WARP} simultaneously achieves competitive utility and strictly lower attack success: on all forget samples it attains the best AUC and TPR@0.1\% ($0.516$ and $0.003$), and on the most–memorised subset it reduces AUC to $0.598$ and TPR@0.1\% to $0.015$, outperforming every Langevin configuration by a wide margin.  
Notably, relative to the lowest-noise setting ($\varepsilon=16$), \textsc{NGP+WARP} matches accuracy ($0.797$ vs.\ $0.798$) while cutting the memorised AUC from $0.935$ to $0.598$ and TPR@0.1\% from $0.224$ to $0.015$.  
For stronger nominal privacy ($\varepsilon=1$ or $4$), Langevin noise severely degrades accuracy (down to $0.682$) yet still leaves substantially higher attack AUC and TPR than \textsc{WARP}.  
Overall, these results suggest that isotropic DP noise is poorly aligned with the specific memorization patterns exploited by U-LiRA: it injects substantial randomness into all updates, harming utility without reliably protecting the most vulnerable examples, whereas \textsc{WARP} reshapes the parameter space in a targeted way that yields a markedly better empirical privacy--utility frontier.

% \paragraph{Takeaway.}
Taken together, these observations clarify the roles of the two approaches.  
Langevin noise offers a principled mechanism for \emph{certified} unlearning in the restricted setting of convex, DP-trained models, but its guarantees do not extend to the non-convex MU regime nor to pretrained models obtained without DP noise.  
Consequently, applying Langevin updates post hoc to deep networks provides no formal protection and yields an unfavourable privacy–utility trade-off in practice.  
By contrast, \textsc{WARP} operates directly on arbitrary pretrained models, targets the directions most responsible for memorization, and empirically achieves substantially stronger resistance to membership inference at comparable accuracy.  
A compelling direction for future work is to investigate whether the geometric structure exploited by \textsc{WARP} can be combined with, or serve as a foundation for, certified unlearning mechanisms that simultaneously handle non-convex objectives and non-DP initialisation—a capability not supported by current DP-Langevin frameworks.

\section{Adaptive Reconstruction With Symmetry--Aware Attacker}
\label{sec:adaptive-recon}

Teleportation acts by composing the unlearning update with a symmetry transform that preserves predictions but redistributes parameter mass along loss–invariant directions (Section~\ref{sec:method-teleportation-based-defense}).  
This raises a natural question: can a stronger white-box adversary, aware of the teleportation family, \emph{invert} or compensate for these symmetry moves and recover the residual forget gradient?  
More concretely, if the attacker can parameterise and optimize over the change-of-basis (COB) scales $\tau$ used in neural teleportation~\citep{armenta2023neural}, does this restore reconstruction quality and defeat WARP?

It is worth noting that our privacy evaluation already includes two adaptive–attack families: U-LiRA and GLiR, both of which instantiate adaptive membership-inference attacks by optimising proxy models or surrogate loss landscapes.
However, the reconstruction attack considered in Section~\ref{sec:method-recon}—which directly targets instance-level recovery of the forgotten data—was \emph{not} adaptive: the attacker optimized only over the dummy image while keeping the teleportation parameters fixed.
To fully test the robustness of symmetry-based teleportation, we now consider a strictly stronger attacker that \emph{jointly} optimizes both the dummy image and the teleportation parameters themselves.

Concretely, we study whether an attacker who can parameterise and optimize over the change-of-basis (COB) symmetry scales $\tau$ used in neural teleportation~\citep{armenta2023neural} can undo the defender’s symmetry moves, thereby restoring the clean gradient geometry required for successful reconstruction.
This experiment directly probes whether teleportation is merely hiding the forget gradient behind a reversible reparameterisation, or whether it fundamentally reshapes the inverse problem faced by reconstruction attacks.

\paragraph{Attack formulation.}
In the adaptive setting, we give the attacker full knowledge of the teleportation family and let them \emph{shadow} the defender’s operations.  
Specifically, starting from the original pretrained weights $\theta_{\mathrm{org}}$, the attacker first applies a change-of-basis symmetry parametrised by COB scales $\tau = \{\tau_a > 0\}$, obtaining
\begin{equation}
  \theta_{\mathrm{org}}^{(\tau)} \;=\; T_\tau(\theta_{\mathrm{org}}),
\end{equation}
where $T_\tau$ is the COB teleportation map (Appendix~\ref{sec:appendix-alternative}).  
They then perform a single gradient step in parameter space using a dummy image--label pair $(x,y)$:
\begin{equation}
  \theta^{(\tau)}(x,y)
  \;=\;
  \theta_{\mathrm{org}}^{(\tau)}
  \;+\;
  \eta_{\mathrm{att}}\,
  \nabla_\theta \ell\big(f(x;\theta_{\mathrm{org}}^{(\tau)}), y\big),
  \label{eq:adaptive-shadow-step}
\end{equation}
with attack step size $\eta_{\mathrm{att}}>0$.  
The attacker’s goal is to choose $(x,\tau)$ so that the shadowed update in~\eqref{eq:adaptive-shadow-step} closely matches the actual unlearned parameters $\theta_u$ produced by WARP.  
Formally, we solve
\begin{equation}
  \hat x_f, \hat\tau
  \;\in\;
  \argmin_{x,\tau}
  \Big[
    D\big(
      \theta^{(\tau)}(x,y),
      \theta_u
    \big)
    + \lambda_{\mathrm{TV}}\operatorname{TV}(x)
    + \lambda_{\tau}\,\Omega(\tau)
  \Big],
  \label{eq:adaptive-obj}
\end{equation}
where $D(\cdot,\cdot)$ is a parameter-space discrepancy (we use $\ell_2$ distance over all weights), $\operatorname{TV}(x)$ is the total-variation regulariser on the image, and $\Omega(\tau)$ implements a Gaussian prior $\tau_a \sim \mathcal{N}(1,\sigma_{\mathrm{cob}}^2)$ on each COB scale.  
We optimize~\eqref{eq:adaptive-obj} by alternating gradient steps on $x$ and $\tau$, with $\tau$ clipped to a bounded interval around~$1$ to avoid degenerate scalings.

\paragraph{Experimental setup.}
For a fair comparison, we reuse exactly the reconstruction protocol of Section~\ref{sec:exp-recon} (same model, dataset, forgotten examples, optimizer, and image priors), and only extend the attack to optimize over the COB parameters $\tau$ via~\eqref{eq:adaptive-obj}.  
We vary the COB prior variance $\sigma_{\mathrm{cob}}$ that defines $\Omega(\tau)$, treating each $\tau_a$ as a scalar random variable centred at $1$ with variance $\sigma_{\mathrm{cob}}$.  
We sweep $\sigma_{\mathrm{cob}} \in {0, 0.1, 0.2, 0.4, 0.8}$, where $\sigma_{\mathrm{cob}}=0$ recovers the non-adaptive attack with fixed $\tau\equiv 1$, and larger values correspond to stronger dispersion along the symmetry orbit induced by WARP.  
Following the evaluation protocol of Table~\ref{tab:recon_teleport}, we quantify reconstruction quality using PSNR, SSIM, LPIPS, and feature MSE, reporting averages over 30 randomly drawn forget examples.

% \begin{table}[t]
% \caption{\textbf{Adaptive reconstruction under change-of-basis teleportation (NGP, ImageNet-1K).}
% We report reconstruction quality as a function of the COB variance $\sigma^2_{\mathrm{cob}}$ used in the symmetry prior $\Omega(\tau)$.
% Higher is better for PSNR/SSIM; lower is better for LPIPS/MSE.
% Even when jointly optimising over the image and COB parameters, increasing symmetry variance consistently degrades reconstruction.}
% \vspace{2mm}
% \centering
% \footnotesize
% \resizebox{0.9\linewidth}{!}{%
% \begin{tabular}{lccccc}
% \toprule
% $\sigma^2_{\mathrm{cob}}$ 
% & PSNR (dB) $\uparrow$
% & LPIPS (VGG) $\downarrow$
% & LPIPS (Alex) $\downarrow$
% & SSIM $\uparrow$
% & Feat MSE $\downarrow$ \\
% \midrule
% $0$ (no COB noise)      & \dots & \dots & \dots & \dots & \dots \\
% $\sigma_1^2$ (low var.) & \dots & \dots & \dots & \dots & \dots \\
% $\sigma_2^2$ (med. var.)& \dots & \dots & \dots & \dots & \dots \\
% $\sigma_3^2$ (high var.)& \dots & \dots & \dots & \dots & \dots \\
% \bottomrule
% \end{tabular}
% }
% \label{tab:adaptive-recon-cob}
% \end{table}

\begin{figure}[t]
    \centering
    \begin{minipage}{0.48\linewidth}
        \centering
        \includegraphics[width=\linewidth]{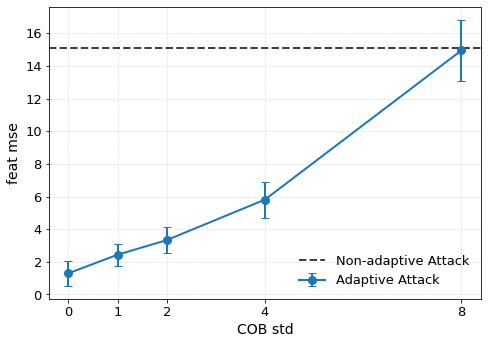}
        \vspace{1mm}
        \textbf{(A)}~Feature MSE
    \end{minipage}
    \hfill
    \begin{minipage}{0.48\linewidth}
        \centering
        \includegraphics[width=\linewidth]{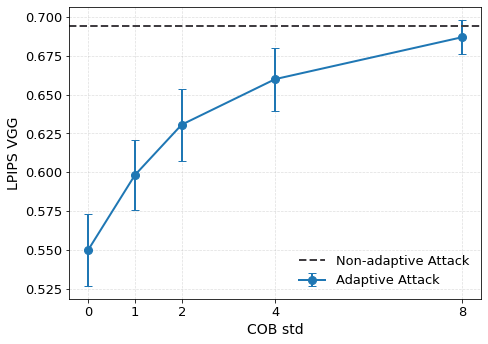}
        \vspace{1mm}
        \textbf{(B)}~LPIPS (VGG)
    \end{minipage}
    \vspace{2mm}
    \caption{\textbf{Adaptive reconstruction under change-of-basis teleportation (NGP, ImageNet-1K).}
(A) Feature MSE and (B) LPIPS (VGG) as a function of the COB standard deviation $\sigma_{\mathrm{cob}}$.
Increasing the symmetry variance consistently worsens reconstruction quality across both metrics.}
    \label{fig:adaptive-recon-cob}
\end{figure}

% \paragraph{Results and connection to theory.}
% Empirically, allowing the attacker to optimize over $\tau$ does \emph{not} recover the sharp reconstructions observed for NGP without teleportation.  
% Across all tested settings, the adaptive symmetry-aware attack yields substantially worse PSNR and SSIM and larger LPIPS / feature MSE than the non-teleported NGP baseline, and reconstruction quality \emph{degrades monotonically} as $\sigma_{\mathrm{cob}}$ increases (Table~\ref{tab:adaptive-recon-cob}).  
% Qualitatively, the recovered images collapse toward class prototypes or generative-prior artefacts, with fine-grained details of the forgotten samples largely absent.

\paragraph{Results and connection to theory.}
Figure~\ref{fig:adaptive-recon-cob} shows how reconstruction quality changes as we increase the COB prior std $\sigma_{\mathrm{cob}}$ that controls the spread of admissible symmetry scales.  
When $\sigma_{\mathrm{cob}}=0$ the symmetry prior collapses around $\tau_a\approx 1$, so the attacker effectively searches over a narrow neighbourhood of the defender’s true teleportation and can partially recover the forgotten signal: the adaptive attack achieves substantially lower feature MSE and LPIPS than the non–adaptive WARP attack (dashed line).  
However, the attacker never observes the ground-truth COB scales used by the defender; as $\sigma_{\mathrm{cob}}$ grows and the symmetry orbit broadens, the optimisation over $(x,\tau)$ quickly becomes unstable.  
Both metrics deteriorate almost monotonically with $\sigma_{\mathrm{cob}}$: already at moderate variance the gains over the non–adaptive attack largely disappear, and for the largest tested $\sigma_{\mathrm{cob}}$ the adaptive reconstructions are statistically indistinguishable from (or slightly worse than) the non–adaptive baseline.
Importantly, the COB standard deviation is a defender-controlled knob: in this symmetry family we can raise $\sigma_{\mathrm{cob}}$ up to $1.0$ without changing the realised network function, and in our main reconstruction experiments in Table~\ref{tab:recon_teleport} we set $\sigma_{\mathrm{cob}}=0.8$, already placing the attacker in a high-variance regime where adaptive reconstruction is strongly impaired.

This trend is consistent with our theoretical analysis in Appendix~\ref{sec:appendix-theoretical}, which shows that the expected reconstruction error increases with the variance of the COB scales.  
Larger $\sigma_{\mathrm{cob}}$ expands the symmetry orbit of $\theta_{\mathrm{org}}$ and $\theta_u$, so the update $\Delta\theta$ admits many symmetry–equivalent decompositions whose gradients are nearly orthogonal to the true forget gradient $g_f$.  
The optimisation problem in~\eqref{eq:adaptive-obj} thus becomes a highly ill-posed inverse problem over the joint space $(x,\tau)$, where many different configurations of $(x,\tau)$ produce similar matches in parameter space.  
Empirically, the adaptive optimiser drifts toward such low–signal-to-noise solutions that satisfy the symmetry constraints but no longer encode the specific forgotten example, explaining the systematic degradation in reconstruction quality as symmetry variance (or std) increases.

% This behavior matches our theoretical analysis in the appendix~\ref{sec:appendix-theoretical}, which shows that the expected reconstruction error increases with the variance of the COB scales.  
% Intuitively, larger $\sigma_{\mathrm{cob}}$ expands the symmetry orbit of $\theta_{\mathrm{org}}$ and $\theta_u$, so the update $\Delta\theta$ can be decomposed into many symmetry–equivalent components that are nearly orthogonal to the true forget gradient $g_f$.  
% The optimisation in~\eqref{eq:adaptive-obj} must then solve a more ill-posed inverse problem over a higher-dimensional search space $(x,\tau)$, in which many configurations of $(x,\tau)$ produce similar gradient matches.  
% In practice, the optimizer gravitates toward low–signal-to-noise configurations that satisfy the symmetry constraints but no longer encode the specific forgotten example.

\paragraph{Takeaway.}
Even under a strong white-box threat model—where the attacker knows the teleportation family and jointly adapts both the dummy input and the symmetry parameter—teleportation continues to disrupt reconstruction effectively.
The injected symmetry components become entangled with the forget-induced update $\Delta\theta$, enlarging the attacker’s search space and destroying the geometric alignment between parameter differences and the underlying forgotten example.
Thus, teleportation does not merely reparameterise the model in a way that can be inverted; instead, by injecting symmetry variance into the update, it structurally increases reconstruction error and removes the clean gradient-based signal that standard reconstruction attacks depend on.
This provides empirical and theoretical evidence that symmetry-based teleportation fundamentally hardens the inverse problem faced by adaptive adversaries.

\section{Teleportation-Aware Information-Theoretic Bounds on Gradient-Based Reconstruction}
\label{sec:appendix-theoretical}

\subsection{Overview of the Theoretical Analysis}

This appendix develops an information-theoretic lower bound on the
minimal reconstruction mean-squared error (MSE) achievable by a
gradient-based inversion adversary within a shared probabilistic
model for gradients. We first adapt standard
entropy--MSE relationships to the case where the attacker observes
gradients rather than intermediate features, closely following the
spirit of the analysis in Xia et al.~\cite{xia2025theoretical}. We
then introduce a Gaussian-mixture model (GMM) for gradient features
and derive a parametric lower bound on the conditional entropy
$H(x\mid g)$, analogous to the intermediate-feature analysis in
\cite{xia2025theoretical} but specialized to gradients. Finally, we
incorporate teleportation (change-of-basis) noise as private
randomness in the training dynamics and analyze its impact on the
\emph{same} lower-bound pipeline, under an explicit diagonal
approximation and an energy-preserving design assumption on the
change-of-basis distribution. Throughout, we keep the modelling
assumptions identical between the teleported and non-teleported
channels, so any improvement we prove directly reflects a genuine
tightening of the analytic lower bound on reconstruction error---and
hence a provable gain in information-theoretic privacy \emph{within
this common generative framework}. We emphasize that $H(x)$ is fixed
by the dataset distribution, so only \emph{relative} differences
between the channels are meaningful.

\subsection{Setup and threat model}

\paragraph{Data and model.}
Let $x \in \R^d$ denote the $d$-dimensional input random variable, distributed
according to some unknown data distribution on a measurable subset
$\mathcal{X} \subseteq \R^d$. We assume throughout that $x$ admits a
density w.r.t.\ Lebesgue measure and has finite second moment.
(If one wishes to model discrete or manifold-supported data, the
analysis can be recovered by adding an arbitrarily small Gaussian
perturbation to $x$ as is standard in differential-entropy arguments;
we implicitly assume such smoothing has been applied so that
conditional covariances below are positive definite.)

Consider a deep network with parameters $W \in \mathbb{R}^{m \times d}$ and first-layer
pre-activations
\[
  z = W x \in \R^m,
\]
and a subsequent decoder $F_d$. Let $\ell(\cdot,y)$ be a loss for a label $y$, and define
the gradient with respect to $z$:
\[
  g_z \;=\; \nabla_z \ell(F_d(z),y) \in \R^m.
\]
In the analysis below, the attacker’s observation will be a
gradient-based signal $g$ (not necessarily equal to $g_z$ directly)
that is deterministically related to $(x,y,W)$ plus noise. In a
white-box setting, for instance, the adversary can observe weight
differences across steps, which are affine functions of the
underlying gradient features; since mutual information and our
entropy-based bounds are invariant under fixed invertible affine
reparametrizations, it is without loss of generality to work with a
canonical gradient feature $g$.

\paragraph{Adversarial objective.}
An inversion adversary aims to reconstruct $x$ from the observable
$g$. Given an estimator $\hat x(g)$, we measure reconstruction
quality by the mean-squared error (MSE)
\begin{equation}
  \xi_g(\hat x)
  \;\coloneqq\;
  \frac{1}{d}\,\E\bigl[\|x - \hat x(g)\|_2^2\bigr].
  \label{eq:mse-def}
\end{equation}
The \emph{minimal} MSE $\xi_g$ is the infimum of~\eqref{eq:mse-def}
over all measurable estimators $\hat x(\cdot)$. We interpret
``information-theoretic robustness'' as the regime where the attacker
is Bayes-optimal under the assumed generative model, i.e.\ has access
to the true posterior $P(x\mid g)$ induced by that model and
implements the Minimum Mean Square Error (MMSE) estimator.

\begin{assumption}[Basic regularity]
\label{ass:basic}
We assume:
\begin{enumerate}[label=(\roman*), nosep]
  \item $x$ has a density on $\R^d$ and finite second moment;
  \item for the observation $g$, the conditional distribution
        $P(x\mid g)$ admits a density with finite second moment, and
        its covariance matrix $\Cov(x\mid g)$ is positive definite
        almost surely;
  \item all entropies, mutual informations and expectations used
        below are finite.
\end{enumerate}
\end{assumption}

These conditions are standard in information-theoretic MMSE analysis
(see, e.g.,~\cite{xia2025theoretical}) and ensure that all quantities
are well-defined and that the maximum-entropy characterization for
Gaussians can be applied without degeneracy.

\subsection{Minimal MSE from gradients and an entropy-based lower bound}

\paragraph{Bayes-optimal reconstruction from gradients}

We first recall the standard MMSE characterization.

\begin{proposition}[Minimal reconstruction MSE from gradients]
\label{prop:min-mse-grad}
Let $x \in \R^d$ and an observation $g$ satisfy
\Cref{ass:basic}. Consider estimators $\hat x(g)$ of $x$ based on $g$
and define $\xi_g(\hat x)$ as in~\eqref{eq:mse-def}. Then:
\begin{enumerate}[label=(\roman*),nosep]
\item The estimator that minimizes $\xi_g(\hat x)$ is the conditional
      mean $\hat x^\star(g) = \E[x\mid g]$.
\item The corresponding minimal MSE is
\begin{equation}
  \xi_g
  \;\coloneqq\;
  \inf_{\hat x} \xi_g(\hat x)
  \;=\;
  \frac{1}{d}\,\E_g\Bigl[\Tr\bigl(\Cov(x\mid g)\bigr)\Bigr],
  \label{eq:xi-covxg}
\end{equation}
where $\Cov(x\mid g)$ denotes the conditional covariance of $x$ given
$g$ and $\E_g$ is expectation w.r.t.\ $g$.
\end{enumerate}
\end{proposition}

\begin{proof}
For any fixed $g$, the conditional risk
$\E[\|x - \hat x(g)\|_2^2 \mid g]$ is uniquely minimized by
$\hat x^\star(g) = \E[x\mid g]$ (standard MMSE theory, cf.\
\cite{xia2025theoretical}). The minimal conditional risk at $g$ is
\[
  \E\bigl[\|x - \E[x\mid g]\|_2^2 \mid g\bigr]
  \;=\; \Tr\bigl(\Cov(x\mid g)\bigr),
\]
since for any random vector $X$ with mean $\mu$ and covariance
$\Sigma$ one has $\E\|X-\mu\|_2^2 = \Tr(\Sigma)$. Taking
expectation over $g$ and dividing by $d$ yields~\eqref{eq:xi-covxg}.
\end{proof}

Thus, when we refer to the ``minimal MSE achievable by an attacker''
for a given observation model, we mean $\xi_g$ as given
in~\eqref{eq:xi-covxg}, corresponding to a Bayes-optimal adversary
within that model.

\paragraph{An entropy-based lower bound on the minimal MSE}

We now relate the minimal MSE $\xi_g$ to the conditional entropy
$H(x\mid g)$, generalizing standard entropy--MMSE inequalities (cf.\
\cite{xia2025theoretical}).

\begin{theorem}[Entropy-based lower bound on gradient reconstruction]
\label{thm:entropy-mse-grad}
Under \Cref{ass:basic}, let $H(x\mid g)$ be the conditional
differential entropy of $x$ given the observation $g$. Then the
minimal reconstruction MSE $\xi_g$ in~\eqref{eq:xi-covxg} satisfies
\begin{equation}
  \xi_g
  \;\ge\;
  \frac{1}{2\pi e}\,\exp\!\Bigl(\frac{2}{d}\,H(x\mid g)\Bigr).
  \label{eq:mse-lb-entropy}
\end{equation}
\end{theorem}

\begin{proof}
Fix $g$ and define $\Sigma(g) \coloneqq \Cov(x\mid g)$. Under
\Cref{ass:basic}, $\Sigma(g)$ is symmetric and positive definite
almost surely. For each such $g$, the conditional distribution of $x$
given $g$ has entropy bounded above by that of a Gaussian with the
same covariance:
\[
  H(x\mid g=g)
  \;\le\;
  \frac{1}{2}\log\bigl((2\pi e)^d \det(\Sigma(g))\bigr),
\]
with equality iff $x\mid g$ is Gaussian. This is the usual maximum
entropy property of Gaussians. Taking expectation over $g$ gives
\begin{equation}
  H(x\mid g)
  \;=\;
  \E_g\bigl[H(x\mid g=g)\bigr]
  \;\le\;
  \E_g\Bigl[\tfrac12 \log\bigl((2\pi e)^d \det(\Sigma(g))\bigr)\Bigr].
  \label{eq:entropy-upper-cov}
\end{equation}
Let $\lambda_1(g),\dots,\lambda_d(g)$ be the eigenvalues of
$\Sigma(g)$ (all positive). Then
\[
  \det(\Sigma(g)) = \prod_{j=1}^d \lambda_j(g),
  \quad
  \Tr(\Sigma(g)) = \sum_{j=1}^d \lambda_j(g).
\]
By the Arithmetic Mean-Geometric Mean (AM--GM) inequality,
\[
  \prod_{j=1}^d \lambda_j(g)
  \;\le\;
  \Bigl(\frac{1}{d}\sum_{j=1}^d \lambda_j(g)\Bigr)^d
  = \Bigl(\frac{\Tr(\Sigma(g))}{d}\Bigr)^d,
\]
so
\[
  \log\det(\Sigma(g))
  \;\le\;
  d \log\Bigl(\frac{\Tr(\Sigma(g))}{d}\Bigr).
\]
Substituting into~\eqref{eq:entropy-upper-cov},
\[
  H(x\mid g)
  \;\le\;
  \E_g\Bigl[
    \tfrac12 \log\bigl((2\pi e)^d \det(\Sigma(g))\bigr)
  \Bigr]
  \;\le\;
  \E_g\Bigl[
    \frac{d}{2}\log\Bigl(2\pi e \frac{\Tr(\Sigma(g))}{d}\Bigr)
  \Bigr].
\]
Since $\log(\cdot)$ is concave, Jensen’s inequality yields
\[
  \E_g\Bigl[
    \log\Bigl(2\pi e \frac{\Tr(\Sigma(g))}{d}\Bigr)
  \Bigr]
  \;\le\;
  \log\Bigl(2\pi e \frac{\E_g[\Tr(\Sigma(g))]}{d}\Bigr).
\]
Therefore
\begin{equation}
  H(x\mid g)
  \;\le\;
  \frac{d}{2}\log\Bigl(
    2\pi e \frac{\E_g[\Tr(\Sigma(g))]}{d}
  \Bigr).
  \label{eq:H-upper-Trace}
\end{equation}
By \Cref{prop:min-mse-grad},
$\E_g[\Tr(\Sigma(g))] = d\,\xi_g$, so~\eqref{eq:H-upper-Trace}
becomes
\[
  H(x\mid g)
  \;\le\;
  \frac{d}{2}\log(2\pi e\,\xi_g).
\]
Rearranging,
\[
  \log(2\pi e\,\xi_g) \;\ge\; \frac{2}{d}H(x\mid g),
  \qquad
  2\pi e\,\xi_g \;\ge\; \exp\!\Bigl(\frac{2}{d}H(x\mid g)\Bigr),
\]
which yields~\eqref{eq:mse-lb-entropy}.
\end{proof}

Note that $H(x)$---and hence the absolute scale of these lower
bounds---is fully determined by the underlying dataset distribution
and does not depend on teleportation. In our comparisons between
teleported and non-teleported channels, $H(x)$ cancels and only
\emph{differences} or ratios matter.

\subsection{A parametric lower bound on $H(x| g)$ via Gaussian mixtures}

We now introduce a specific probabilistic model for the gradient
signal and derive a tractable parametric lower bound on $H(x\mid g)$.
The modelling choices mirror those used for intermediate features in
\cite{xia2025theoretical}, but here are applied to gradients.

\subsubsection{Gradient feature and observation model}

\paragraph{Clean gradient feature.}

Let $G: \mathbb{R}^{d} \rightarrow \mathbb{R}^{m}$ be a deterministic mapping producing a \emph{clean} gradient feature from input $x$. Specifically, let $u = G(x) \in \R^m$ denote a feature
derived deterministically from $(x,y,W)$ (e.g., the gradient with
respect to first-layer pre-activations, or a flattened stack of
first-layer weight gradients). Thus $u$ is a deterministic function
of $x$ once the model and label are fixed.

\begin{assumption}[Gaussian Mixture Model (GMM) for $u$]
\label{ass:gmm-u}
We assume that the marginal distribution of $u$ can be well
approximated by a Gaussian mixture
\begin{equation}
  u \sim \sum_{i=1}^K \pi_i\,\N(\mu_i,\Sigma_i),
  \quad
  \sum_{i=1}^K \pi_i = 1,\quad \pi_i > 0,\quad \Sigma_i \succ 0.
  \label{eq:gmm-u}
\end{equation}
\end{assumption}

This GMM assumption is standard in information-theoretic analyses of
representations~\cite{xia2025theoretical} and serves as our common
surrogate model for gradient features.

\paragraph{Noisy gradient observation.}
We model the attacker’s baseline observation as a noisy version of
$u$:
\begin{equation}
  g_0 = u + \varepsilon,
  \quad
  \varepsilon \sim \N(0,\Sigma_g),\quad
  \varepsilon \perp (x,u),
  \label{eq:g-channel}
\end{equation}
where $\Sigma_g \succ 0$ is a fixed positive-definite covariance
matrix. This captures gradient perturbations due to stochastic
training, subsampling, or other noise sources; $\Sigma_g$ is assumed
known to the attacker, as in~\cite{xia2025theoretical}. 
% This is a
% stylized channel model rather than a literal description of SGD
% noise.
We use this Gaussian channel as the standard abstraction of gradient perturbations for the subsequent information-theoretic analysis.

\subsubsection{A mutual-information identity for deterministic features}

We will repeatedly use the following simple lemma for deterministic
features.

\begin{lemma}[Mutual information for deterministic feature maps]
\label{lem:mi-deterministic}
Let $u = G(x)$ be a deterministic function of $x$, and let $g$ be a
random variable such that $p(g\mid x,u) = p(g\mid u)$
(i.e., $g$ depends on $(x,u)$ only through $u$). Then
\[
  I(x;g) = I(u;g).
\]
\end{lemma}

Where $I(x; g)$ denotes the mutual information between $x$ and $g$.

\begin{proof}
Since $u$ is a deterministic function of $x$, we have
$H(u\mid x)=0$ and $H(x,u) = H(x)$. Moreover,
$p(g\mid x) = p(g\mid u)$ by the conditional-independence
assumption, so
\[
  H(g\mid x) = \E_x H(g\mid x=x)
             = \E_x H(g\mid u=G(x))
             = H(g\mid u).
\]
Therefore
\[
  I(x;g)
  = H(g) - H(g\mid x)
  = H(g) - H(g\mid u)
  = I(u;g).\qedhere
\]
\end{proof}

We will apply this lemma to both the baseline channel $g_0$ and the
teleported channel $g$ below.

\subsubsection{Parametric GMM-based lower bound on $H(x\mid g_0)$}

We now adapt the mixture-entropy bound used in
\cite{xia2025theoretical} to gradients.

\begin{theorem}[Parametric lower bound on $H(x\mid g_0)$]
\label{thm:grad-gmm-lb}
Under \Cref{ass:basic} and \Cref{,ass:gmm-u} and the channel
\eqref{eq:g-channel}, the conditional entropy $H(x\mid g_0)$
satisfies
\begin{equation}
  H(x\mid g_0)
  \;\ge\;
  H(x)
  -
  \sum_{i=1}^K \pi_i\left(
    -\log\pi_i
    + \frac{1}{2}\log\frac{|\Sigma_i + \Sigma_g|}{|\Sigma_g|}
  \right).
  \label{eq:Hxg-lb-gmm}
\end{equation}
\end{theorem}

\begin{proof}
Because $u = G(x)$ is deterministic given $x$, and $g_0$ depends on
$(x,u)$ only through $u$ via~\eqref{eq:g-channel}, we have
$g_0 \perp x \mid u$ and the conditions of \Cref{lem:mi-deterministic}
hold. Thus
\[
  I(x;g_0) = I(u;g_0),
\]
and
\[
  H(x\mid g_0) = H(x) - I(x;g_0) = H(x) - I(u;g_0).
\]
We bound $I(u;g_0)$ from above using the GMM model. We have
\[
  I(u;g_0) = H(g_0) - H(g_0\mid u).
\]
From~\eqref{eq:g-channel}, $g_0\mid u \sim \N(u,\Sigma_g)$, so
\[
  H(g_0\mid u)
  = \tfrac12\log\bigl((2\pi e)^m |\Sigma_g|\bigr).
\]
Marginally, $g_0$ is the convolution of the GMM $u$ with the Gaussian
$\varepsilon$, hence
\[
  g_0 \sim \sum_{i=1}^K \pi_i\,\N(\mu_i,\Sigma_i + \Sigma_g).
\]
For any mixture density $p(z) = \sum_i \pi_i p_i(z)$ with components
$p_i$, the differential entropy satisfies the standard upper bound
\[
  H(p) \;\le\; H(\pi) + \sum_i \pi_i H(p_i),
\]
where $H(\pi) = -\sum_i \pi_i\log\pi_i$ is the discrete entropy of
the mixture weights (this follows by considering the joint entropy of
the component index and the sample). Applying this with Gaussian
components $p_i = \N(\mu_i,\Sigma_i+\Sigma_g)$ yields
\[
  H(g_0)
  \;\le\;
  \sum_{i=1}^K \pi_i\left(
    -\log\pi_i
    + \frac{1}{2}\log\bigl((2\pi e)^m|\Sigma_i + \Sigma_g|\bigr)
  \right),
\]
as in~\cite{xia2025theoretical}. Therefore
% \begin{align*}
%   I(u;g_0)
%   &\le
%   \sum_{i=1}^K \pi_i\left(
%     -\log\pi_i
%     + \frac{1}{2}\log\bigl((2\pi e)^m|\Sigma_i + \Sigma_g|\bigr)
%   \right)
%   - \frac{1}{2}\log\bigl((2\pi e)^m |\Sigma_g|\bigr) \\
%   &=
%   \sum_{i=1}^K \pi_i\left(
%     -\log\pi_i
%     + \frac{1}{2}\log\frac{|\Sigma_i + \Sigma_g|}{|\Sigma_g|}
%   \right),
% \end{align*}

\begin{align*}
  I(u;g_0)
  &\le
  \sum_{i=1}^K \pi_i\left(
    -\log\pi_i
    + \frac{1}{2}\log\bigl((2\pi e)^m|\Sigma_i + \Sigma_g|\bigr)
  \right)
  - \frac{1}{2}\log\bigl((2\pi e)^m |\Sigma_g|\bigr) \\
  &= \sum_{i=1}^K \pi_i\left(
      -\log\pi_i
      + \frac{1}{2}\log\bigl((2\pi e)^m|\Sigma_i + \Sigma_g|\bigr)
    \right)
    + \sum_{i=1}^K \pi_i\left(
      - \frac{1}{2}\log\bigl((2\pi e)^m |\Sigma_g|\bigr)
    \right) \\
  &= \sum_{i=1}^K \pi_i\left(
      -\log\pi_i
      + \frac{1}{2}\log\bigl((2\pi e)^m|\Sigma_i + \Sigma_g|\bigr)
      - \frac{1}{2}\log\bigl((2\pi e)^m |\Sigma_g|\bigr)
    \right) \\
  &= \sum_{i=1}^K \pi_i\left(
      -\log\pi_i
      + \frac{1}{2}\log\frac{|\Sigma_i + \Sigma_g|}{|\Sigma_g|}
    \right).
\end{align*}

where the $(2\pi e)^m$ terms cancel. Substituting into
$H(x\mid g_0) = H(x) - I(u;g_0)$ yields~\eqref{eq:Hxg-lb-gmm}.
\end{proof}

\Cref{thm:grad-gmm-lb} yields a parametric lower bound on
$H(x\mid g_0)$---parametric in the GMM and noise covariances. Via
\Cref{thm:entropy-mse-grad}, this in turn induces a lower bound on
the minimal reconstruction MSE for an attacker observing $g_0$. Our
teleportation analysis will reuse exactly the same ingredients (GMM
approximation and mixture-entropy bound) so comparisons are on equal
footing.

\subsection{Teleportation / change-of-basis noise on gradients}

We now incorporate teleportation (change-of-basis; CoB) symmetry as a
source of private randomness in the gradient dynamics and analyze its
impact on the \emph{same} lower-bound pipeline used for $g_0$.

\subsubsection{Teleportation as private multiplicative noise}

\paragraph{Teleportation structure.}
% For each layer $\ell$, let $\tau^{[\ell]}$ denote a CoB vector with
% non-zero entries. The teleported gradient at layer $\ell$ satisfies
% \[
%   dV^{[\ell]} = \tau^{[\ell-1]} \odot dW^{[\ell]}
%                 \odot \bigl(1 / \tau^{[\ell]}\bigr),
% \]
% Where $\odot$ denotes elementwise multiplication with implicit broadcasting over the appropriate input and output dimensions.
For each layer $\ell$, let $\tau^{[\ell]}$ denote the corresponding CoB vector (with all entries nonzero).
The teleported gradient at layer $\ell$ is obtained by column-scaling with $\tau^{[\ell-1]}$ and row-scaling with $1/\tau^{[\ell]}$, i.e.
\[
    dV^{[\ell]}
    \;=\;
    \tau^{[\ell-1]} \,\bullet\, dW^{[\ell]} \,\bullet\, \bigl(1/\tau^{[\ell]}\bigr),
\]
\[
dV^{[\ell]}_{ij}
\;=\;
\tau^{[\ell-1]}_j \;
dW^{[\ell]}_{ij} \;
\bigl(1/\tau^{[\ell]}_i\bigr),
\]
where the left operation multiplies each column of $dW^{[\ell]}$ by the corresponding coordinate of $\tau^{[\ell-1]}$, and the right operation multiplies each row by the corresponding coordinate of $1/\tau^{[\ell]}$.
Consequently, each gradient entry acquires a multiplicative factor equal to a ratio of CoB coordinates.
As such, each gradient entry picks up a multiplicative factor equal to a
ratio of CoB entries. Flattening all gradient parameters into a
single vector, we write the clean gradient feature as $u$ and its
teleported version as
\begin{equation}
  \tilde u = R(\tau)\,u,
  \label{eq:teleported-u}
\end{equation}
where $R(\tau)$ is a diagonal matrix with entries
$r_j(\tau) = \tau_{b(j)}/\tau_{a(j)}$ corresponding to the
appropriate input/output channels $(a(j),b(j))$ of coordinate $j$.
In practice, these ratios are constrained by the underlying
channel-wise $\tau^{[\ell]}$ structure; our analysis below treats
$\{r_j(\tau)\}$ as effective per-coordinate scalings induced by that
structure.

\paragraph{Threat model for teleportation.}
We adopt the following threat model.

\begin{assumption}[Teleportation threat model]
\label{ass:tele-threat}
\mbox{}
\begin{enumerate}[label=(\roman*),nosep]
  \item The CoB parameters $\tau$ are sampled from a distribution
        $P_\tau$ that is independent of $(x,u)$.
  \item Teleportation is applied internally in the training update
        rule, so that the observable gradient feature (e.g.,
        weight differences across a step) is a function of $\tilde u$
        rather than $u$. Algebraically, this yields an observation
        of the form~\eqref{eq:g-tele} below.
  \item The adversary has white-box access to the model architecture
        and weights but \emph{does not} observe $\tau$ directly.
        They know the distribution $P_\tau$.
\end{enumerate}
\end{assumption}

Under \Cref{ass:tele-threat}, the teleported observation channel is
\begin{equation}
  g = \tilde u + \varepsilon
  = R(\tau)\,u + \varepsilon,
  \quad
  \varepsilon \sim \N(0,\Sigma_g),
  \quad
  \varepsilon \perp (x,u,\tau).
  \label{eq:g-tele}
\end{equation}
This is the same additive-noise form as in~\eqref{eq:g-channel},
applied to a multiplicatively perturbed feature $R(\tau)u$.

\subsubsection{Teleportation-aware entropy lower bound}

We now derive the teleportation-aware counterpart of
\Cref{thm:grad-gmm-lb}, using the same GMM approximation for $u$.
Here the relevant mutual-information identity is again supplied by
\Cref{lem:mi-deterministic}.

\begin{theorem}[Teleportation-aware lower bound on $H(x\mid g)$]
\label{thm:teleportation-Hxg}
Under \Cref{ass:basic}, \Cref{ass:gmm-u}, \Cref{ass:tele-threat} and the teleported
channel~\eqref{eq:g-tele}, the conditional entropy $H(x\mid g)$
satisfies
\begin{equation}
  H(x\mid g)
  \;\ge\;
  H(x)
  -
  \sum_{i=1}^K \pi_i\left(
    -\log\pi_i
    + \frac{1}{2}\,\E_\tau\log\frac{|R(\tau)\Sigma_i R(\tau)^\top + \Sigma_g|}
                                         {|\Sigma_g|}
  \right),
  \label{eq:Hxg-tele}
\end{equation}
where the expectation is taken w.r.t.\ $\tau \sim P_\tau$.
\end{theorem}

\begin{proof}
As before, $u = G(x)$ is deterministic given $x$, and $g$ depends on
$(x,u)$ only through $(u,\tau)$ via~\eqref{eq:g-tele}. In particular,
we have the Markov chain
\[
  x \to u \to (g,\tau) \to g,
\]
and $g \perp x \mid (u,\tau)$. Integrating over the independent
$\tau$ yields $p(g\mid x,u) = p(g\mid u)$, and hence the conditions
of \Cref{lem:mi-deterministic} hold, giving
\[
  I(x;g) = I(u;g),
  \quad
  H(x\mid g) = H(x) - I(x;g) = H(x) - I(u;g).
\]

We bound $I(u;g)$ from above. By the chain rule and independence of
$u$ and $\tau$,
\[
  I(u;g)
  = I(u;g,\tau) - I(u;\tau\mid g)
  = I(u;g\mid\tau) - I(u;\tau\mid g)
  \;\le\; I(u;g\mid\tau),
\]
since $I(u;\tau\mid g) \ge 0$. Here $I(u;g\mid\tau)$ is conditional
mutual information and can be written as
\[
  I(u;g\mid\tau)
  = \E_\tau\bigl[I(u;g\mid\tau=t)\bigr].
\]
For a fixed realization $\tau=t$, the channel is linear with Gaussian
noise:
\[
  g\mid\tau=t = R(t)u + \varepsilon.
\]
Conditionally on mixture component $i$, $u\mid i \sim \N(\mu_i,\Sigma_i)$, so
\[
  g \mid (i,\tau=t)
  \sim \N\bigl(R(t)\mu_i,\;R(t)\Sigma_i R(t)^\top + \Sigma_g\bigr),
\]
and $g\mid\tau=t$ is a GMM with components indexed by $i$. For this
fixed $t$,
\[
  I(u;g\mid\tau=t) = H(g\mid\tau=t) - H(g\mid u,\tau=t).
\]
Since $g\mid(u,\tau=t) \sim \N(R(t)u,\Sigma_g)$, we obtain
\[
  H(g\mid u,\tau=t)
  = \tfrac12\log\bigl((2\pi e)^m|\Sigma_g|\bigr).
\]
Using the same mixture-entropy bound as before, applied to the GMM
$g\mid\tau=t$, we have
\[
  H(g\mid\tau=t)
  \;\le\;
  \sum_{i=1}^K \pi_i\left(
    -\log\pi_i
    + \tfrac12\log\bigl((2\pi e)^m|R(t)\Sigma_i R(t)^\top + \Sigma_g|\bigr)
  \right).
\]
Therefore
\begin{align*}
  I(u;g\mid\tau=t)
  &\le
  \sum_{i=1}^K \pi_i\left(
    -\log\pi_i
    + \tfrac12\log\bigl((2\pi e)^m|R(t)\Sigma_i R(t)^\top + \Sigma_g|\bigr)
  \right)
  - \tfrac12\log\bigl((2\pi e)^m|\Sigma_g|\bigr) \\
  &=
  \sum_{i=1}^K \pi_i\left(
    -\log\pi_i
    + \tfrac12\log\frac{|R(t)\Sigma_i R(t)^\top + \Sigma_g|}{|\Sigma_g|}
  \right).
\end{align*}
Taking expectation over $\tau$ yields
\[
  I(u;g\mid\tau)
  = \E_\tau I(u;g\mid\tau=t)
  \;\le\;
  \sum_{i=1}^K \pi_i\left(
    -\log\pi_i
    + \tfrac12\,\E_\tau\log\frac{|R(\tau)\Sigma_i R(\tau)^\top + \Sigma_g|}
                                        {|\Sigma_g|}
  \right).
\]
Combining $I(u;g) \le I(u;g\mid\tau)$ with
$H(x\mid g) = H(x) - I(u;g)$ gives~\eqref{eq:Hxg-tele}.
\end{proof}

\Cref{thm:teleportation-Hxg} is the teleportation analogue of
\Cref{thm:grad-gmm-lb}, obtained via the same steps, with $\Sigma_i$
replaced by $R(\tau)\Sigma_i R(\tau)^\top$ and an additional
expectation over $\tau$.

\subsection{Diagonal approximation and the role of the CoB distribution}

To make the teleportation effect more interpretable at a
per-coordinate level, we now adopt a diagonal approximation. This is
a modelling simplification, similar in spirit to
\cite{xia2025theoretical}, and all comparisons between teleported and
baseline channels will be made \emph{within} this shared surrogate
approximation.

\subsubsection{Diagonal approximation}

\begin{assumption}[Diagonal covariance approximation]
\label{ass:diag}
We work in the canonical channel basis in which teleportation is
defined and posit that, in this basis,
\[
  \Sigma_i = \diag(\sigma_{i,1}^2,\dots,\sigma_{i,m}^2),
  \quad
  \Sigma_g = \diag(\gamma_1^2,\dots,\gamma_m^2),
\]
and the teleportation matrix has the form
\[
  R(\tau) = \diag(r_1(\tau),\dots,r_m(\tau)).
\]
\end{assumption}

That is, we adopt a surrogate model in which gradient covariance,
observation noise and CoB factors act coordinatewise in the natural
channel basis, rather than attempting to diagonalize arbitrary
covariances and then reinterpret teleportation in that rotated frame.
This is not claimed to be an exact description of real networks, but
a structured approximation for per-coordinate interpretation.

Under \Cref{ass:diag},
\[
  R(\tau)\Sigma_i R(\tau)^\top + \Sigma_g
  = \diag\bigl(\gamma_1^2 + r_1(\tau)^2\sigma_{i,1}^2,\dots,
               \gamma_m^2 + r_m(\tau)^2\sigma_{i,m}^2\bigr),
\]
and hence
\begin{equation}
  \frac{|R(\tau)\Sigma_i R(\tau)^\top + \Sigma_g|}{|\Sigma_g|}
  = \prod_{j=1}^m \Bigl(1 + \alpha_{i,j} r_j(\tau)^2\Bigr),
  \quad
  \alpha_{i,j} \coloneqq \frac{\sigma_{i,j}^2}{\gamma_j^2}.
  \label{eq:ratio-diag}
\end{equation}
Taking logs and expectation in~\eqref{eq:Hxg-tele}, we obtain
\[
  \E_\tau\log\frac{|R(\tau)\Sigma_i R(\tau)^\top + \Sigma_g|}{|\Sigma_g|}
  = \sum_{j=1}^m \psi_{i,j},
\]
where we define the per-coordinate quantities
\begin{equation}
  \psi_{i,j}
  \;\coloneqq\;
  \E_\tau\bigl[\log(1 + \alpha_{i,j} r_j(\tau)^2)\bigr].
  \label{eq:psi-ij-def}
\end{equation}
Thus \Cref{thm:teleportation-Hxg} becomes, under \Cref{ass:diag},
\begin{equation}
  H(x\mid g)
  \;\ge\;
  H(x)
  -
  \sum_{i=1}^K \pi_i\left(
    -\log\pi_i
    + \frac{1}{2}\sum_{j=1}^m \psi_{i,j}
  \right).
  \label{eq:Hxg-tele-diag}
\end{equation}

\subsubsection{Baseline (non-teleported) diagonal bound}

For comparison, if no teleportation is applied, we have $R(\tau)\equiv I$ and
$r_j(\tau)^2 \equiv 1$. Under the same diagonal surrogate,
\[
  \frac{|\Sigma_i + \Sigma_g|}{|\Sigma_g|}
  = \prod_{j=1}^m (1+\alpha_{i,j}),
\]
and the GMM-based entropy lower bound~\eqref{eq:Hxg-lb-gmm} reduces
to
\begin{equation}
  H(x\mid g_0)
  \;\ge\;
  H^{\mathrm{lb}}_0
  \;\coloneqq\;
  H(x)
  -
  \sum_{i=1}^K \pi_i
  \left(
    -\log\pi_i
    + \frac{1}{2}\sum_{j=1}^m \log\bigl(1 + \alpha_{i,j}\bigr)
  \right).
  \label{eq:H0-lb-def}
\end{equation}
We explicitly introduce $H^{\mathrm{lb}}_0$ to denote the analytic
lower bound on $H(x\mid g_0)$ obtained under the GMM and diagonal
surrogate.

Similarly, in the teleported diagonal setting
\eqref{eq:Hxg-tele-diag} we define
\begin{equation}
  H(x\mid g)
  \;\ge\;
  H^{\mathrm{lb}}_{\tele}
  \;\coloneqq\;
  H(x)
  -
  \sum_{i=1}^K \pi_i\left(
    -\log\pi_i
    + \frac{1}{2}\sum_{j=1}^m \psi_{i,j}
  \right).
  \label{eq:Htele-lb-def}
\end{equation}
Both $H^{\mathrm{lb}}_0$ and $H^{\mathrm{lb}}_{\tele}$ are computed
from exactly the same modelling ingredients and diagonal surrogate.

\subsubsection{Energy-preserving CoB and improvement of the bound}

To isolate teleportation as a pure source of \emph{randomization}
(rather than a trivial global rescaling of gradient energy), we
consider energy-preserving CoB distributions at the level of the
per-coordinate effective ratios.

\begin{assumption}[Energy-preserving CoB marginals]
\label{ass:energy-pres}
For each coordinate $j$, the marginal distribution of $r_j(\tau)^2$
satisfies $\E_\tau[r_j(\tau)^2] = 1$.
\end{assumption}

This condition enforces that, on average, teleportation does not
inflate or shrink per-coordinate gradient energy; it only redistributes
it stochastically. In practice, the defender controls the sampling of
$\tau$ and hence the induced distribution of ratios $\{r_j(\tau)\}$,
subject to architectural constraints (shared channels, etc.). We do
not model those constraints explicitly here; we treat $\{r_j(\tau)\}$
as effective per-coordinate scalings whose marginals can be chosen to
satisfy \Cref{ass:energy-pres}.

We do not assume independence of $r_j(\tau)$ across $j$, only these
marginals.

Define, for each $(i,j)$,
\begin{equation}
  \Delta\psi_{i,j}
  \;\coloneqq\;
  \log\bigl(1+\alpha_{i,j}\bigr) - \psi_{i,j}
  \;=\;
  \log\bigl(1+\alpha_{i,j}\bigr)
  - \E_\tau\bigl[\log(1+\alpha_{i,j} r_j(\tau)^2)\bigr].
  \label{eq:Deltapsi-def}
\end{equation}
Subtracting \eqref{eq:H0-lb-def} from \eqref{eq:Htele-lb-def} yields
an exact relation between the two analytic entropy lower bounds under
the diagonal surrogate.

\begin{corollary}[Exact relation between diagonal entropy lower bounds]
\label{cor:tele-improvement-exact-lb}
Under \Cref{ass:diag}, the diagonal entropy lower bounds
\eqref{eq:H0-lb-def}--\eqref{eq:Htele-lb-def} satisfy
\begin{equation}
  H^{\mathrm{lb}}_{\tele}
  \;=\;
  H^{\mathrm{lb}}_0
  + \frac{1}{2}\sum_{i=1}^K \pi_i\sum_{j=1}^m \Delta\psi_{i,j},
  \label{eq:Hxg-improvement-exact-lb}
\end{equation}
with $\Delta\psi_{i,j}$ defined in~\eqref{eq:Deltapsi-def}. If, in
addition, \Cref{ass:energy-pres} holds, then each $\Delta\psi_{i,j}$
is non-negative, and hence
\begin{equation}
  H^{\mathrm{lb}}_{\tele}
  \;\ge\;
  H^{\mathrm{lb}}_0.
  \label{eq:Hlb-tele-vs-Hlb0}
\end{equation}
\end{corollary}

\begin{proof}
Equation~\eqref{eq:Hxg-improvement-exact-lb} is obtained by direct
subtraction of \eqref{eq:H0-lb-def} from \eqref{eq:Htele-lb-def} and
using \eqref{eq:Deltapsi-def}. For the sign of $\Delta\psi_{i,j}$,
fix $\alpha>0$ and define
$\phi_\alpha(t)\coloneqq \log(1+\alpha t)$, which is concave on
$t>0$. Under \Cref{ass:energy-pres},
\[
  \psi_{i,j}
  = \E_\tau[\phi_{\alpha_{i,j}}(r_j(\tau)^2)]
  \;\le\;
  \phi_{\alpha_{i,j}}\bigl(\E_\tau[r_j(\tau)^2]\bigr)
  = \phi_{\alpha_{i,j}}(1)
  = \log(1+\alpha_{i,j}),
\]
so $\Delta\psi_{i,j}\ge 0$ for all $(i,j)$, implying
\eqref{eq:Hlb-tele-vs-Hlb0}.
\end{proof}

\begin{remark}[Scope and strength of the entropy result]
\label{rem:scope}
Within the shared modelling assumptions (GMM, diagonal surrogate,
energy-preserving CoB), \Cref{cor:tele-improvement-exact-lb} shows
that teleportation \emph{never decreases} the analytic entropy lower
bound:
\[
  H(x\mid g_0) \;\ge\; H^{\mathrm{lb}}_0,
  \qquad
  H(x\mid g) \;\ge\; H^{\mathrm{lb}}_{\tele} \;\ge\; H^{\mathrm{lb}}_0.
\]
We stress that we do \emph{not} claim $H(x\mid g)\ge H(x\mid g_0)$
for the true channels. Rather, we compare the surrogate quantities
$H^{\mathrm{lb}}_0$ and $H^{\mathrm{lb}}_{\tele}$ arising under the
same generative model; under this common lens, teleportation strictly
improves the analytic lower bound on uncertainty about $x$.
\end{remark}

\subsection{Teleportation-aware reconstruction lower bound}

We now translate the entropy bounds into reconstruction MSE lower
bounds using \Cref{thm:entropy-mse-grad}.

\subsubsection{Baseline and teleported MSE lower bounds}

From \eqref{eq:H0-lb-def}--\eqref{eq:Htele-lb-def} and
\Cref{thm:entropy-mse-grad}, we obtain analytic lower bounds on the
minimal reconstruction MSE for the baseline and teleported channels:
\begin{align}
  \underline{\xi}_0
  &\;\coloneqq\;
  \frac{1}{2\pi e}\,
  \exp\!\left(
    \frac{2}{d} H^{\mathrm{lb}}_0
  \right),
  \label{eq:xi0-lb} \\
  \underline{\xi}_{\tele}
  &\;\coloneqq\;
  \frac{1}{2\pi e}\,
  \exp\!\left(
    \frac{2}{d} H^{\mathrm{lb}}_{\tele}
  \right).
  \label{eq:xitele-lb}
\end{align}
By construction and monotonicity of the exponential,
\begin{equation}
  \xi_{g_0} \;\ge\; \underline{\xi}_0,
  \qquad
  \xi_{g} \;\ge\; \underline{\xi}_{\tele},
  \label{eq:true-vs-lb}
\end{equation}
where $\xi_{g_0}$ and $\xi_g$ are the true minimal MSEs for the
baseline and teleported channels, respectively. Again, $H(x)$ is
common to both channels and cancels in all \emph{relative} statements
about $\underline{\xi}_{\tele}/\underline{\xi}_0$.

\subsubsection{Improvement factor on the analytic MSE bound}

Combining the definitions, the ratio between the teleported and
baseline MSE \emph{lower bounds} satisfies
\begin{equation}
  \frac{\underline{\xi}_{\tele}}{\underline{\xi}_0}
  \;=\;
  \exp\!\left(
    \frac{2}{d}(H^{\mathrm{lb}}_{\tele} - H^{\mathrm{lb}}_0)
  \right)
  \;=\;
  \exp\!\left(
    \frac{1}{d}\sum_{i=1}^K \pi_i\sum_{j=1}^m \Delta\psi_{i,j}
  \right),
  \label{eq:ratio-lb-general}
\end{equation}
with $\Delta\psi_{i,j}$ as in~\eqref{eq:Deltapsi-def}.

Under the energy-preserving assumption (\Cref{ass:energy-pres}),
$\Delta\psi_{i,j} \ge 0$, hence the exponential factor in
\eqref{eq:ratio-lb-general} is at least $1$, and the analytic
teleportation-aware MSE lower bound is never smaller than the
baseline one. In other words, teleportation provably raises the
information-theoretic floor on reconstruction accuracy \emph{as
captured by this shared surrogate model}. We do not assert any
ordering between the true minimal MSEs $\xi_{g_0}$ and $\xi_g$.

\begin{remark}[Interpretation for privacy]
\label{rem:interpretation-final}
Equation~\eqref{eq:ratio-lb-general} provides a quantitative,
distribution-aware guarantee: under the shared assumptions
(GMM, diagonal surrogate, energy-preserving CoB), teleportation
inflates the analytic lower bound on the attacker’s reconstruction
MSE by a factor given by the RHS of
\eqref{eq:ratio-lb-general}. This factor depends on the CoB
distribution only through $\Delta\psi_{i,j}$, which in turn are
functions of the per-coordinate signal-to-noise ratios $\alpha_{i,j}$
and the marginals of $r_j(\tau)^2$. Thus teleportation is not merely
a heuristic perturbation: for any attacker whose behaviour is
dominated by this generative model (in essentially the same sense as
in \cite{xia2025theoretical}), there is a formal lower bound on how
accurately they can reconstruct $x$.
\end{remark}

\subsection{Log-normal CoB family (effective model)}

We now specialize the general diagonal analysis to an effective
log-normal model for the CoB-induced per-coordinate scalings
$r_j(\tau)^2$, to make the dependence on CoB variance explicit in
the analytic MSE lower bounds.

\paragraph{Log-normal marginal model.}
We model each per-coordinate scaling as
\[
  r_j(\tau)^2 = \exp(Y_j),
\]
where
\[
  Y_j \sim \N\bigl(-\tfrac{1}{2}s_j^2,\, s_j^2\bigr),
\]
so that
\[
  \E[r_j(\tau)^2]
  = \E[e^{Y_j}]
  = \exp\bigl(-\tfrac{1}{2}s_j^2 + \tfrac{1}{2}s_j^2\bigr)
  = 1.
\]
This ensures the energy-preserving condition
$\E[r_j(\tau)^2]=1$ (\Cref{ass:energy-pres}), while the parameter
$s_j^2 \ge 0$ controls the strength of teleportation-induced
variability on coordinate $j$. Practically, the defender can aim to
implement such marginals by sampling $\tau$ so that the induced
ratios $r_j(\tau)^2$ are approximately log-normal; we do not model
the exact mapping from channel-wise $\tau^{[\ell]}$ to ratio
marginals. We emphasize that this is an \emph{effective parametric
family} for $r_j^2$, chosen for analytical clarity; our rigorous
inequalities rely only on \Cref{ass:energy-pres}, while log-normality
is used to express the dependence on a small number of variance
parameters.

Under this model, the per-coordinate quantities
$\psi_{i,j}$ and $\Delta\psi_{i,j}$ admit explicit expressions.

\begin{corollary}[Log-normal teleportation and analytic MSE bound improvement]
\label{cor:lognormal-improvement}
Under the log-normal CoB marginal model above, for each mixture
component $i$ and coordinate $j$,
\begin{equation}
  \psi_{i,j}(s_j^2)
  \;=\;
  \E_{Y_j \sim \N(-\tfrac{1}{2}s_j^2,\,s_j^2)}
  \bigl[\log(1 + \alpha_{i,j} e^{Y_j})\bigr],
  \label{eq:psi-lognormal}
\end{equation}
and
\begin{equation}
  \Delta\psi_{i,j}(s_j^2)
  \;=\;
  \log(1+\alpha_{i,j})
  -
  \E_{Y_j \sim \N(-\tfrac{1}{2}s_j^2,\,s_j^2)}
  \bigl[\log(1 + \alpha_{i,j} e^{Y_j})\bigr].
  \label{eq:Deltapsi-lognormal}
\end{equation}
Let $\underline{\xi}_0$ and $\underline{\xi}_{\tele}$ denote the
analytic lower bounds on the minimal reconstruction MSE for the
non-teleported and teleported channels, respectively, as defined in
\eqref{eq:xi0-lb}--\eqref{eq:xitele-lb}. Then
\begin{equation}
  \frac{\underline{\xi}_{\tele}(s^2)}{\underline{\xi}_0}
  \;=\;
  \exp\!\left(
    \frac{1}{d}
    \sum_{i=1}^K \pi_i\sum_{j=1}^m
    \Bigl[
      \log(1+\alpha_{i,j}) - \psi_{i,j}(s_j^2)
    \Bigr]
  \right),
  \label{eq:ratio-lb-lognormal-exact}
\end{equation}
where $s^2 = (s_1^2,\dots,s_m^2)$ collects the log-variance parameters
across coordinates.
\end{corollary}

\begin{proof}
The identities \eqref{eq:psi-lognormal}–\eqref{eq:Deltapsi-lognormal}
are obtained by substituting $r_j^2 = e^{Y_j}$ with
$Y_j~\sim~\N(-\tfrac{1}{2}s_j^2,\,s_j^2)$ into the definition
\eqref{eq:psi-ij-def} of $\psi_{i,j}$ and the definition
\eqref{eq:Deltapsi-def} of $\Delta\psi_{i,j}$. The ratio
\eqref{eq:ratio-lb-lognormal-exact} then follows immediately by
plugging $\Delta\psi_{i,j}(s_j^2)$ into
\eqref{eq:ratio-lb-general}, which relates the analytic MSE lower
bounds $\underline{\xi}_{\tele}$ and $\underline{\xi}_0$ to the
$\Delta\psi_{i,j}$.
\end{proof}

\begin{remark}[Local small-variance expansion (heuristic)]
\label{rem:small-var}
To gain intuition about the dependence on teleportation strength, it
is useful to consider the regime $s_j^2 \ll 1$, where the log-normal
marginals are close to the degenerate case $r_j^2 \equiv 1$. This
section provides a local Taylor expansion for intuition; it is
\emph{not} used in our rigorous inequalities, which already follow
from \Cref{ass:energy-pres}.

For $t_j \coloneqq r_j^2 = e^{Y_j}$ with
$Y_j \sim \N(-\tfrac{1}{2}s_j^2,\,s_j^2)$ we have
\[
  \E[t_j] = 1,
  \qquad
  \Var(t_j) = \E[t_j^2] - 1 = \exp(s_j^2) - 1.
\]
Thus $\Var(t_j) = s_j^2 + O(s_j^4)$ as $s_j^2 \to 0$. Writing
$t_j = 1 + \delta_j$, we have $\E[\delta_j]=0$ and
$\Var(\delta_j)~=~\Var(t_j)$.

Since log-normal marginals have finite moments of all orders, a
second-order Taylor expansion of
$\phi_\alpha(t) \coloneqq \log(1+\alpha t)$ around $t=1$ yields
\[
  \phi_\alpha(1+\delta)
  =
  \log(1+\alpha)
  + \frac{\alpha}{1+\alpha}\,\delta
  - \frac{\alpha^2}{2(1+\alpha)^2}\,\delta^2
  + R_\alpha(\delta),
\]
with $|R_\alpha(\delta)| \le C_\alpha |\delta|^3$ for some constant
$C_\alpha$ depending on $\alpha$. Taking expectations with
$\E[\delta]=0$ and $\E[\delta^2] = \Var(t_j)$ gives
\[
  \E[\phi_\alpha(1+\delta)]
  =
  \log(1+\alpha)
  - \frac{\alpha^2}{2(1+\alpha)^2}\,\Var(t_j)
  + O\bigl(\E[|\delta|^3]\bigr).
\]
Applying this with $\alpha = \alpha_{i,j}$ and $t_j = r_j^2$, and
recalling that
$\psi_{i,j} = \E[\log(1+\alpha_{i,j} r_j^2)]$, we obtain the local
approximation
\[
  \Delta\psi_{i,j}(s_j^2)
  =
  \log(1+\alpha_{i,j}) - \psi_{i,j}(s_j^2)
  \;\approx\;
  \frac{\alpha_{i,j}^2}{2(1+\alpha_{i,j})^2}\,
  \Var(t_j),
\]
with an error term controlled by $\E[|\delta_j|^3]$. For the log-normal
model, $\Var(t_j) = \exp(s_j^2) - 1$, so we arrive at the heuristic
expression
\[
  \Delta\psi_{i,j}(s_j^2)
  \;\approx\;
  \frac{\alpha_{i,j}^2}{2(1+\alpha_{i,j})^2}\,
  \bigl(\exp(s_j^2) - 1\bigr),
  \qquad s_j^2 \ll 1.
\]
Substituting this into \eqref{eq:ratio-lb-lognormal-exact} yields the
corresponding small-variance approximation for the logarithm of the
analytic MSE bound ratio:
\[
  \log\frac{\underline{\xi}_{\tele}(s^2)}{\underline{\xi}_0}
  \;\approx\;
  \frac{1}{2d}
  \sum_{i=1}^K \pi_i\sum_{j=1}^m
  \frac{\alpha_{i,j}^2}{(1+\alpha_{i,j})^2}\,
  \bigl(\exp(s_j^2) - 1\bigr),
  \qquad s_j^2 \ll 1.
\]

This approach highlight that, in the small-variance regime, the teleportation-induced
improvement in the analytic reconstruction MSE lower bound grows approximately linearly in $s_j^2$ (via $\exp(s_j^2)-1$), with a slope governed by the per-coordinate signal-to-noise ratios $\alpha_{i,j}$ and the mixture weights~$\pi_i$

\end{remark}

\section{Ablation: Sensitivity of Teleportation Hyperparameters}
\label{sec:teleport-hp-sensitivity}

Teleportation introduces a small set of additional hyperparameters that control how strongly we move along symmetry directions.  
In this section we study the sensitivity of WARP to two core choices:  
(i) the target retain-variance fraction used to choose the per-layer rank $k$ in the SVD projector (Section~\ref{sec:method-teleportation-based-defense}), and  
(ii) the size of the retain minibatch $B_r$ used to estimate the retain subspace.  
Both directly govern the geometry of the retain null space and the amount of stochasticity in the teleportation step, and were explicitly highlighted as potential sources of instability.

\paragraph{Setup.}
We perform a controlled sweep on CIFAR-10 with ResNet-18 and NGP+WARP under the U-LiRA black-box auditor (Section~\ref{sec:exp-ulira}).  
For the SVD projector, we vary the target retain-variance level from $95\%$ to $99.9\%$, which induces different per-layer ranks $k_\ell$ such that the top singular vectors of $R_\ell(D_r)$ capture the chosen fraction of retain energy.  
For the retain minibatch, we vary the teleportation batch size $|B_r| \in \{256, 512, 1024, 2048, 4096\}$ while keeping the forget minibatch and unlearning hyperparameters fixed.  
For each configuration we run the full unlearning pipeline and record test accuracy as well as privacy measured by $(1-\mathrm{AUC})$ of U-LiRA (higher is better).

\begin{figure}[t]
    \centering
    \begin{minipage}{0.48\linewidth}
        \centering
        \includegraphics[width=\linewidth]{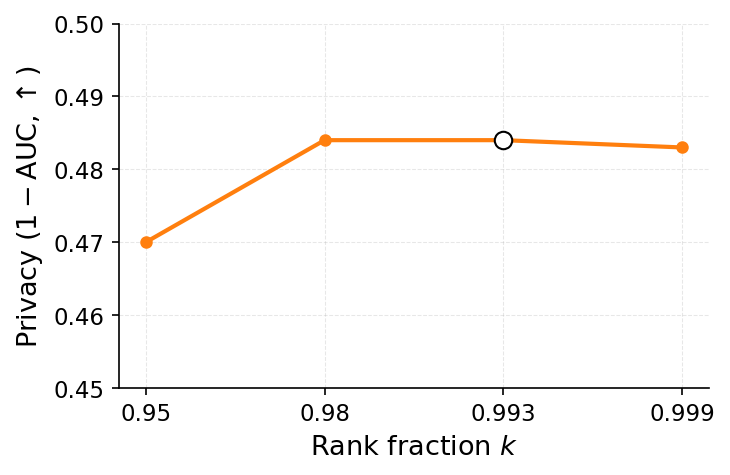}
        \vspace{1mm}
        \textbf{(A)}~Privacy $(1-\mathrm{AUC})$ vs.\ retain-variance target
    \end{minipage}
    \hfill
    \begin{minipage}{0.48\linewidth}
        \centering
        \includegraphics[width=\linewidth]{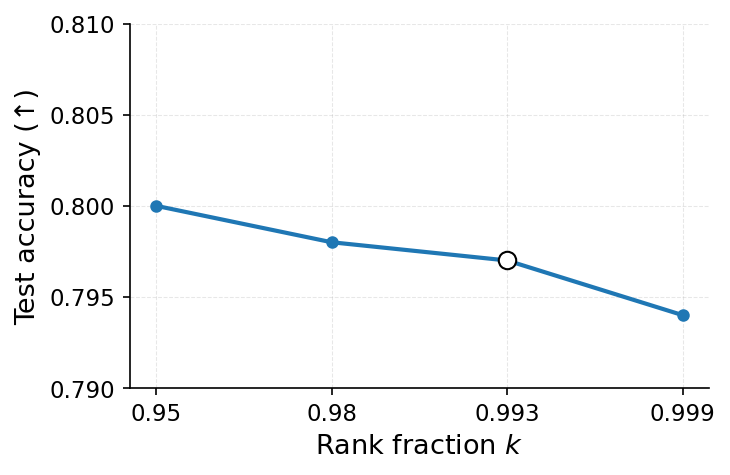}
        \vspace{1mm}
        \textbf{(B)}~Test accuracy vs.\ retain-variance target
    \end{minipage}

    \vspace{2mm}

    \begin{minipage}{0.48\linewidth}
        \centering
        \includegraphics[width=\linewidth]{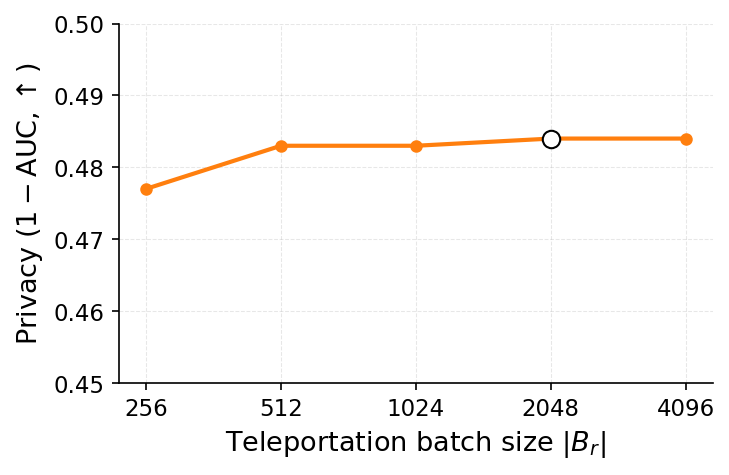}
        \vspace{1mm}
        \textbf{(C)}~Privacy $(1-\mathrm{AUC})$ vs.\ teleportation batch size $|B_r|$
    \end{minipage}
    \hfill
    \begin{minipage}{0.48\linewidth}
        \centering
        \includegraphics[width=\linewidth]{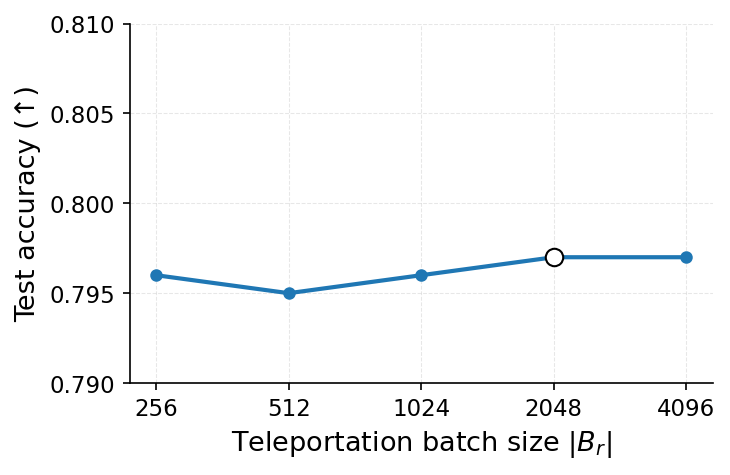}
        \vspace{1mm}
        \textbf{(D)}~Test accuracy vs.\ teleportation batch size $|B_r|$
    \end{minipage}

    \vspace{2mm}
    \caption{\textbf{Sensitivity of teleportation hyperparameters.}
        Plots (A,B) vary the target retain-variance level used to set the per-layer rank $k_\ell$; plots (C,D) vary the retain minibatch size $|B_r|$ used to estimate the retain subspace.
        Privacy is measured as $1-\mathrm{AUC}$ of U-LiRA (higher is better).  
        Markers highlight the configuration used in our main experiments ($95.3\%$ retain variance and $|B_r|=2048$).}
    \label{fig:teleport-hp-sensitivity}
\end{figure}

\paragraph{Results and discussion.}
Figure~\ref{fig:teleport-hp-sensitivity} shows that teleportation is \emph{remarkably insensitive} to both hyperparameters in the regime we consider.

\emph{Retain-variance target.}  
Increasing the target retain-variance from $95\%$ to $99.9\%$ changes privacy $(1-\mathrm{AUC})$ by less than $0.015$ in absolute terms, while test accuracy varies in a narrow band of $\approx 0.79$--$0.80$.  
Privacy slightly improves as we move from $95\%$ to around $99.3\%$, after which the curve flattens: very high targets effectively make the retain projector full-rank, leaving less room for teleportation to move in symmetry directions and yielding diminishing returns.  
The configuration used in the main experiments (target retain-variance $\approx 99.3\%$) lies near this plateau, indicating that our chosen rank provides a good privacy--utility compromise.

\emph{Retain minibatch size $|B_r|$.}  
Varying $|B_r|$ over an order of magnitude has only a minor effect: privacy $(1-\mathrm{AUC})$ shifts by at most $\sim 0.01$, and test accuracy remains within $\pm 0.2\%$ points of $0.796$.  
Even relatively small batches ($|B_r|=256$) already provide a sufficiently representative retain subspace for teleportation, and larger batches only yield a slight, saturating gain in privacy.  
This suggests that the random retain minibatch need not tightly approximate the full retain set to obtain a stable projector and effective defense; in practice, a modest $|B_r|$ balances computational cost with stable subspace estimation.

Overall, these ablations show that WARP’s performance does not hinge on fragile hyperparameter choices: both privacy and utility are stable across wide ranges of the SVD rank and retain minibatch size.  
Moreover, the small spread in test accuracy (\(<0.6\%\) across all settings) empirically confirms that teleportation remains approximately loss-preserving on the retain set, providing an implicit bound on worst-case retain-loss drift in our experiments.

\end{document}